\documentclass{article}

\usepackage[numbers]{natbib}

\usepackage[final]{nips_2016}

\usepackage[utf8]{inputenc} 
\usepackage[T1]{fontenc}    
\usepackage{hyperref}       
\usepackage{url}            
\usepackage{booktabs}       
\usepackage{amsfonts}       
\usepackage{nicefrac}       
\usepackage{microtype}      

\usepackage{amsmath}
\usepackage{amssymb}
\usepackage{amsthm}

\usepackage{graphicx} 
\usepackage{subfigure} 
\graphicspath{{figures/}}

\usepackage{xcolor}

\usepackage{wrapfig}

\usepackage{algorithm}
\usepackage{algorithmic}

\input{Definitions}

\ifx\proof\undefined
\newenvironment{proof}{\par\noindent{\bf Proof\ }}{\hfill\BlackBox\\[2mm]}
\fi

\ifx\theorem\undefined
\newtheorem{theorem}{Theorem}
\fi
\ifx\example\undefined
\newtheorem{example}{Example}
\fi
\ifx\lemma\undefined
\newtheorem{lemma}[theorem]{Lemma}
\fi

\ifx\corollary\undefined
\newtheorem{corollary}[theorem]{Corollary}
\fi

\ifx\assumption\undefined
\newtheorem{assumption}{Assumption}
\fi

\ifx\definition\undefined
\newtheorem{definition}{Definition}
\fi

\ifx\proposition\undefined
\newtheorem{proposition}[theorem]{Proposition}
\fi

\ifx\remark\undefined
\newtheorem{remark}{Remark}
\fi

\ifx\conjecture\undefined
\newtheorem{conjecture}[theorem]{Conjecture}
\fi

\ifx\factoid\undefined
\newtheorem{factoid}[theorem]{Fact}
\fi

\ifx\axiom\undefined
\newtheorem{axiom}[theorem]{Axiom}
\fi

\newcommand{\RN}[1]{%
  \textup{\lowercase\expandafter{\it \romannumeral#1}}%
}

\title{Stochastic Gradient MCMC with Stale Gradients}

\author{
	Changyou Chen$^\dag$~~~~~~~ Nan Ding$^\ddag$~~~~~~~ Chunyuan Li$^\dag$~~~~~~~ Yizhe Zhang$^\dag$ ~~~~~~~ Lawrence Carin$^\dag$ \\
	$^\dag$Dept. of Electrical and Computer Engineering, Duke University, Durham, NC, USA \\
	$^\ddag$Google Inc., Venice, CA, USA \\
	\texttt{$^\dag$\{cc448,cl319,yz196,lcarin\}@duke.edu; $^\ddag$dingnan@google.com}
}

\begin{document}

\maketitle

\begin{abstract}
Stochastic gradient MCMC (SG-MCMC) has played an important role in large-scale 
Bayesian learning, with well-developed theoretical convergence properties. 
In such applications of SG-MCMC, it is becoming increasingly popular to employ 
distributed systems, where stochastic gradients are computed based on some 
outdated parameters, yielding what are termed {\em stale gradients}. While 
stale gradients could be directly used in SG-MCMC, their impact on convergence 
properties has not been well studied. In this paper we develop theory to show 
that while the bias and MSE of an SG-MCMC 
algorithm depend on the staleness of stochastic gradients, its estimation variance 
(relative to the {\em expected} estimate, based on a prescribed number of samples) is 
independent of it. In a simple Bayesian distributed system with SG-MCMC, where stale 
gradients are computed asynchronously by a set of workers, our theory indicates a
linear speedup on the decrease of estimation variance w.r.t.\! the number of workers. 
Experiments on synthetic data and deep neural 
networks validate our theory, demonstrating the effectiveness and scalability of SG-MCMC 
with stale gradients.
\end{abstract}

\section{Introduction}
\label{sec:intro}

The pervasiveness of big data has made scalable machine learning increasingly 
important, especially for deep models. A basic technique is to adopt stochastic optimization algorithms \cite{Bottou:98}, 
{\it e.g.}, stochastic gradient descent and its extensions \cite{Bottou:12}.
In each iteration of stochastic optimization, a minibatch of data is used to 
evaluate the gradients of the objective function and update model 
parameters (errors are introduced in the gradients, because they are computed 
based on minibatches rather than the entire dataset; since the minibatches are 
typically selected at random, this yields the term ``stochastic'' gradient). 
This is highly scalable because processing a minibatch of data in each 
iteration is relatively cheap compared to analyzing the entire (large) dataset at once. Under 
certain conditions, stochastic optimization is guaranteed to converge to a (local) 
optima \cite{Bottou:98}. Because of its scalability, the minibatch strategy has 
recently been extended to Markov Chain Monte Carlo (MCMC) Bayesian sampling 
methods, yielding SG-MCMC \cite{WellingT:ICML11,ChenFG:ICML14,DingFBCSN:NIPS14}.

In order to handle large-scale data, distributed 
stochastic optimization algorithms have been developed, for example \cite{AgarwalD:NIPS11}, 
to further improve scalability. In a distributed setting, a cluster of machines with 
multiple cores cooperate with each other, typically through an asynchronous 
scheme, for scalability \cite{Dean2012,Chen2015mxnet,tensorflow2015-whitepaper}. 
A downside of an asynchronous implementation is that stale gradients must be used in 
parameter updates (``stale gradients'' are stochastic gradients computed based on outdated
parameters, instead of the latest parameters; they are easier to compute in a distributed system,
but introduce additional errors relative to traditional stochastic gradients). While some theory has been 
developed to guarantee the convergence of stochastic optimization with stale 
gradients \cite{HoCCKLGGGX:NPS13,LiASY:NPS14,LianHLL:NIPS15}, little analysis has been
done in a Bayesian setting, where SG-MCMC is applied. Distributed SG-MCMC 
algorithms share characteristics with distributed stochastic optimization, and thus are highly 
scalable and suitable for large-scale Bayesian learning. Existing Bayesian distributed
systems with traditional MCMC methods, such 
as \cite{AhmedAGNS:WSDM12}, usually employ stale {\em statistics} 
instead of stale {\em gradients}, where stale statistics are summarized based on outdated 
parameters, {\it e.g.}, outdated topic distributions in distributed Gibbs sampling \cite{AhmedAGNS:WSDM12}. 
Little theory exists to guarantee the convergence of such methods.
For existing distributed SG-MCMC methods, typically only standard stochastic gradients are used, for limited 
problems such as matrix factorization, without rigorous convergence theory \cite{AhnSW:ICML14,AhnKLRW:ICML14,Simsekli:15}.

In this paper, by extending techniques from standard SG-MCMC \cite{ChenDC:NIPS15},
we develop theory to study the convergence behavior of SG-MCMC with Stale gradients (S$^2$G-MCMC).
Our goal is to evaluate the {\em posterior average} of a test function $\phi(\xb)$,
defined as
$\bar{\phi} \triangleq \int_{\mathcal{X}} \phi(\xb) \rho(\xb) \mathrm{d}\xb$,
where $\rho(\xb)$ is the desired posterior distribution with $\xb$ the possibly augmented
model parameters (see Section~\ref{sec:prelim}). In practice, S$^2$G-MCMC generates 
$L$ samples $\{\xb_{l}\}_{l=1}^L$ and uses the {\em sample average} 
$\hat{\phi}_L \triangleq \frac{1}{L} \sum_{l = 1}^L \phi(\xb_{l})$ to approximate $\bar{\phi}$. 
We measure how $\hat{\phi}_L$ approximates $\bar{\phi}$ in terms of {\em bias}, 
{\em MSE} and {\em estimation variance}, defined as $|\mathbb{E}\hat{\phi}_L - \bar{\phi}|$,
$\mathbb{E}\left(\hat{\phi}_L - \bar{\phi}\right)^2$ and 
$\mathbb{E}\left(\hat{\phi}_L - \mathbb{E} \hat{\phi}_L\right)^2$, respectively.
From the definitions, the bias and MSE characterize how accurately $\hat{\phi}_L$
approximates $\bar{\phi}$, and the variance characterizes how fast $\hat{\phi}_L$ 
converges to its own expectation (for a prescribed number of samples $L$).
Our theoretical results show that while the bias and MSE depend on the staleness of stochastic gradients,
the variance is independent of it. In a simple {\em asynchronous Bayesian distributed system} with S$^2$G-MCMC,
our theory indicates a linear speedup on the decrease of the variance w.r.t.\! the number of workers used to 
calculate the stale gradients, while {\em maintaining the same optimal bias level as standard SG-MCMC}. 
We validate our theory on several synthetic experiments and deep neural network 
models, demonstrating the effectiveness and scalability of the proposed S$^2$G-MCMC framework.

\paragraph{Related Work}
Using stale gradients is a standard setup in distributed stochastic optimization systems. 
Representative algorithms include, but are not limited to, the 
ASYSG-CON \cite{AgarwalD:NIPS11} and {HOGWILD!} algorithms \cite{NiuRW:NIPS11},
and some more recent developments \cite{FeyzmahdavianAJ:ARXIV15,ChaturapruekDR:NIPS15}.
Furthermore, recent research on stochastic optimization
has been extended to non-convex problems with provable convergence rates \cite{LianHLL:NIPS15}.
In Bayesian learning with MCMC, existing work has focused on running parallel chains on 
subsets of data \cite{ScottBB:13,RabinovichAJ:NIPS15,NeiswangerWX:UAI14,WangGHD:NIPS15},
and little if any effort has been made to use stale stochastic gradients, the setting considered in this paper.

\section{Stochastic Gradient MCMC} \label{sec:prelim}

Throughout this paper, we denote vectors as bold lower-case letters, and matrices as 
bold upper-case letters. For example, $\Ncal(\mb, \Sigmab)$ means a multivariate 
Gaussian distribution with mean $\mb$ and covariance $\Sigmab$. In the analysis we 
consider algorithms with fixed-stepsizes for simplicity; decreasing-stepsize variants can
be addressed similarly as in \cite{ChenDC:NIPS15}.

The goal of SG-MCMC is to generate random samples from a posterior distribution 
$p(\thetab | \Db) \propto p(\thetab)\prod_{i=1}^N p(\db_i | \thetab)$, which 
are used to evaluate a test function. Here $\thetab \in \mathbb{R}^n$ 
represents the parameter vector and  $\Db = \{\db_1, \cdots, \db_N\}$ 
represents the data, $p(\thetab)$ is the prior distribution, and $p(\db_i|\thetab)$
the likelihood for $\db_i$.
SG-MCMC algorithms are based on a class of stochastic differential equations, called 
It\^{o} diffusion, defined as
\begin{align}
\mathrm{d}\xb_t &= F(\xb_t)\mathrm{d}t + g(\xb_t)\mathrm{d}\mathcal{\wb}_t~, \label{eq:itodiffusion}
\end{align}
where $\xb \in \mathbb{R}^m$ represents the model states, typically $\xb$ augments 
$\thetab$ such that $\thetab \subseteq \xb$ 
and  $n \le m$; $t$ is the time index, $\mathcal{\wb}_t \in \mathbb{R}^m$ is 
$m$-dimensional Brownian motion, functions $F: \mathbb{R}^m \to \mathbb{R}^m$ 
and $g: \mathbb{R}^m \rightarrow \mathbb{R}^{m\times m}$ are assumed to satisfy 
the usual Lipschitz continuity condition \cite{MattinglyST:JNA10}.

For appropriate functions $F$ and $g$, the stationary distribution, $\rho(\xb)$, 
of the It\^{o} diffusion \eqref{eq:itodiffusion} has a marginal distribution equal 
to the posterior distribution $p(\thetab | \Db)$ \cite{MaCF:NIPS15}. For example, 
denoting the unnormalized negative log-posterior as 
$U(\thetab) \triangleq -\log p(\thetab) - \sum_{i=1}^N \log p(\db_i | \thetab)$,
the stochastic gradient Langevin dynamic (SGLD) method \cite{WellingT:ICML11} 
is based on 1st-order Langevin dynamics, with $\xb = \thetab$, and
$F(\xb_t) = -\nabla_{\thetab} U(\thetab), \hspace{0.1cm} g(\xb_t) = \sqrt{2}\Ib_n$,
where $\Ib_n$ is the $n \times n$ identity matrix. The stochastic gradient Hamiltonian 
Monte Carlo (SGHMC) method \cite{ChenFG:ICML14} is based on 2nd-order Langevin 
dynamics, with $\xb = (\thetab, \qb)$, and 
$F(\xb_t)= \Big( \begin{array}{c}
		\qb \\
		-B \qb-\nabla_\thetab U(\thetab) \end{array} \Big),
	g(\xb_t) = \sqrt{2B}\Big( \begin{array}{cc}
		{\bf 0} & {\bf 0} \\
		{\bf 0} & \Ib_n \end{array} \Big)$
for a scalar $B > 0$; $\qb$ is an auxiliary variable known as the momentum \cite{ChenFG:ICML14,DingFBCSN:NIPS14}. 
Diffusion forms for other SG-MCMC algorithms, such as the stochastic gradient 
thermostat \cite{DingFBCSN:NIPS14} and variants with Riemannian information geometry \cite{PattersonT:NIPS13,MaCF:NIPS15,LICCC:AAAI16}, are defined similarly.

In order to efficiently draw samples from the continuous-time diffusion~\eqref{eq:itodiffusion}, 
SG-MCMC algorithms typically apply two approximations: $\RN{1}$) Instead of analytically 
integrating infinitesimal increments $\mathrm{d}t$, numerical integration over small step size $h$ 
is used to approximate the integration of the true dynamics. $\RN{2}$) Instead of working with the 
full gradient $\nabla_{\thetab} U(\thetab_{lh})$, a stochastic gradient $\nabla_{\thetab} \tilde{U}_l (\thetab_{lh})$, defined as
\vspace{-0.4cm}\begin{align}\label{eq:sg}
	\nabla_{\thetab}\tilde{U}_l(\thetab) \triangleq -\nabla_{\thetab}\log p(\thetab) - \frac{N}{J}\sum_{i=1}^J \nabla_{\thetab}\log p(\db_{\pi_i} | \thetab),
\end{align}
is calculated from a minibatch of size $J$, where $\{\pi_1, \cdots, \pi_J\}$ is a random subset of $\{1, \cdots, N\}$. 
Note that to match the time index $t$ in \eqref{eq:itodiffusion}, parameters have been and will be 
indexed by ``$lh$'' in the $l$-th iteration.

\section{Stochastic Gradient MCMC with Stale Gradients}\label{sec:dsgmcmc}

In this section, we extend SG-MCMC to the stale-gradient setting, commonly 
met in asynchronous distributed systems \cite{Dean2012,Chen2015mxnet,tensorflow2015-whitepaper}, 
and develop theory to analyze convergence properties.

\subsection{Stale stochastic gradient MCMC (S$^2$G-MCMC)}
  
The setting for S$^2$G-MCMC is the same as the standard SG-MCMC described above, except
that the stochastic gradient \eqref{eq:sg} is replaced with a stochastic 
gradient evaluated with outdated parameter $\thetab_{(l-\tau_l)h}$ instead of the
latest version $\thetab_{lh}$ (see Appendix~\ref{supp:dsgmcmc} for an example):
\begin{align}\label{eq:ssg}
	\nabla_{\thetab}\hat{U}_{\tau_l}(\thetab) \triangleq 
	-\nabla_{\thetab}\log p(\thetab_{(l-\tau_l)h}) - \frac{N}{J}\sum_{i=1}^J \nabla_{\thetab}\log p(\db_{\pi_i} | \thetab_{(l-\tau_l)h}),
\end{align}
where $\tau_l \in \mathbb{Z}^{+}$ denotes the {\em staleness} of the parameter used to 
calculate the stochastic gradient in the $l$-th iteration. A distinctive difference between
S$^2$G-MCMC and SG-MCMC is that stale stochastic gradients are no longer unbiased
estimations of the true gradients. This leads to additional challenges in developing 
convergence bounds, one of the main contributions of this paper.

\begin{wrapfigure}{R}{0.55\linewidth}\vspace{-0.8cm}
    \begin{minipage}{\linewidth}
	\begin{algorithm}[H]
		\caption{State update of SGHMC with the stale stochastic gradient $\nabla_{\thetab} \hat{U}_{\tau_l}(\thetab)$}\label{alg:sghmc}
		\begin{algorithmic}
			\STATE {\bf Input:} $\xb_{lh} = (\thetab_{lh}, \qb_{lh})$, $\nabla_{\thetab} \hat{U}_{\tau_l}(\thetab)$, $\tau_l$, $\tau$, $h$, $B$
			\STATE {\bf Output:} $\xb_{(l+1)h} = (\thetab_{(l+1)h}, \qb_{(l+1)h})$
			\IF {$\tau_l \le \tau$}
				\STATE Draw $\zetab_{l} \sim \Ncal(0, \Ib)$;
				\STATE $\qb_{(l+1)h} = (1 - Bh)\qb_{lh} - \nabla_{\thetab} \hat{U}_{\tau_l}(\thetab) h + \sqrt{2Bh}\zetab_{l}$;
				\STATE $\thetab_{(l+1)h} = \thetab_{lh} + \qb_{(l+1)h} h$;
			\ENDIF
		\end{algorithmic}
	\end{algorithm}
    \end{minipage}\vspace{-0.5cm}
\end{wrapfigure}

We assume a bounded staleness for all $\tau_l$'s, 
{\it i.e.}, $$\max_l \tau_l \leq \tau$$ for some constant $\tau$. As an example, 
Algorithm~\ref{alg:sghmc} describes the update rule
of the stale-SGHMC in each iteration with the Euler integrator, 
where the stale gradient $\nabla_{\thetab} \hat{U}_{\tau_l}(\thetab)$ 
with staleness $\tau_l$ is used.
 
\subsection{Convergence analysis} \label{sec:theory}

This section analyzes the convergence properties of the basic S$^2$G-MCMC;
an extension with multiple chains is discussed in Section~\ref{sec:mul-server}.
It is shown that the bias and MSE 
depend on the staleness parameter $\tau$, while the variance  is 
independent of it, yielding significant speedup in Bayesian distributed systems.

\paragraph{Bias and MSE}
In \cite{ChenDC:NIPS15}, the bias and MSE of the standard SG-MCMC algorithms with a 
$K$th order integrator were analyzed, where the order of an integrator reflects 
how accurately an SG-MCMC algorithm approximates the corresponding continuous 
diffusion. Specifically, if evolving $\xb_t$ with a numerical integrator using discrete
time increment $h$ induces an error bounded by $O(h^K)$, the integrator is called a $K$th order 
integrator, {\it e.g.}, the popular Euler method used in SGLD \cite{WellingT:ICML11} 
is a 1st-order integrator. In particular, \cite{ChenDC:NIPS15} proved 
the bounds stated in Lemma~\ref{lem:biasmse}.

\begin{lemma}[\cite{ChenDC:NIPS15}]\label{lem:biasmse}
	Under standard assumptions (see Appendix~\ref{sec:ass}), the bias and MSE 
	of SG-MCMC with a $K$th-order integrator at time $T = hL$ are bounded as:
	\begin{align*}
		\mbox{Bias: }&\left|\mathbb{E}\hat{\phi}_L - \bar{\phi}\right| = O\left(\frac{\sum_l \left\|\mathbb{E}\Delta V_l\right\|}{L} + \frac{1}{Lh} + h^K\right) \\
		\mbox{MSE: }&\mathbb{E}\left(\hat{\phi}_L - \bar{\phi}\right)^2 = O \left(\frac{\frac{1}{L}\sum_l\mathbb{E}\left\|\Delta V_l\right\|^2}{L} + \frac{1}{Lh} + h^{2K}\right)
	\end{align*}
\end{lemma}
Here $\Delta V_l \triangleq \mathcal{L} - \tilde{\mathcal{L}}_l$, 
where $\mathcal{L}$ is the generator of the It\^{o} 
diffusion \eqref{eq:itodiffusion} defined as
\begin{align}\label{eq:generator}
	\mathcal{L}f(\xb_t) \triangleq \lim_{h \rightarrow 0^{+}} \frac{\mathbb{E}\left[f(\xb_{t+h})\right] - f(\xb_t)}{h} 
	= \rbr{F(\xb_t) \cdot \nabla_{\xb} + \frac{1}{2}\left(g(\xb_t) g(\xb_t)^T\right)\!:\! \nabla_{\xb}\!\nabla^T_{\xb}} f(\xb_t)~,
\end{align}
for any compactly supported twice differentiable function $f: \mathbb{R}^n \rightarrow \mathbb{R}$,
$h\rightarrow 0^{+}$ means $h$ approaches zero along the positive real axis.
$\tilde{\mathcal{L}}_l$ is the same as $\mathcal{L}$ except using the stochastic gradient $\nabla \tilde{U}_{l}$ instead of the full gradient.

We show that the bounds of the bias and MSE of S$^2$G-MCMC share similar forms as SG-MCMC,
but with additional dependence on the staleness parameter. In addition to 
the assumptions in SG-MCMC \cite{ChenDC:NIPS15} (see details in Appendix~\ref{sec:ass}), 
the following additional assumption is imposed.
\begin{assumption}\label{ass:sto}
The noise in the stochastic gradients is well-behaved, such that: 1) the stochastic gradient is unbiased, {\it i.e.}, 
$\nabla_{\thetab} U(\thetab) = \mathbb{E}_{\xi} \nabla_{\thetab} \tilde{U}(\thetab)$ where $\xi$ denotes the random permutation over
$\{1, \cdots, N\}$; 2) the variance of stochastic gradient is bounded, {\it i.e.}, 
$\mathbb{E}_{\xi}\left\|U(\thetab) - \tilde{U}(\thetab)\right\|^2 \leq \sigma^2$; 3) the gradient function $\nabla_{\thetab} U$
is Lipschitz (so is $\nabla_{\thetab} \tilde{U}$), {\it i.e.}, $\left\|\nabla_{\thetab} U(\xb) - \nabla_{\thetab} U(\yb)\right\| \leq C \left\|\xb - \yb\right\|, \forall \xb, \yb$.
\end{assumption}

In the following theorems, we omit the assumption statement for conciseness.
Due to the staleness of the stochastic gradients, the term $\Delta V_l$ in 
S$^2$G-MCMC is equal to $\Lcal - \tilde{\Lcal}_{l-\tau_l}$, where $\tilde{\Lcal}_{l-\tau_l}$ 
arises from $\nabla_{\thetab} \hat{U}_{\tau_l}$. The challenge arises to bound 
these terms involving $\Delta V_l$.  To this end,
define $f_{lh} \triangleq \left\|\xb_{lh} - \xb_{(l-1)h}\right\|$, and $\psi$ 
to be a functional satisfying the 
\emph{Poisson Equation}\footnote{The existence of a nice $\psi$ is guaranteed in the elliptic/hypoelliptic SDE settings when $\xb$ is on a torus \cite{MattinglyST:JNA10}.}:
\vspace{-0.2cm}\begin{align}\label{eq:PoissonEq1}
	\frac{1}{L}\sum_{l=1}^L\mathcal{L} \psi(\xb_{lh}) =  \hat{\phi}_L - \bar{\phi}~.
\end{align}

\begin{theorem}\label{theo:bias}
	After $L$ iterations, the bias of
	S$^2$G-MCMC with a $K$th-order integrator is bounded, for some constant $D_1$ 
	independent of $\{L, h, \tau\}$, as:
	\begin{align*}
		\left|\mathbb{E}\hat{\phi}_L - \bar{\phi}\right| \leq D_1\left(\frac{1}{Lh} + M_1 \tau h + M_2h^K\right)~,
	\end{align*}
	where $M_1 \triangleq \max_l \left|\mathcal{L}f_{lh}\right| \max_l\left\|\mathbb{E}\nabla \psi(\xb_{lh})\right\| C$,
	$M_2 \triangleq \sum_{k=1}^K\frac{\sum_l\mathbb{E}\tilde{\Lcal}_l^{k+1}\psi(\xb_{(l-1)h})}{(k+1)! L}$
	are constants.
\end{theorem}

\begin{theorem}\label{theo:MSE}
	After $L$ iterations, 
	the MSE of S$^2$G-MCMC with a $K$th-order integrator is bounded, for some constant $D_2$ 
	independent of $\{L, h, \tau\}$, as:
	\begin{align*}
		\mathbb{E}\left(\hat{\phi}_L - \bar{\phi}\right)^2 \leq D_2\left(\frac{1}{Lh} + \tilde{M}_1\tau^2 h^{2} + \tilde{M}_2 h^{2K}\right)~,
	\end{align*}\vspace{-0.3cm}
	where constants $\tilde{M}_1 \triangleq \max_{l} \left\|\mathbb{E}\nabla\psi(\xb_{lh})\right\|^2 \max_l \left(\mathcal{L}f_{lh}\right)^2 C^2$,
	$\tilde{M}_2 \triangleq \mathbb{E}(\frac{\sum_l\tilde{\Lcal}_l^{K+1}\psi(\xb_{(l-1)h)}}{L(K+1)!})^2$.
\end{theorem}

The theorems indicate that both the bias and MSE depend on the staleness parameter $\tau$.
For a fixed computational time, this could possibly lead to unimproved bounds, compared to 
standard SG-MCMC, when $\tau$ is too large,
{\it i.e.}, the terms with $\tau$ would dominate, as is the case in the distributed system discussed in 
Section~\ref{sec:app}. Nevertheless, better bounds than standard 
SG-MCMC could be obtained if the decrease of $\frac{1}{Lh}$ is faster than the increase of the 
staleness in a distributed system.

\paragraph{Variance}
Next we investigate the convergence behavior of the variance,  
$\mbox{Var}(\hat{\phi}_L)\triangleq \mathbb{E}\left(\hat{\phi}_L - \mathbb{E} \hat{\phi}_L\right)^2$.
Theorem~\ref{theo:var} indicates the variance is independent of $\tau$,
hence a linear speedup in the decrease of variance is always achievable when stale gradients are 
computed in parallel. An example is discussed in the Bayesian distributed system
in Section~\ref{sec:app}.

\begin{theorem}\label{theo:var}
After $L$ iterations, 
the variance of S$^2$G-MCMC with a $K$th-order integrator is bounded, for some constant $D$, as:
\vspace{-0.3cm}\begin{align*}
	\mbox{Var}\left(\hat{\phi}_L\right) \leq D\left(\frac{1}{Lh} + h^{2K} \right)~.
\end{align*}
\end{theorem}

The variance bound is the same as for standard SG-MCMC, whereas
$L$ could increase linearly w.r.t.\! the number of workers in a distributed setting, yielding
significant variance reduction.
When optimizing the the variance bound w.r.t.\! $h$, we get an optimal variance bound
stated in Corollary~\ref{coro:opt_rates}.

\begin{corollary}\label{coro:opt_rates}
	In term of estimation variance, the optimal convergence rate of S$^2$G-MCMC with 
	a $K$th-order integrator is bounded as:
		$\mbox{Var}\left(\hat{\phi}_L\right) \le O\left(L^{-2K/(2K+1)}\right)$~.
\end{corollary}\vspace{-0.3cm}

In real distributed systems, the decrease of $1/Lh$ and increase of $\tau$, in the bias and
MSE bounds, would typically cancel, leading to the same bias and MSE level compared to standard 
SG-MCMC, whereas a linear speedup on the decrease of variance w.r.t. the number of workers 
is always achievable. More details are discussed in Section~\ref{sec:app}.

\subsection{Extension to multiple parallel chains}\label{sec:mul-server}
This section extends the theory to the setting with $S$ parallel chains, each independently 
running an S$^2$G-MCMC algorithm.
After generating samples from the $S$ chains, an aggregation step is needed to combine the
sample average from each chain, {\it i.e.}, $\{\hat{\phi}_{L_s}\}_{s=1}^M$, where 
$L_s$ is the number of iterations on chain $s$. For generality, we allow each 
chain to have different step sizes, {\it e.g.}, 
$(h_s)_{s=1}^S$. We aggregate the sample averages as
$\hat{\phi}_L^S \triangleq \sum_{s=1}^S \frac{T_s}{T} \hat{\phi}_{L_s}$,
where $T_s \triangleq L_s h_s$, $T \triangleq \sum_{s=1}^S T_s$.

Interestingly, with increasing $S$, using multiple chains does not seem to directly improve 
the convergence rate for the bias, but improves the MSE bound, as stated in Theorem~\ref{theo:mul-servers}.

\begin{theorem}\label{theo:mul-servers}
	Let $T_m \triangleq \max_l T_l$, $h_m \triangleq \max_l h_l$, $\bar{T} = T/S$,
	the bias and MSE of $S$ parallel
	S$^2$G-MCMC chains with a $K$th-order integrator are bounded, for some constants $D_1$ and $D_2$ independent of $\{L, h, \tau\}$, as:
	\begin{align*}
		\mbox{Bias: }&\left|\mathbb{E}\hat{\phi}_L^S - \bar{\phi}\right| \leq D_1\left(\frac{1}{\bar{T}} + \frac{T_m}{\bar{T}} \left(M_1 \tau h_s + M_2h_s^K\right)\right) \\
		\mbox{MSE: }&\mathbb{E}\left(\hat{\phi}_L^S - \bar{\phi}\right)^2 
		\leq D_2\left(\frac{1-1/\bar{T}}{T} + \frac{1}{\bar{T}^2} + \frac{T_m^2}{\bar{T}^2} \left(M_1^2\tau^2 h_s^{2} + M_2^2 h_s^{2K}\right)\right)~.
	\end{align*}
\end{theorem}

Assume that $\bar{T} = T/S$ is independent of the number of chains. As a result, using multiple chains
does not directly improve the bound for the bias\footnote{It means the bound does not directly relate
to low-order terms of $S$, though constants might be improved.}. However, for the MSE bound, 
although the last two terms are independent of $S$, the first term decreases linearly with respect to $S$
because $T = \bar{T}S$.
This indicates a decreased estimation variance with more chains. This matches the
intuition because more samples can be obtained with more chains in a given amount of time.

The decrease of MSE for multiple-chain is due to the decrease of the variance as stated in 
Theorem~\ref{theo:mul-servers-var}.

\begin{theorem}\label{theo:mul-servers-var}
	The variance of $S$ parallel S$^2$G-MCMC chains with a $K$th-order integrator is bounded, for some 
	constant $D$ independent of $\{L, h, \tau\}$, as:
	\begin{align*}
		\mathbb{E}\left(\hat{\phi}_L^S - \mathbb{E}\hat{\phi}_L^S\right)^2 
		\leq D\left(\frac{1}{T} + \sum_{s=1}^S \frac{T_s^2}{T^2} h_s^{2K}\right)~.
	\end{align*}
\end{theorem}

When using the same step size for all chains, Theorem~\ref{theo:mul-servers-var} gives 
an optimal variance bound of $O\left((\sum_s L_s)^{-2K/(2K+1)}\right)$, {\it i.e.} a linear 
speedup with respect to $S$ is achieved.

In addition, Theorem~\ref{theo:mul-servers} with $\tau = 0$ and $K = 1$ provides 
convergence rates for the distributed SGLD algorithm in \cite{AhnSW:ICML14}, {\it i.e.}, 
improved MSE and variance bounds compared to the single-server SGLD.

\section{Applications to Distributed SG-MCMC Systems}\label{sec:app}

Our theory for S$^2$G-MCMC is general, serving as a basic analytic tool for 
distributed SG-MCMC systems. We propose two simple Bayesian distributed systems
with S$^2$G-MCMC in the following.

\paragraph{Single-chain distributed SG-MCMC}
Perhaps the simplest architecture is an asynchronous distributed SG-MCMC system, 
where a server runs an S$^2$G-MCMC algorithm, with stale gradients computed asynchronously
from $W$ workers. The detailed operations of the server and workers are described in 
Appendix~\ref{supp:dsgmcmc}. 

With our theory, now we explain the convergence property of this simple distributed system
with SG-MCMC, {\it i.e.}, a linear speedup w.r.t. the number of workers on the decrease of
variance, while maintaining the same bias level. To this end, rewrite $L = W\bar{L}$ from 
Theorems~\ref{theo:bias} and \ref{theo:MSE}, where $\bar{L}$ 
is the average number of iterations on each worker. We can observe from the theorems that when 
$M_1 \tau h > M_2 h^K$ in the bias and 
$\tilde{M}_1 \tau^2 h^2 > \tilde{M}_2 h^{2K}$ in the MSE, the terms with $\tau$ dominate. 
Optimizing the bounds with respect to $h$ yields a bound of $O((\tau/W\bar{L})^{1/2})$ 
for the bias, and
$O((\tau/W\bar{L})^{2/3})$ for the MSE. In practice, we usually observe $\tau \approx W$, making $W$ in the optimal bounds cancels, {\it i.e.}, the same optimal
bias and MSE bounds as standard SG-MCMC are obtained, no theoretical speedup is achieved 
when increasing $W$. However, from Corollary~\ref{coro:opt_rates}, the variance is 
independent of $\tau$, thus a linear speedup on the variance bound can be always obtained when 
increasing the number of workers,
{\it i.e.}, the distributed SG-MCMC system convergences a factor of $W$ faster than 
standard SG-MCMC with a single machine. 
We are not aware of similar conclusions from optimization, because most of the research 
focuses on the convex setting, thus only variance (equivalent to MSE) is studied.

\paragraph{Multiple-chain distributed SG-MCMC}
We can also adopt multiple servers based on the multiple-chain 
setup in Section~\ref{sec:mul-server}, where each chain corresponds to one server. 
The detailed architecture is described in Appendix~\ref{supp:dsgmcmc}.
This architecture trades off communication cost with convergence rates. 
As indicated by Theorems~\ref{theo:mul-servers} and 
\ref{theo:mul-servers-var}, the MSE and variance bounds can be improved with more servers.
Note that when only one worker is associated with one server, we recover the setting of $S$
independent servers. Compared to the single-server architecture described above with 
$S$ workers, from Theorems~\ref{theo:bias}--\ref{theo:mul-servers-var}, 
while the variance bound is the same, the single-server arthitecture improves the bias and MSE 
bounds by a factor of $S$. 

\paragraph{More advanced architectures}
More complex architectures could also be designed to reduce 
communication cost,
for example, by extending the downpour \cite{Dean2012} and elastic SGD \cite{ZhangCL:NIPS15} 
architectures to the SG-MCMC setting. Their convergence properties can also be analyzed with our 
theory since they are essentially using stale gradients. We leave the detailed analysis for future work.

%
%

\section{Experiments}\label{sec:exp}

Our primal goal is to validate the theory, comparing with different 
distributed architectures and algorithms, such as \cite{TehHLVWLB:arxiv15,BardenetDH:arxiv15},
is beyond the scope of this paper.
We first use two synthetic experiments to validate the theory, then apply the distributed architecture
described in Section~\ref{sec:app} for Bayesian deep learning. 
To quantitatively describe the speedup
property, we adopt the  the {\em iteration speedup} \cite{LianHLL:NIPS15}, defined as:
$\mbox{{\em iteration speedup}} \triangleq \frac{\text{\#iterations with a single worker}}{\text{average \#iterations on a worker}}$,
where \# is the iteration count when the same level of precision is achieved.
This speedup best matches with the theory. We also consider the {\em time speedup}, defined as:
$\frac{\text{running time for a single worker}}{\text{running time for $W$ worker}}$,
where the running time is recorded at the same accuracy. It is affected
significantly by hardware, thus is not accurately consistent with the theory.

\subsection{Synthetic experiments}

\begin{wrapfigure}{R}{5.06cm}\vspace{-1cm}
	\centering
	\hspace{-0.37cm}\includegraphics[width=0.38\columnwidth]{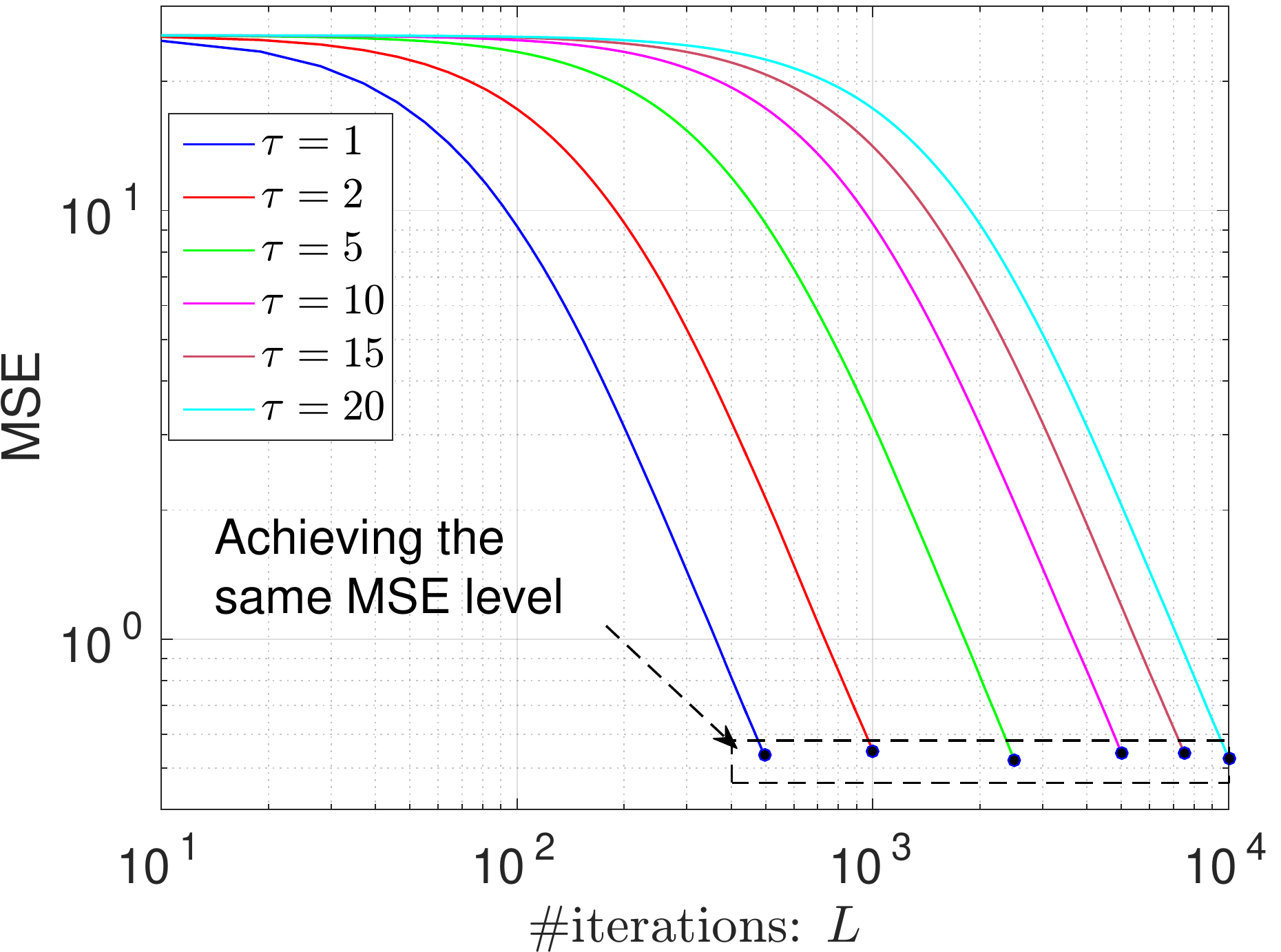}
	\vskip -0.1in
	\caption{MSE vs. \# iterations ($L =500\times\tau$) with increasing staleness $\tau$.
		Resulting in roughly the same MSE.}
	\label{fig:mse_staleness}
	\vskip -0.2in
\end{wrapfigure}

\paragraph{Impact of stale gradients}
A simple Gaussian model is used to verify the impact of stale gradients on the convergence accuracy,
with
$d_i \sim \mathcal{N}(\theta, 1), \theta \sim \mathcal{N}(0, 1)$. 1000 data samples $\{d_i\}$ 
are generated, with minibatches of size 10 to calculate stochastic gradients. The test function 
is $\phi(\theta) \triangleq \theta^2$. The distributed SGLD algorithm is adopted in this experiment.
We aim to verify that the optimal MSE bound $\propto \tau^{2/3}L^{-2/3}$, derived from 
Theorem~\ref{theo:MSE} and discussed in Section~\ref{sec:app} (with $W = 1$).
The optimal stepsize is $h = C \tau^{-2/3} L^{-1/3}$ for some constant $C$.
Based on the optimal bound, setting $L = L_0 \times \tau$ for some fixed $L_0$ and varying $\tau$'s 
would result in the same MSE, which is $\propto L_0^{-2/3}$. In the experiments we set 
$C = 1/30$, $L_0 = 500$, $\tau = \{1, 2, 5, 10, 15, 20\}$, 
and average over 200 runs to approximate the expectations in the MSE formula. As indicated in 
Figure~\ref{fig:mse_staleness}, approximately the same MSE's are obtained after $L_0\tau$ 
iterations for different $\tau$ values, consistent with the theory.
Note since the stepsizes are set to make end points of the curves reach the optimal MSE’s, 
the curves would not match the optimal MSE curves of $\tau^{2/3}L^{-2/3}$ in general, 
except for the end points, i.e., they are lower bounded by $\tau^{2/3}L^{-2/3}$.

%

\paragraph{Convergence speedup of the variance}
A Bayesian logistic regression model (BLR) is adopted to verify the 
variance convergence properties. We use the Adult 
dataset\footnote{\href{http://www.csie.ntu.edu.tw/ cjlin/libsvmtools/datasets/binary.html}{http://www.csie.ntu.edu.tw/ cjlin/libsvmtools/datasets/binary.html}.}, 
a9a, with 32,561 training samples and 16,281 test samples. 
The test function is defined as the standard logistic loss.
We average over 10 runs to estimate the expectation $\mathbb{E}\hat{\phi}_L$
in the variance. We use the single-server distributed architecture
in Section~\ref{sec:app}, with multiple workers computing stale gradients 
in parallel. We plot the variance versus
the average number of iterations on the workers ($\bar{L}$) and the running time in 
Figure~\ref{fig:var_iter} (a) and (b), respectively. 
 We can see that the variance drops faster with increasing number of workers. 
 To quantitatively relate these results to the theory, Corollary~\ref{coro:opt_rates} indicates
 that $\frac{L_1}{L_2} = \frac{W_1}{W_2}$, where $(W_i, L_i)_{i=1}^2$ means the number of workers
 and iterations at the same variance, {\it i.e.}, a linear speedup is achieved.
The {\em iteration speedup} and {\em time speedup} are plotted in Figure~\ref{fig:var_iter} (c),
showing that the {\em iteration speedup} approximately scales linearly worker numbers, 
consistent with Corollary~\ref{coro:opt_rates}; whereas the {\em time speedup} deteriorates when the 
worker number is large due to high system latency.

\begin{figure} 
	\centering
	\vspace{-0mm}
	\begin{tabular}{c c c}
		\hspace{-9mm}
		\begin{minipage}{0.32\linewidth}\vspace{0mm}
			\includegraphics[width=1\linewidth]{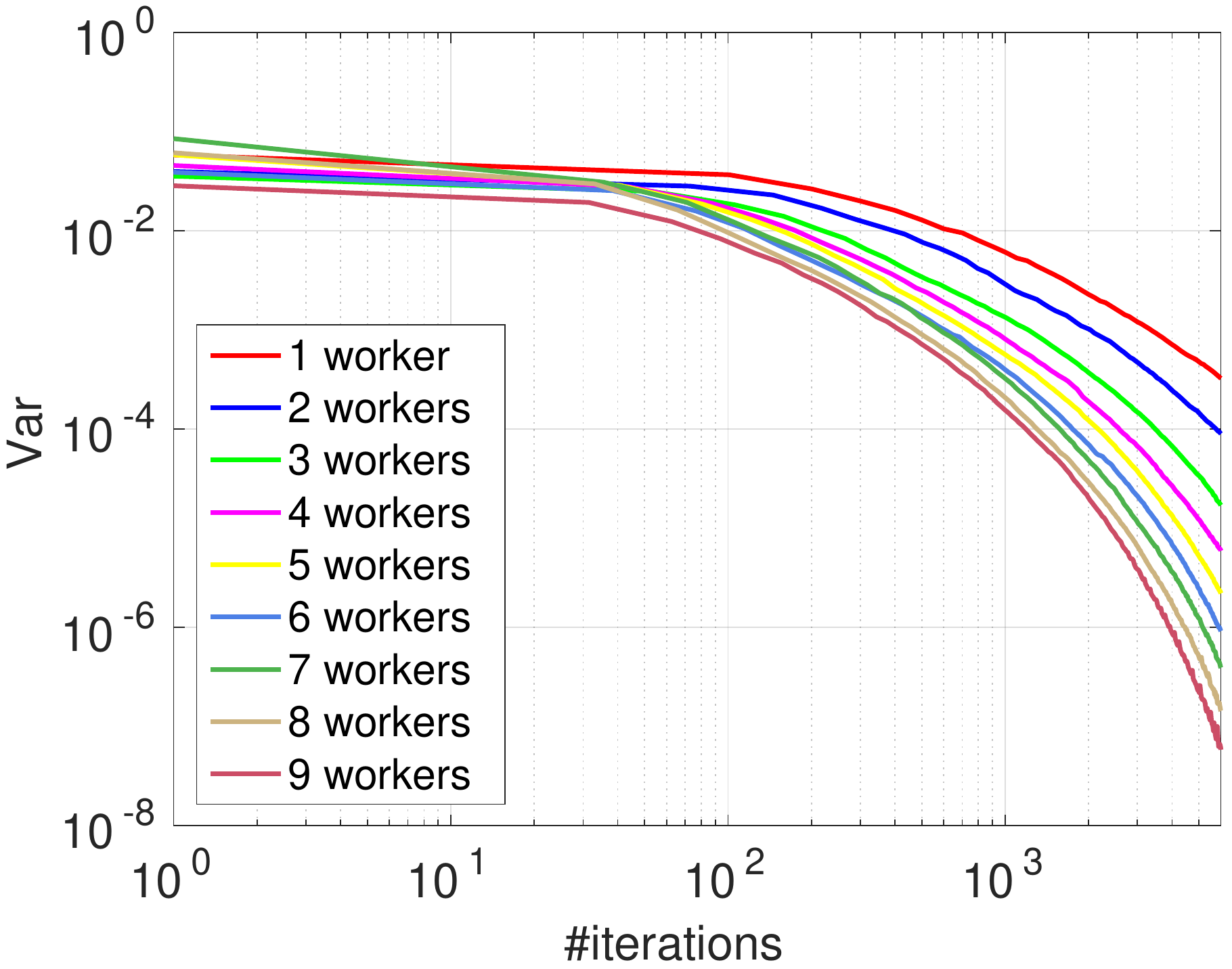} 
		\end{minipage}  &
		\hspace{-9mm}
		\begin{minipage}{0.32\linewidth}\vspace{0mm}\hspace{-3mm}
			\includegraphics[width=1\linewidth]{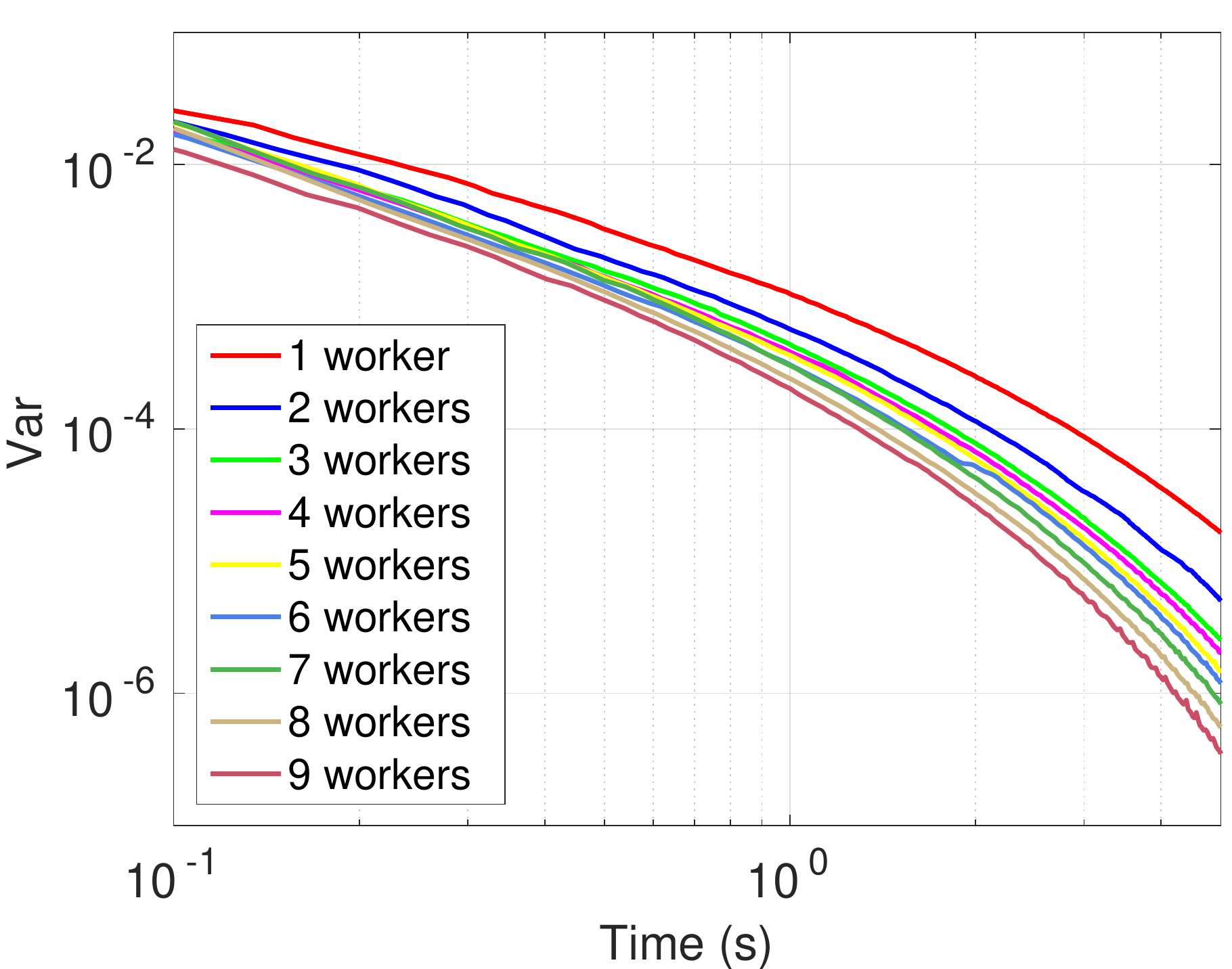} 
		\end{minipage} &
		\hspace{-9mm}
		\begin{minipage}{0.32\linewidth}\vspace{0mm}\hspace{-4mm}
			\includegraphics[width=1\linewidth]{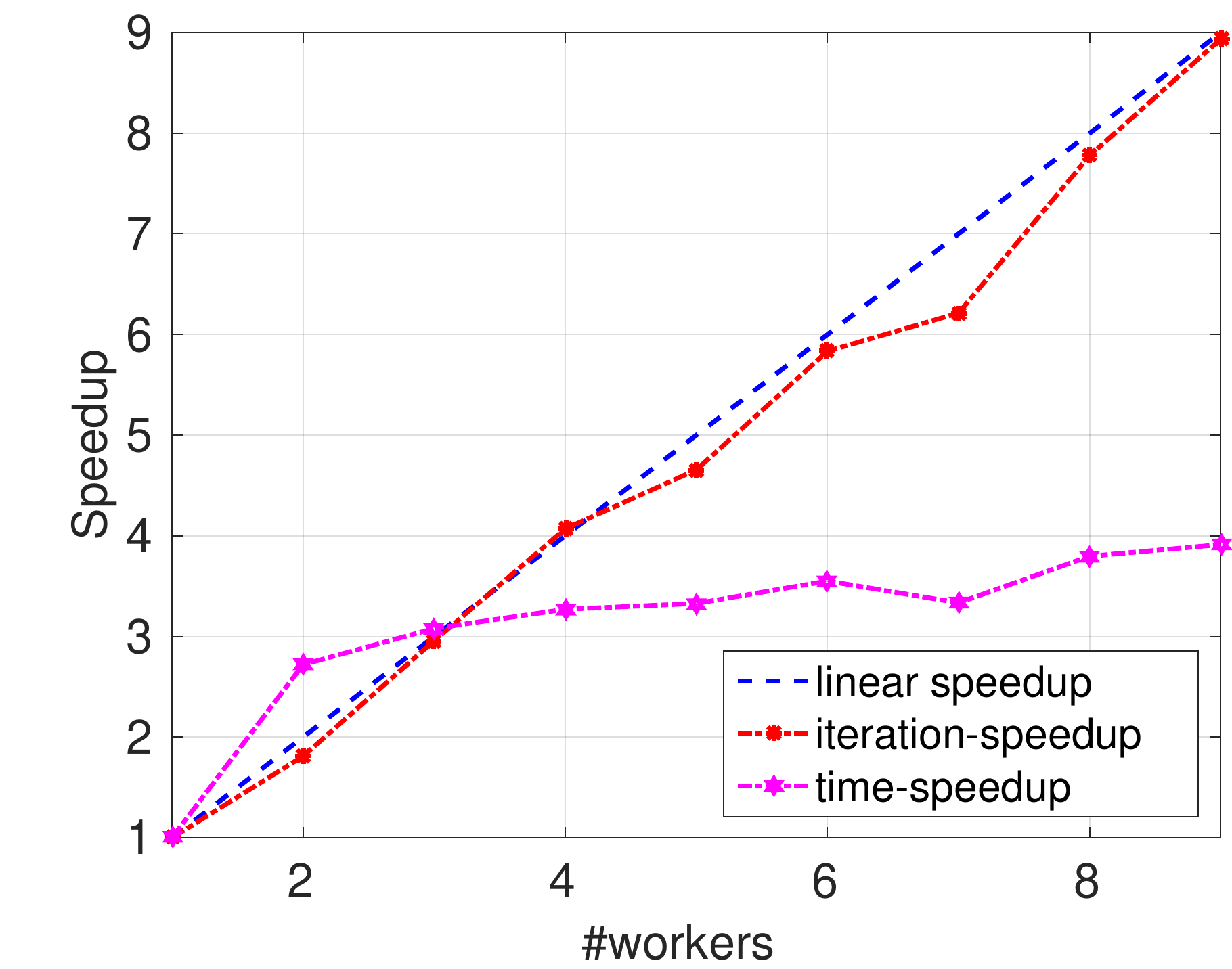} 
		\end{minipage} \\
		\vspace{-5mm}
		\begin{minipage}{0.32\linewidth}\vspace{-0mm}\centering (a) Variance vs. Iteration $\bar{L}$  \end{minipage} 
		& \begin{minipage}{0.32\linewidth}\vspace{0mm}\centering \hspace{-5mm}(b) Variance vs. Time \end{minipage}
		&	\begin{minipage}{0.32\linewidth}\vspace{-0mm}\centering \hspace{-7mm}(c) Speedup \end{minipage}	\\
	\end{tabular}\vspace{1mm}
		\caption{Variance with increasing number of workers.}
		\label{fig:var_iter}\vspace{-0.5cm}
\end{figure}

\subsection{Applications to deep learning}

We further test S$^2$G-MCMC on Bayesian learning of deep neural networks.
The distributed system is developed based on an MPI (message passing interface) 
extension of the popular Caffe package for deep learning \cite{jia2014caffe}. 
We implement the SGHMC algorithm, with the point-to-point communications 
between servers and workers handled by 
the MPICH library.
The algorithm is run 
on a cluster of five machines. Each machine is equipped with eight 3.60GHz Intel(R) Core(TM) i7-4790
CPU cores.

We evaluate S$^2$G-MCMC on the above BLR model and two deep convolutional neural networks (CNN). 
In all these models, zero mean and unit variance Gaussian priors are employed for the weights to capture weight 
uncertainties, an effective way to deal with overfitting \cite{BlundellCKW:ICML15}. We vary the 
number of servers $S$ among $\{1, 3, 5, 7\}$, and the number of workers for each server from 1 to 9.

\paragraph{LeNet for MNIST}
We modify the standard LeNet to a Bayesian setting for the MNIST 
dataset.
LeNet consists of 2 convolutional layers, 2 max pool layers and 2 ReLU nonlinear layers, 
followed by 2 fully connected layers \cite{KrizhevshySH:NIPS12}. The detailed specification 
can be found in Caffe. For simplicity, we use the default parameter setting specified in Caffe, with 
the additional parameter $B$ in SGHMC (Algorithm~\ref{alg:sghmc}) set to $(1 - m)$, where $m$ 
is the {\em moment} variable defined in the SGD algorithm in Caffe.

\paragraph{Cifar10-Quick net for CIFAR10}
The Cifar10-Quick net consists of 3 convolutional layers, 3 max pool layers and 3 ReLU 
nonlinear layers, followed by 2 fully connected layers. 
The CIFAR-10 dataset consists of 60,000 color images of size 32$\times$32 in 10 classes, 
with 50,000 for training and 10,000 for testing.
Similar to LeNet, default parameter setting specified in Caffe is used.

In these models, the test function is defined as the cross entropy of the 
{\em softmax} outputs $\{\ob_1, \cdots, \ob_N\}$ for test data 
$\{(\db_1, y_1), \cdots, (\db_N, y_N)\}$ with $C$ classes, {\it i.e.},
$\mbox{loss} = -\sum_{i=1}^N \ob_{y_i} + N \log \sum_{c=1}^C e^{\ob_c}$.
Since the theory indicates a linear speedup on the decrease of variance w.r.t. the number
of workers, this means for a single run of the models, the loss would converge faster to
its expectation with increasing number of workers. The following experiments
verify this intuition.

\subsubsection{Single-server experiments}

We first test the single-server architecture in Section~\ref{sec:app} on the three models. 
Because the expectations in the bias, MSE or variance are not analytically available in 
these complex models, we instead plot the {\em loss} versus {\em average number of iterations} 
($\bar{L}$ defined in Section~\ref{sec:app}) on each worker and the running {\em time} in 
Figure~\ref{fig:loss_single}. As mentioned above, faster decrease of the {\em loss} with 
more workers is expected.

For the ease of visualization, we only plot the results with $\{1, 2, 4, 6, 9\}$ workers;
more detailed results are provided in Appendix~\ref{sec:results}.
We can see that generally the loss decreases faster with increasing number of workers. 
In the CIFAR-10 dataset, the final losses of 6 and 9 workers are worst than the one 
with 4 workers. It shows that the accuracy of the sample average suffers from the 
increased staleness due to the increased number of workers. Therefore a smaller step 
size $h$ should be considered to maintain high accuracy when using a large number of workers. 
Note the 1-worker curves correspond to the standard SG-MCMC,
whose loss decreases much slower due to high 
estimation variance, though in theory it has the same level of bias as the single-server architecture 
for a given number of iterations (they will converge to the same accuracy).

\begin{figure*}
\begin{minipage}{0.325\linewidth}
	\includegraphics[width=\linewidth]{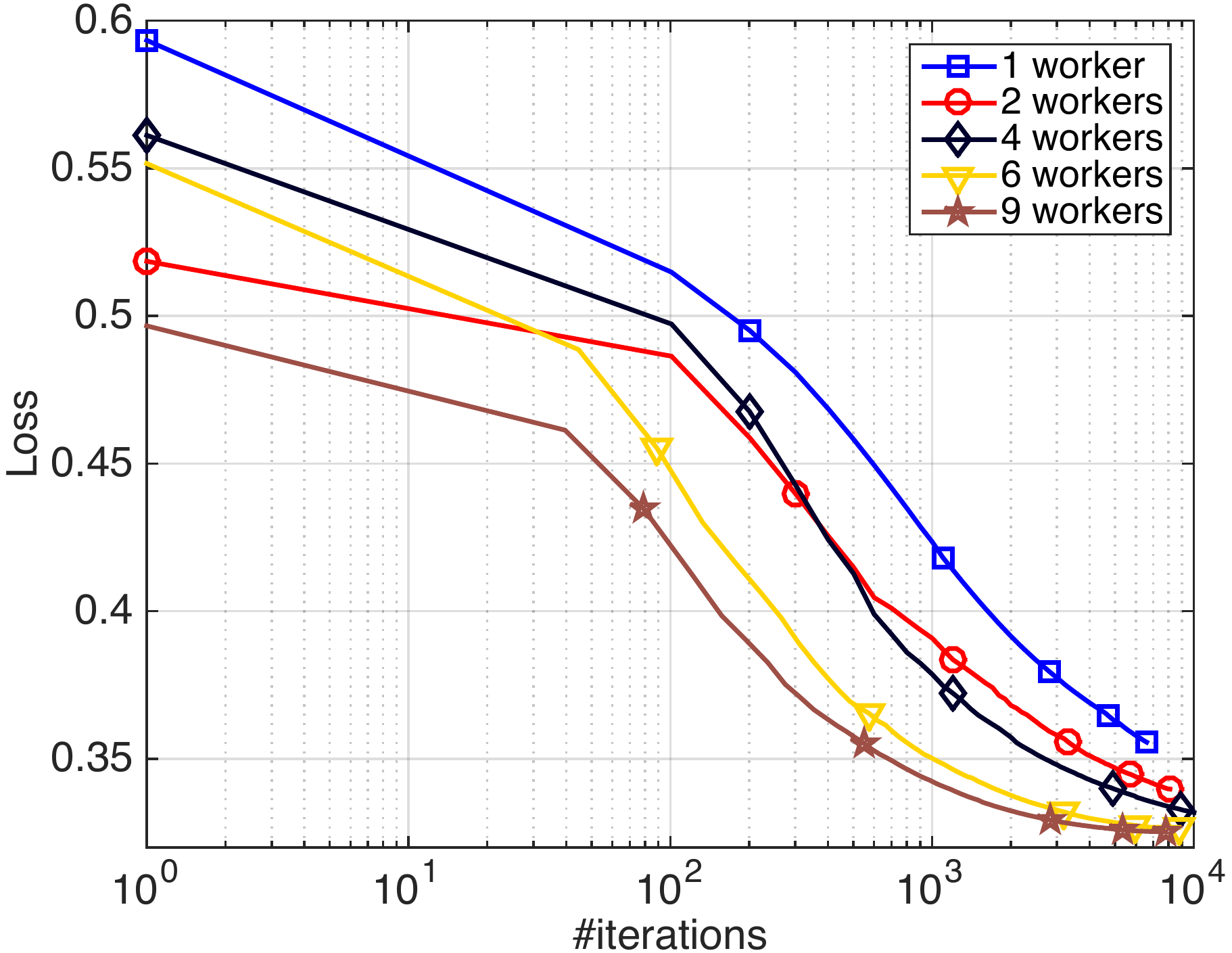}
\end{minipage}\hspace{2mm}
\begin{minipage}{0.325\linewidth}
	\includegraphics[width=1.0\linewidth]{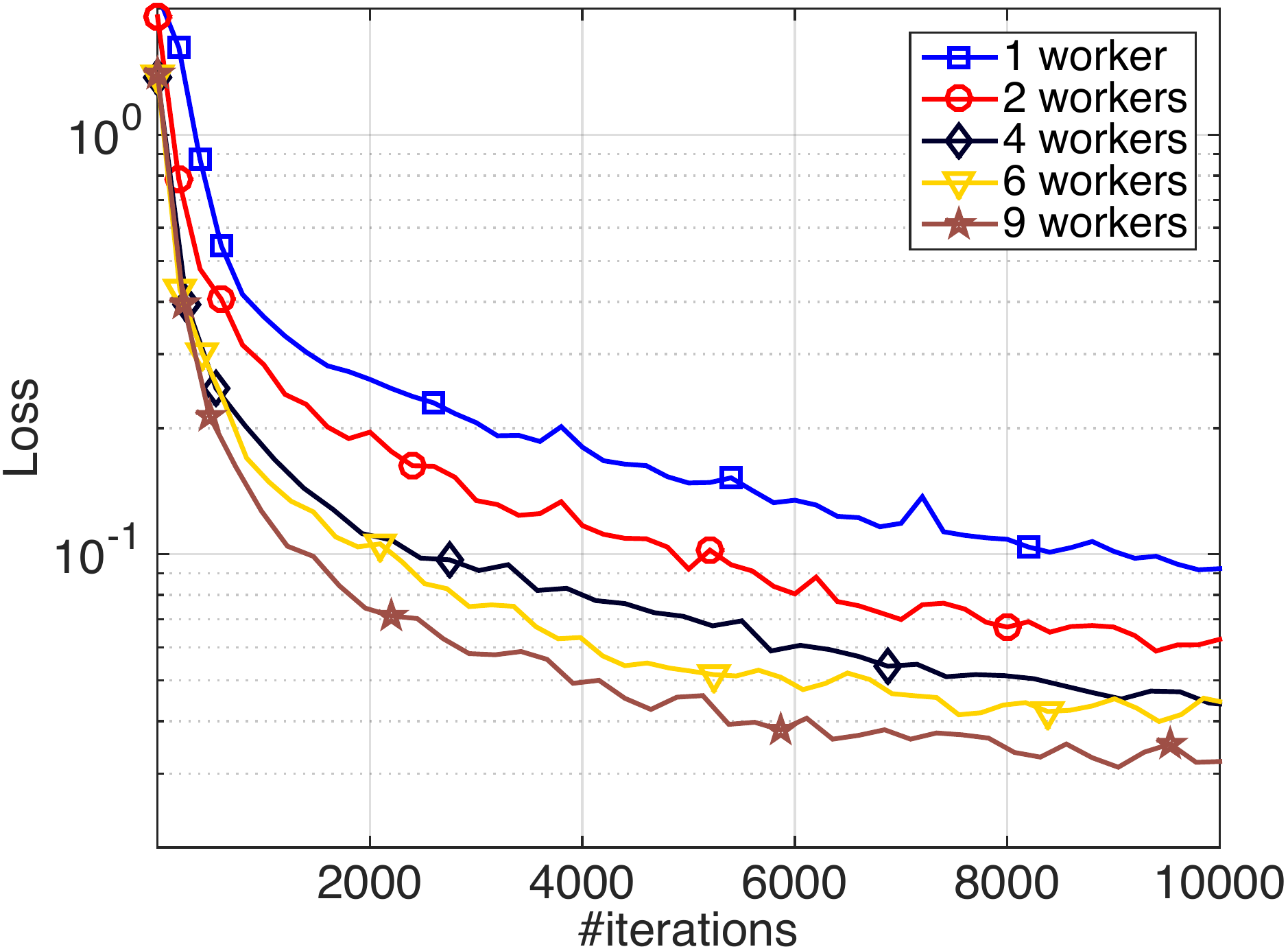}
\end{minipage}
\begin{minipage}{0.325\linewidth}\hspace{-1mm}
     \includegraphics[width=0.96\linewidth]{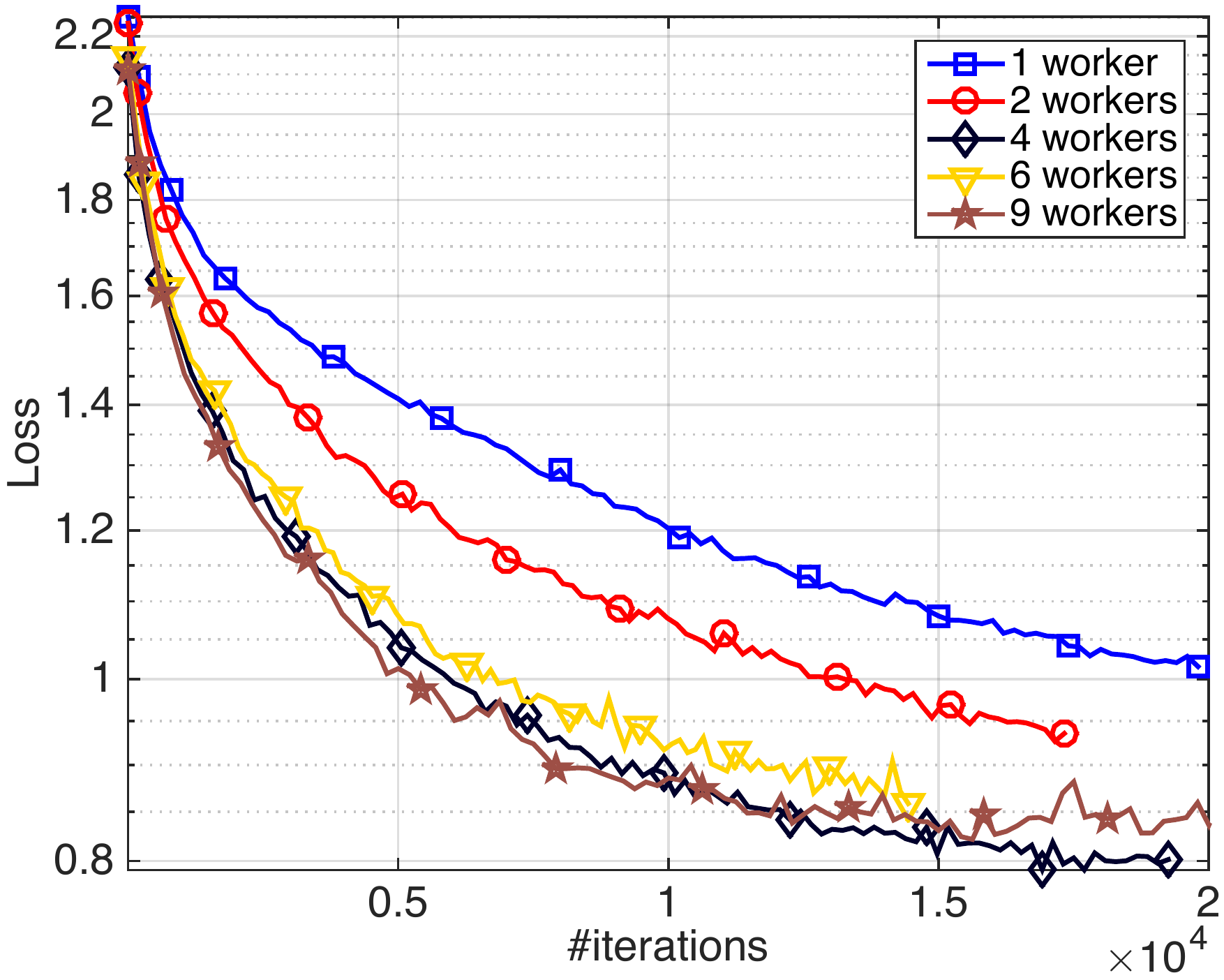}
\end{minipage}\\ \vspace{2mm}

\begin{minipage}{0.325\linewidth}
	\includegraphics[width=\linewidth]{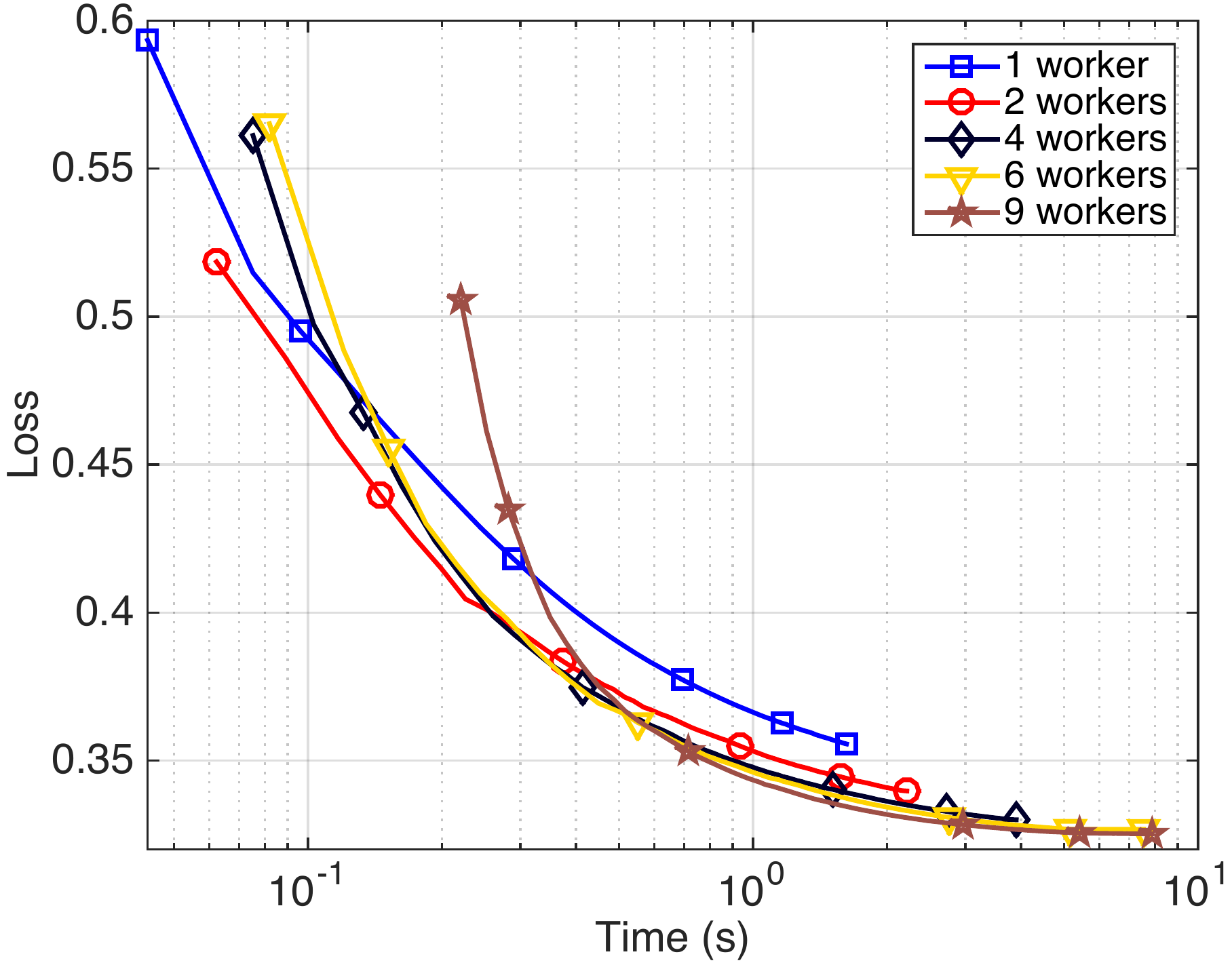}
\end{minipage}\hspace{2mm}
\begin{minipage}{0.325\linewidth}\vspace{0mm}\hspace{-1.5mm}
	\includegraphics[width=\linewidth]{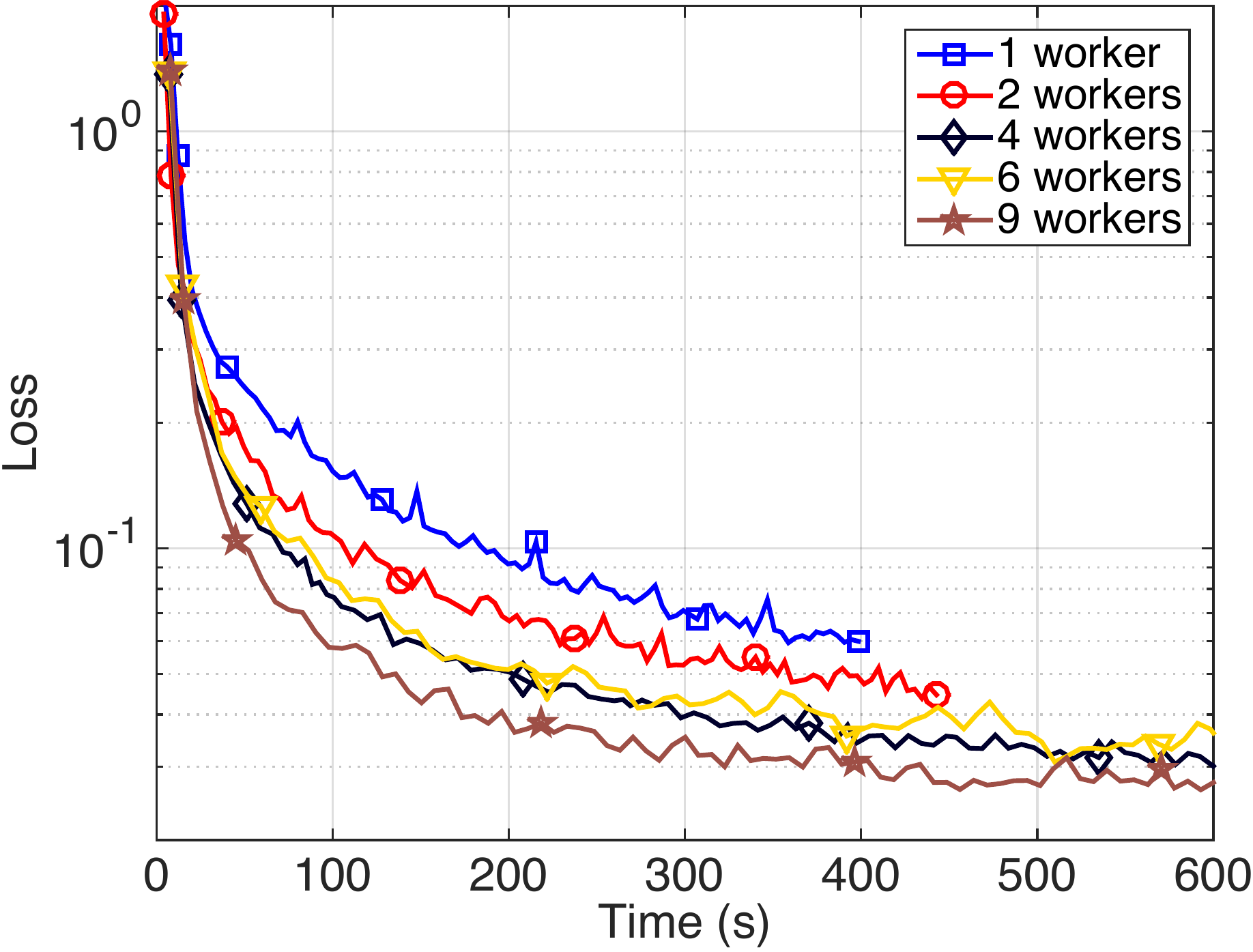}
\end{minipage}
\begin{minipage}{0.325\linewidth}\hspace{-1mm}
	\includegraphics[width=\linewidth]{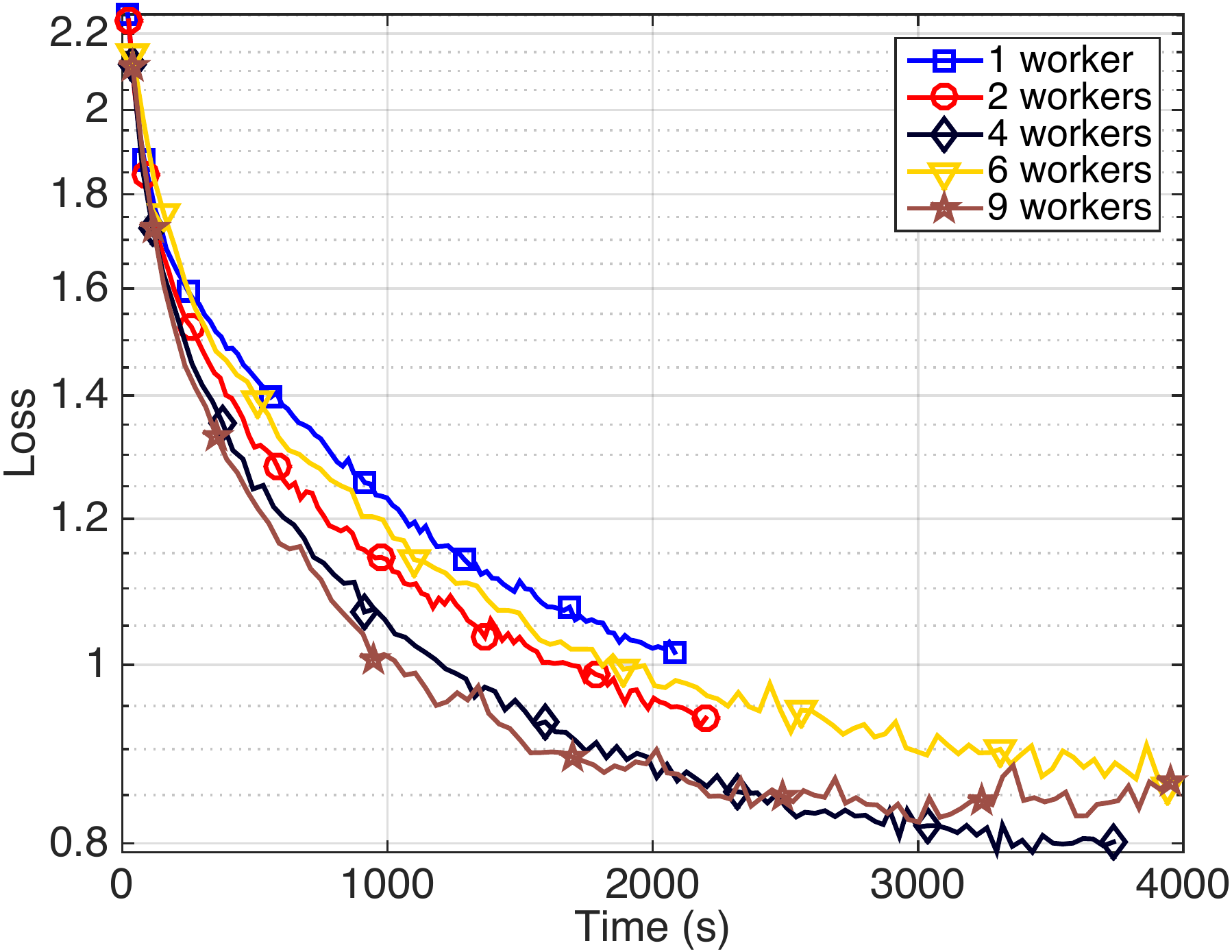}
\end{minipage}
\vskip -0.05in
\caption{Testing loss vs. \#workers. From left to right, each column corresponds to
the a9a, MNIST and CIFAR dataset, respectively. The loss is defined in the text.}
\label{fig:loss_single}
\vskip -0.1in
\end{figure*} 

\subsubsection{Multiple-server experiments}

Finally, we test the multiple-servers architecture on the same models. We use the same
criterion as the single-server setting to measure the convergence behavior.
The {\em loss} versus {\em average number of iterations} on each worker ($\bar{L}$
defined in Section~\ref{sec:app}) for the three datasets are plotted in 
Figure~\ref{fig:loss_mul}, where we vary the number of servers among 
$\{1, 3, 5, 7\}$, and use 2 workers for each server. The plots of {\em loss} versus
{\em time} and using different number of workers for each server are provided in the 
Appendix. We can see that in the simple BLR model, multiple servers do not seem to show 
significant speedup, probably due to the simplicity of the posterior, where the sample 
variance is too small for multiple servers to take effect; while in the more complicated deep neural networks, using more 
servers results in a faster decrease of the {\em loss}, especially in the MNIST dataset.
\begin{figure}[!htb]
\begin{center}
\begin{minipage}{0.325\linewidth}
	\centerline{\includegraphics[width=\linewidth]{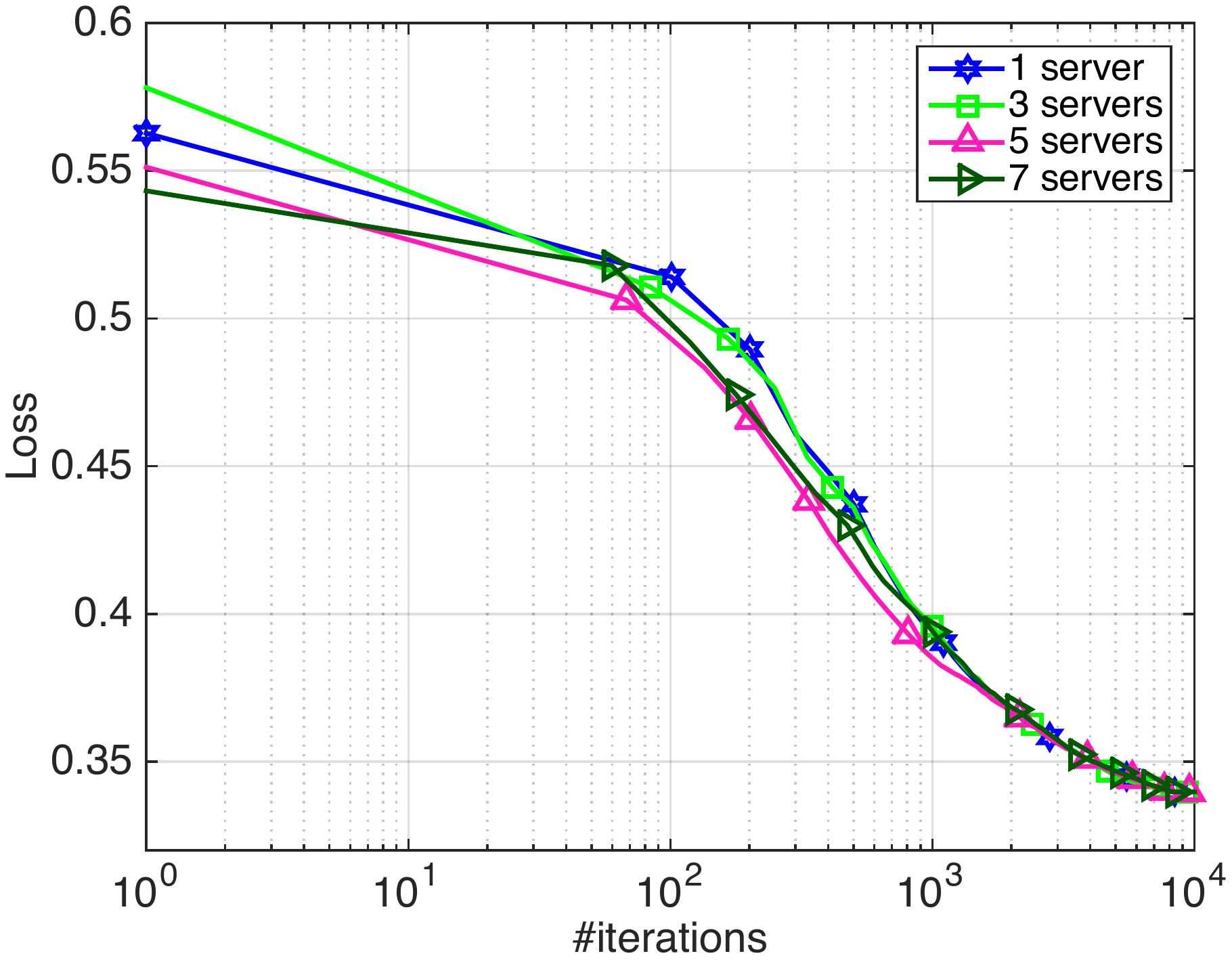}}
\end{minipage}
\begin{minipage}{0.325\linewidth}\vspace{-0mm}\hspace{1mm}
	\centerline{\includegraphics[width=1.00\linewidth,height=0.77\linewidth]{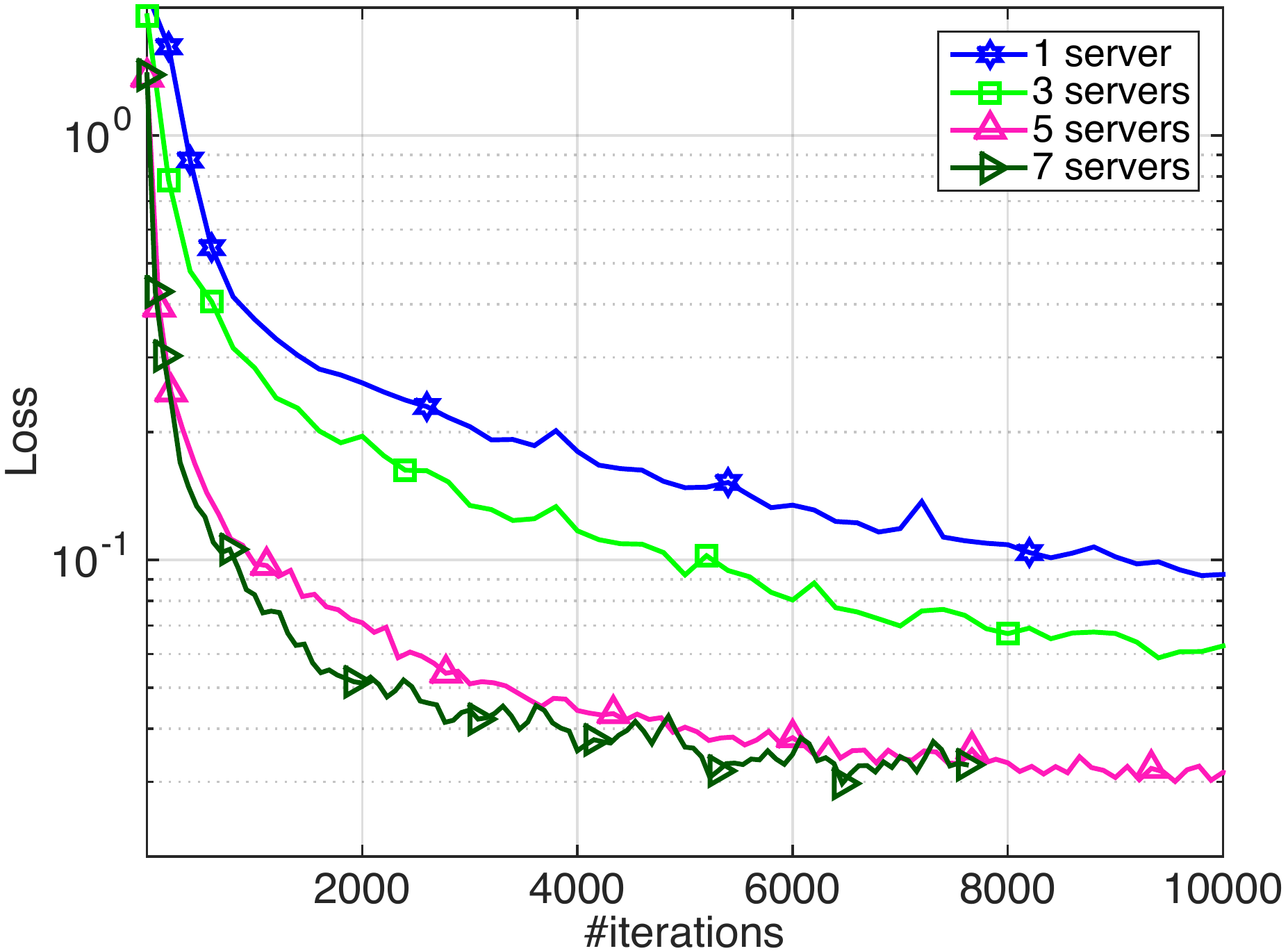}}
\end{minipage}
\begin{minipage}{0.325\linewidth}\vspace{0mm}\hspace{1mm}
	\centerline{\includegraphics[width=1\linewidth]{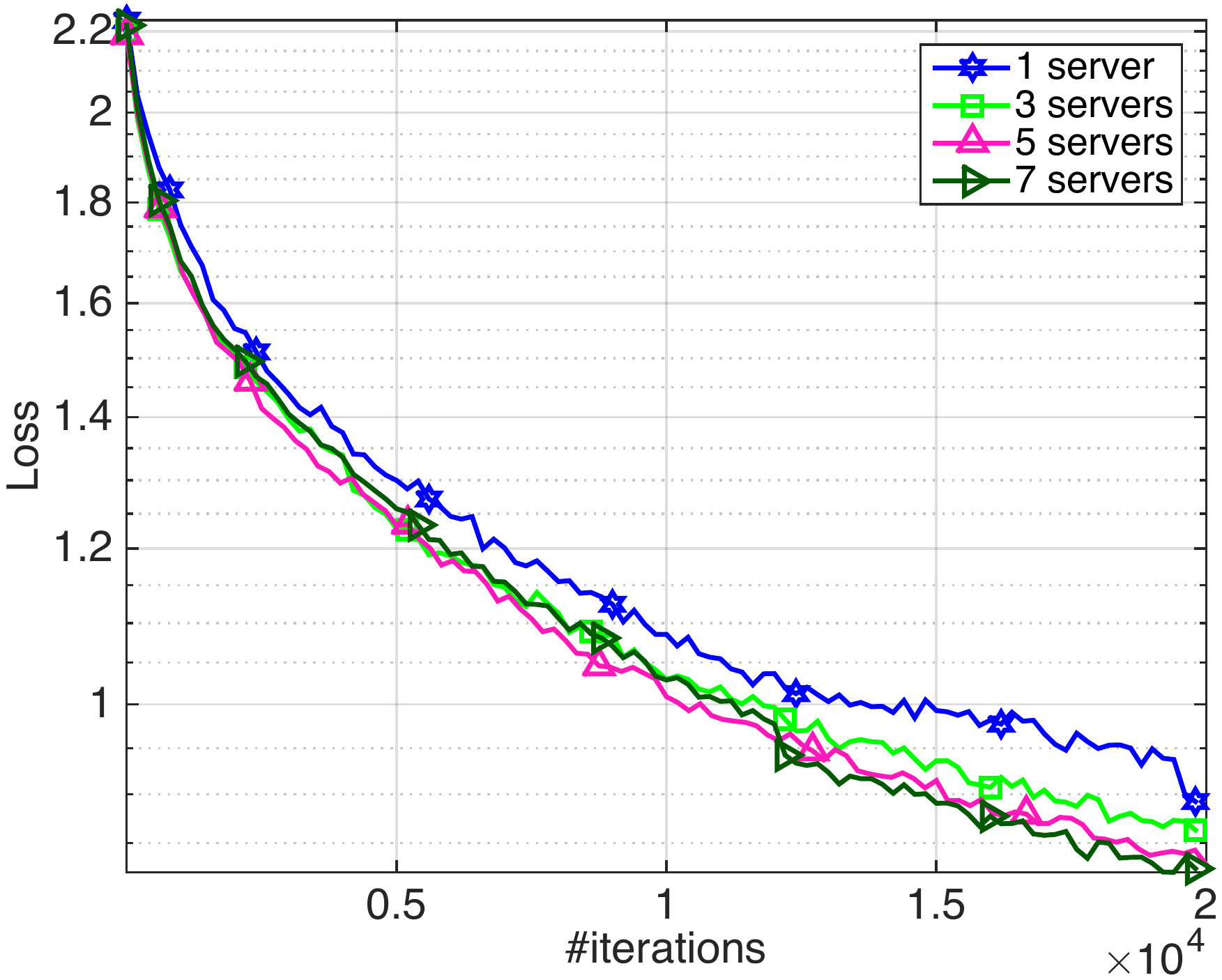}}
\end{minipage}

\caption{Testing loss vs. \#servers. From left to right, each column corresponds to
the a9a, MNIST and CIFAR dataset, respectively. The loss is defined in the text.}
\label{fig:loss_mul}
\end{center}
\vskip -0.3in
\end{figure}

\section{Conclusion}\label{sec:con}

We extend theory from standard SG-MCMC to the stale stochastic gradient
setting, and analyze the impacts of the staleness to the convergence behavior of an
S$^2$G-MCMC algorithm. Our theory reveals that the estimation variance is independent 
of the staleness, leading to a linear speedup w.r.t.\! the number of workers, although 
in practice little speedup in terms of optimal bias and MSE might be achieved due to 
their dependence on the staleness. We test our theory on a simple asynchronous 
distributed SG-MCMC system with two simulated examples and several deep neural 
network models. Experimental results verify the effectiveness and scalability of the 
proposed S$^2$G-MCMC framework. 

\paragraph{Acknowledgements}
Supported in part by ARO, DARPA, DOE, NGA, ONR and NSF.

\newpage
{\small
\setlength{\bibsep}{0.68pt}
\bibliographystyle{unsrt}
\bibliography{reference}
}

\newpage
\noindent\makebox[\linewidth]{\rule{\linewidth}{3.5pt}}
\begin{center}
\bf{\Large Supplementary Material for:\\ Stochastic Gradient MCMC with Stale Gradients}
\end{center}
\noindent\makebox[\linewidth]{\rule{\linewidth}{1pt}}

\begin{center}
	Changyou Chen$^\dag$~~~~~~~ Nan Ding$^\ddag$~~~~~~~ Chunyuan Li$^\dag$~~~~~~~ Yizhe Zhang$^\dag$ ~~~~~~~ Lawrence Carin$^\dag$ \\
	$^\dag$Dept. of Electrical and Computer Engineering, Duke University, Durham, NC, USA \\
	$^\ddag$Google Inc., Venice, CA, USA \\
	\texttt{$^\dag$\{cc448,cl319,yz196,lcarin\}@duke.edu; $^\ddag$dingnan@google.com}
\end{center}

\appendix

\section{A simple Bayesian distributed system based on S$^2$G-MCMC}\label{supp:dsgmcmc}

We provide the detailed architecture of the simple Bayesian distributed system
described in Section~\ref{sec:app}. We put the single-chain 
and multiple-chain distributed SG-MCMCs into a unified framework. 
Suppose there are $S$ servers and $W$ workers, the one with $S=1$ corresponds
to the single-chain distributed SG-MCMC, whereas the one with $S > 1$ corresponds
to the multiple-chain distributed SG-MCMC. The servers and workers are responsible 
for the following tasks: 
\begin{itemize}
	\item Each worker runs independently and communicates with a specific server. 
		They are responsible for computing the stochastic gradients\footnote{This is the most expensive part
		in an SG-MCMC algorithm.} of the parameter given by the server. Once the stochastic gradient is computed, 
		the worker sends it to its assigned server and receive a new parameter sample from the server.
	\item Each server independently maintains its own state vector and timestamp. At 
		the $l$-th timestamp\footnote{Each server is equipped with a timestamp 
		because they are independent with each other.}, it receives a stale stochastic gradient $\nabla_{\thetab} \hat{U}_{\tau_l}(\thetab) \triangleq \nabla_{\thetab} \tilde{U}(\thetab_{(l-\tau_l)h})$ from worker $w$, updates the state vector $\xb_{lh}$ to $\xb_{(l+1)h}$ and increments the timestamp, then sends the new parameter sample $\thetab_{(l+1)h}$ to worker $w$. 
\end{itemize}

The sending and receiving in the servers and workers are performed asynchronously, enabling minimum 
communication cost and latency between the servers and workers. 
At testing, all the samples from the servers are collected 
and applied to a test function. Apparently, the training time using multiple servers is basically 
the same as using a single server
because the sampling in different servers is independent.
Figure~\ref{fig:ps} depicts the architecture of the proposed Bayesian distributed framework. 
Algorithm~\ref{alg:dsgmcmc} details the algorithm on the servers and workers.

\begin{algorithm}[ht]
\caption{Asynchronous Distributed SG-MCMC}\label{alg:dsgmcmc}
\begin{algorithmic}
\STATE {\begin{center}\bfseries Server\end{center}}
\STATE {\bf Output:} $\cbr{\xb_{h}, \ldots, \xb_{Lh}}$
\STATE Initialize $\xb_{0} \in \RR^m$;
\STATE Send $\thetab_0$ to all assigned workers; 
\FOR {$l = 0, 1, \ldots, L-1$}
\STATE Receive a stale stochastic gradient $\nabla \tilde{U}_{(l-\tau_l)h}$ from a worker $w$.
\STATE Update $\xb_{lh}$ to $\xb_{(l+1)h}$ using $\nabla \tilde{U}_{(l-\tau_l)h}$. (*)
\STATE Send $\thetab_{(l+1)h}$ to the worker $w$.
\ENDFOR \\
\end{algorithmic}
\vspace{-0.3cm}
\noindent\makebox[\linewidth]{\rule{\linewidth}{1pt}}
\vspace{-0.5cm}
\begin{algorithmic}
\STATE {\begin{center}\bfseries Worker\end{center}}
\vspace{-0.2cm}
\REPEAT
\STATE Receive $\thetab_{lh}$ from server $s$.
\STATE Compute $\nabla \tilde{U}_{lh}$ with a minibatch.
\STATE Send $\nabla \tilde{U}_{lh}$ to server $s$.
\UNTIL{$stop$}
\end{algorithmic}
\end{algorithm}

The update rule (*) of the state vector in Algorithm~\ref{alg:dsgmcmc} depends on which SG-MCMC algorithm
is employed. For instance, Algorithm~\ref{alg:sghmc} describes the update rule of the SGHMC with a 1st-order Euler integrator.

\section{Assumptions}\label{sec:ass}

First, following \cite{MattinglyST:JNA10}, we will need to assume the corresponding SDE of SG-MCMC
to be either elliptic or hypoelliptic. The ellipticity/hypoellipticity describes whether the Brownian motion 
is able to spread over the whole parameter space. The SDE of the SGLD is elliptic, while for other SG-MCMC 
algorithms such as the SGHMC, the hypoellipticity assumption is usually reasonable. When the domain $\xb$ is on the torus, the ellipticity and hypoellipticity of an SDE guarantees the existence of a nice solution for the Poisson equation~\eqref{eq:PoissonEq1}. The assumption is
summarized in Assumption~\ref{ass:elliptic_ass}.

\begin{figure}[t!]
\vskip 0.1in
\begin{center}
\centerline{\includegraphics[width=0.8\columnwidth]{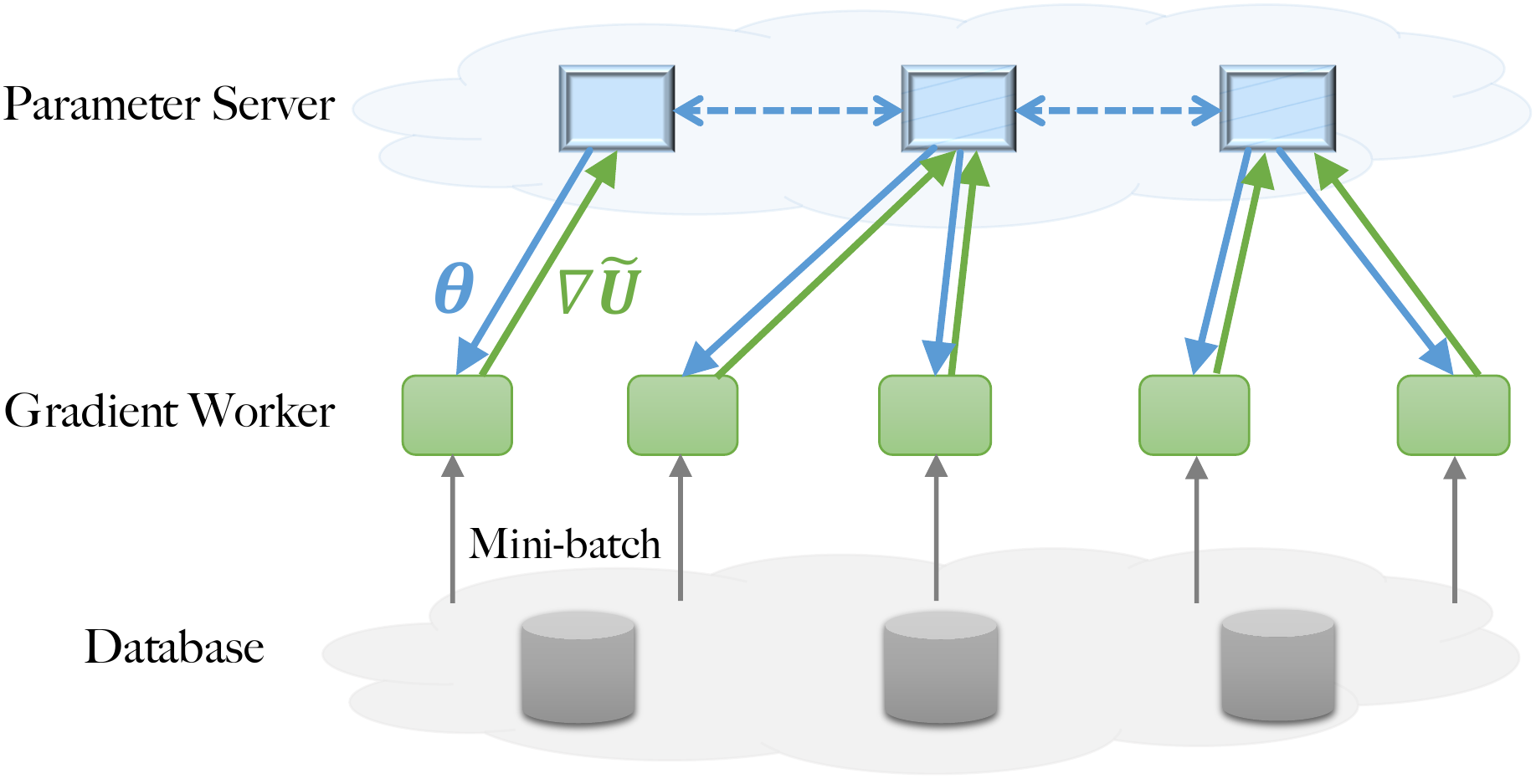}}
\caption{Architecture of the proposed Bayesian distributed framework. In the multi-server case,
the dash lines on the servers indicate a simple averaging operation for testing, otherwise
the servers are independent. Section~\ref{sec:mul-server} provides more details.}
\label{fig:ps}
\end{center}
\vskip -0.2in
\end{figure}

\begin{assumption}\label{ass:elliptic_ass}
The corresponding SDE of a SG-MCMC algorithm is either elliptic or hypoelliptic\footnote{The SDE of
the SGLD can be verified to be elliptic. For other SG-MCMC algorithms such as the SGHMC, the hypoellipticity
assumption is usually reasonable, see \cite{MattinglyST:JNA10} on how to verify hypoellipticity of an SDE.}.
\end{assumption}

When $\xb$ is extended to the domain of $\mathbb{R}^p$ for some integer $p > 0$, we need some assumptions on the solution of the Poisson equation \eqref{eq:PoissonEq1}.
Note \eqref{eq:PoissonEq1} can be equivalently written in an integration form \cite{VollmerZT:arxiv15} 
using It\^{o}'s formula:
\begin{align}\label{eq:PoissonEq2}
	&\frac{1}{t}\int_{0}^t \phi(\xb_s) \mathrm{d}s - \bar{\phi} \\
	=& \frac{1}{t}\left(\psi(\xb_t) - \psi(\xb_0)\right)
	- \frac{1}{t}\int_0^t \nabla \psi(\xb_s)\cdot g(\xb_s) \mathrm{d}\mathcal{\wb}_s~. \nonumber
\end{align} 

Intuitively, $\psi$ needs to be bounded if the discrepancy between $\hat{\phi}_L$ 
and $\bar{\phi}$ were to be bounded. This is satisfied if the SDE is defined in a
bounded domain \cite{MattinglyST:JNA10}. In the unbounded domain as for SG-MCMC algorithms, it turns out the following
boundedness assumptions on $\psi$ suffice \cite{ChenDC:NIPS15}.

\begin{assumption}\label{ass:assumption1}
1) $\psi$ and its up to 3rd-order derivatives, $\mathcal{D}^k \psi$, are bounded by a
function $\mathcal{V}$, {\it i.e.}, 
$\|\mathcal{D}^k \psi\| \leq C_k\mathcal{V}^{p_k}$ for $k=(0, 1, 2, 3)$, $C_k, p_k > 0$. 2)
the expectation of $\mathcal{V}$ on $\{\xb_{lh}\}$ is bounded: $\sup_l \mathbb{E}\mathcal{V}^p(\xb_{lh}) < \infty$.
3) $\mathcal{V}$ is smooth such that 
$\sup_{s \in (0, 1)} \mathcal{V}^p\left(s\xb + \left(1-s\right)\yb\right) \leq C\left(\mathcal{V}^p\left(\xb\right) + \mathcal{V}^p\left(\yb\right)\right)$, $\forall \xb, \yb, p \leq \max\{2p_k\}$ for some $C > 0$.
\end{assumption}

Furthermore, in our proofs the expectation of a function under a diffusion needs to be expanded in
a Taylor expansion style,
{\it e.g.}, $\mathbb{E}\phi(\xb_t) = \sum_{i=0}^\ell \frac{t^i}{i!}\mathcal{L}^i\phi(\xb_0) 
+ t^{\ell+1}r_{\ell, F, \phi}(\xb_0)$ by using Kolmogorov's backward equation. To ensure the remainder
term $r_{\ell, F, \phi}(\xb_0)$ to be bounded, it suffices to make the following assumption on the 
smoothness and boundedness of $F(\xb)$ \cite{VollmerZT:arxiv15,ChenDC:NIPS15}.
\begin{assumption}\label{ass:Fx}
	$F(\xb)$ is infinitely differentiable with bounded derivatives of any order; and $|F(\xb)| \leq A (1 + |\xb|^s)$
	for some integer $s > 0$ and $A > 0$.
\end{assumption}

\section{Notation}

For simplicity, we will simplify some notation used in the proof as follows:
\begin{align*}
	&\nabla_{\thetab}\tilde{U}_{l}(\thetab_{lh}) \triangleq \nabla_{\thetab} \tilde{U}_{lh} \triangleq \tilde{G}_{lh} \\
	&\nabla_{\thetab} U_{l}(\thetab_{lh}) \triangleq \nabla_{\thetab} U_{lh} \triangleq G_{lh} \\
	&\psi(\Xb_{lh}) \triangleq \psi_{lh}
\end{align*}

\section{Proof of Theorem~\ref{theo:bias}}

In S$^2$G-MCMC, for the $l$-th iteration, suppose a stochastic gradient with a staleness $\tau_l$ is used,
{\it e.g.}, $\tilde{G}_{(l - \tau_l)h}$. First, we will bound the difference between
$\tilde{G}_{(l - \tau_l)h}$ and the stochastic gradient at the $l$-th iteration $\tilde{G}_{lh}$,
by using the Lipschitz property of $\tilde{G}_{lh}$, with the following lemma.

\begin{lemma}\label{lem:bound_g}
Let $f_{lh} \triangleq \left\|\xb_{lh} - \xb_{(l-1)h}\right\|$,
the expected difference between $\tilde{G}_{(l - \tau_l)h}$ and $\tilde{G}_{lh}$
is bounded by:
\begin{align}
	&\left\|\mathbb{E}\left(\tilde{G}_{(l - \tau_l)h} - \tilde{G}_{lh}\right)\right\| =  \max_{i=l-\tau_l}^{l-1} \left|\mathcal{L}_if_{ih}\right| C \tau h + O(h^2), \label{bound_g2}
\end{align}
where the expectation is taken over the randomness of the SG-MCMC algorithm, {\it e.g.}, the
randomness from stochastic gradients and the injected Gaussian noise.
\end{lemma}

\begin{proof}
Note the randomness of $\tilde{G}_{lh}$ comes from two sources, the injected Gaussian noise
and the stochastic gradient noise. We denote the expectations with respect to these two randomness
as $\mathbb{E}_{\zeta}$ and $\mathbb{E}_{g}$, respectively. The whole expectation thus can be 
decomposed as $\mathbb{E} = \mathbb{E}_{\zeta} \mathbb{E}_g$.

Applying the Lipschitz property of $\tilde{G}_{lh}$, we have
%
\begin{allowdisplaybreaks}
\begin{align*}
	&\left\|\mathbb{E}\left(\tilde{G}_{(l - \tau_l)h} - \tilde{G}_{lh}\right)\right\| = \left\|\mathbb{E}_{\zeta}\left(G_{(l - \tau_l)h} - G_{lh}\right)\right\| \\
	&\leq \mathbb{E}_{\zeta}\left\|\left(G_{(l - \tau_l)h} - G_{lh}\right)\right\| \\
	&\leq C\mathbb{E}_{\zeta}\left\|\left(\thetab_{(l - \tau_l)h} - \theta_{lh}\right)\right\| \\
	&\leq C\mathbb{E}_{\zeta}\left\|\sum_{i = l - \tau_l}^{l-1} \left(\thetab_{(ih)} - \thetab_{(i+1)h}\right)\right\| \\
	&\leq C \sum_{i = l - \tau_l}^{l-1} \mathbb{E}_{\zeta}\left\|\left(\thetab_{(ih)} - \thetab_{(i+1)h}\right)\right\| \\
	&\leq C \sum_{i = l - \tau_l}^{l-1} \mathbb{E}_{\zeta}\left\|\xb_{(i+1)h} - \xb_{ih}\right\|\\
\end{align*}
From the definition of $K$th-order integrator, {\it i.e.}, $\mathbb{E}_{\zeta}f(\xb_{lh}) = e^{\tilde{\mathcal{L}}_l h}f(\xb_{(l-1)h}) + O(h^{K+1})$,
if we let 
$$f(\xb_{lh}) = \left\|\xb_{lh} - \xb_{(l-1)h}\right\| \triangleq f_{lh}~,$$ 
where $\xb_{(l-1)h}$ is the starting point in the $l$-th iteration, 
and note that 
$$f(\xb_{(l-1)h}) = 0~.$$ 
We have
\begin{align}
	&C \sum_{i = l - \tau_l}^{l-1} \mathbb{E}_{\zeta}\left\|\xb_{(i+1)h} - \xb_{ih}\right\| \triangleq C \sum_{i = l - \tau_l}^{l-1} \mathbb{E}_{\zeta} f(\xb_{lh}) \\
	\leq& C \sum_{i = l - \tau_l}^{l-1} \left(e^{\mathcal{L}_i h} f(\xb_{(i-1)h}) + O(h^{K+1})\right)\\
	\leq& C  \sum_{i = l - \tau_l}^{l-1} \left|\mathcal{L}_if_{ih}\right| h + O(h^2) \label{eq:x_dif2} \\
	\leq& \max_{i=l-\tau_l}^{l-1} \left|\mathcal{L}_i f_{ih}\right| C \tau h + O(h^2)~, \nonumber
\end{align}
where \eqref{eq:x_dif2} is obtained by expanding the exponential operator and the assumption that
the high order terms are bounded.
\end{allowdisplaybreaks}
\end{proof}

Now we proceed to prove Theorem~\ref{theo:bias}. The basic technique follows \cite{ChenDC:NIPS15},
thus we skip some derivations for some steps.

\begin{proof}[Proof of Theorem~\ref{theo:bias}]

Before the proof, let us first define some notation. First, define the operator $\Delta V_l$ for each $l$ as a differential operator as
for any function $\psi$:
\begin{align*}
	\Delta V_l \psi \triangleq \left(\tilde{G}_{l - \tau_l} - G_l\right) \cdot \nabla_{\pb} \psi~.
\end{align*}

Second, define the local generator, $\tilde{\mathcal{L}}_l$, for an It\^{o} diffusion, where the true gradient in \eqref{eq:itodiffusion} 
is replaced with the stochastic gradient from the $l$-th iteration, {\it i.e.}, $\tilde{\mathcal{L}}_lf(\Xb_t) \triangleq$
\begin{align*}
	\rbr{\tilde{F}_l(\xb_t) \cdot \nabla + \frac{1}{2}\left(\sigma(\xb_t) \sigma(\xb_t)^T\right)\!:\! \nabla \nabla^T} f(\xb_t)~,
\end{align*} 
for a compactly supported twice differentiable function $f$, where $\tilde{F}_l$ is the same as $F$ but with
the full gradient $G_{lh}$ replaced with the stochastic gradient $\tilde{G}_{lh}$. Based on these definitions,
we have
\begin{align*}
	\tilde{\mathcal{L}}_l = \mathcal{L} + \Delta V_l~.
\end{align*}

Following \cite{ChenDC:NIPS15}, for an SG-MCMC with a $K$th-order integrator, and a test function $\phi$,
we have: 
\begin{align} \label{eq:split_flow1}
	\mathbb{E}&[\psi(\xb_{lh})] = \left(\mathbb{I} + h\tilde{\Lcal}_l\right) \psi(\xb_{(l-1)h}) \\
	&+ \sum_{k=2}^K\frac{h^k}{k!}\tilde{\Lcal}_l^k\psi(\xb_{(l-1)h}) + O\left(\frac{h^{K+1}}{(K+1)!}\tilde{\Lcal}_l^{K+1}\psi_{(l-1)h}\right)~, \nonumber
\end{align}
where $\mathbb{I}$ is the identity map. Sum over $l = 1, \cdots, L$ in \eqref{eq:split_flow1}, 
take expectation on both sides, and use the relation $\tilde{\Lcal}_l = \Lcal + \Delta V_l$ 
to expand the first order term. We obtain
\begin{align*}
	\sum_{l=1}^{L}\mathbb{E}&[\psi(\xb_{lh})] = \psi(\xb_0) + \sum_{l=1}^{L-1} \mathbb{E}[\psi(\xb_{lh})] \\
	&+ h\sum_{l=1}^{L} \mathbb{E}[\mathcal{L}\psi(\xb_{(l-1)h})]
	+ h\sum_{l=1}^L \mathbb{E}[\Delta V_l \psi(\xb_{(l-1)h})] \\
	&+ \sum_{k=2}^K\frac{h^k}{k!}\sum_{l=1}^L \EE[\tilde{\Lcal}_l^k\psi(\xb_{(l-1)h})] \\
	&+ O\left(\frac{h^{K+1}}{(K+1)! }\sum_l\mathbb{E}\tilde{\Lcal}_l^{K+1}\psi_{(l-1)h}\right).
\end{align*}
Divide both sides by $Lh$, use the Poisson equation \eqref{eq:PoissonEq1}, and reorganize terms. We have:
\begin{align}\label{eq:expansion11}
	&\mathbb{E}[\frac{1}{L}\sum_l\phi(\xb_{lh}) - \bar{\phi}] = \frac{1}{L}\sum_{l=1}^{L} \mathbb{E}[\mathcal{L}\psi(\xb_{(l-1)h})] \\
	=&\frac{1}{Lh}\left(\mathbb{E}[\psi(\xb_{lh})] - \psi(\xb_0)\right)
	- \frac{1}{L}\sum_l \mathbb{E}[\Delta V_l\psi(\xb_{(l-1)h})] \nonumber\\
	&- \sum_{k=2}^K\frac{h^{k-1}}{k!L}\sum_{l=1}^L \EE[\tilde{\Lcal}_l^k\psi(\xb_{(l-1)h})] 
	+ O\left(\frac{h^{K}}{(K+1)! L}\sum_l\mathbb{E}\tilde{\Lcal}_l^{K+1}\psi_{(l-1)h}\right) \nonumber
\end{align}
According to \cite{ChenDC:NIPS15}, the term $\sum_l \mathbb{E}[\tilde{\Lcal}_l^{k} \psi(\xb_{(l-1)h})]$ is bounded by
$\sum_l \mathbb{E}[\tilde{\Lcal}_l^{k} \psi(\Xb_{(l-1)h})]$
\begin{align}\label{eq:expansion23}
	= O\left(\frac{1}{h} + h^{K-k+1}\sum_l\mathbb{E}\tilde{\Lcal}_l^{K+1}\psi_{(l-1)h}\right)~,
\end{align}
Substituting \eqref{eq:expansion23} into \eqref{eq:expansion11}, after simplification, we have:
$\mathbb{E}\left(\frac{1}{L}\sum_l\phi(\xb_{lh}) - \bar{\phi}\right)$
\begin{align*}
	=&\frac{1}{Lh}\underbrace{\left(\mathbb{E}[\psi(\xb_{lh})] - \psi(\xb_0)\right)}_{C_1}
	- \underbrace{\frac{1}{L}\sum_l \mathbb{E}[\Delta V_l\psi(\xb_{(l-1)h})]}_{C_2} \\
	&- \sum_{k=2}^KO\left(\frac{h^{k-1}}{Lh} + \frac{h^{K}}{L}\sum_l\frac{1}{k!}\mathbb{E}\tilde{\Lcal}_l^{K}\psi_{(l-1)h}\right) 
	+ \frac{h^{K}}{(K+1)! L}\sum_l\mathbb{E}\tilde{\Lcal}_l^{K+1}\psi_{(l-1)h}~,
\end{align*}
According to the assumption, the term $C_1$ is bounded. 
For term $C_2$, according to the Cauchy--Schwarz inequality, we have
\begin{align*}
	&\left|C_2\right| = \frac{1}{L} \left| \sum_l \mathbb{E}\left(\tilde{G}_{(l - \tau_l)h} - G_{lh}\right) \cdot \mathbb{E}\nabla \psi_{(l-1)h}\right| \\
	\leq& \frac{1}{L} \sum_l \left| \mathbb{E}\left(\tilde{G}_{(l - \tau_l)h} - G_{lh}\right) \cdot \mathbb{E}\nabla \psi_{(l-1)h}\right| \\
	\leq&\frac{1}{L}\sum_l \left\|\mathbb{E}\left(\tilde{G}_{(l - \tau_l)h} - G_{lh}\right)\right\|\left\|\mathbb{E}\nabla \psi_{(l-1)h}\right\| \\
	\leq&\frac{1}{L}\sum_l \left(\left\|\mathbb{E}\left(\tilde{G}_{(l - \tau_l)h} - \tilde{G}_{lh}\right)\right\| + \left\|\mathbb{E}\left(\tilde{G}_{lh} - G_{lh}\right)\right\|\right) 
	\left\|\mathbb{E}\nabla \psi_{(l-1)h}\right\| \\
	=&\frac{1}{L} \sum_l \left\|\mathbb{E}\left(\tilde{G}_{(l - \tau_l)h} - \tilde{G}_{lh}\right)\right\|\left\|\mathbb{E}\nabla \psi_{(l-1)h}\right\|
\end{align*}

Applying \eqref{bound_g2} from Lemma~\ref{lem:bound_g}, we have
\begin{align*}
	\left|C_2\right| &\leq \frac{1}{L} \sum_l \left(\max_{i=l-\tau_l}^l \left\|\mathcal{L}_{i}\right\| \left\|\mathbb{E}\nabla \psi_{lh}\right\| C\tau_l h\right) \\
	&\leq \max_l \left\|\mathcal{L}_{l}\right\| \max_l\left\|\mathbb{E}\nabla \psi_{lh}\right\| C \tau h~.
\end{align*}

As a result, collecting low order terms, the bias can be expressed as:
\begin{align}
	&\left|\mathbb{E}\hat{\phi} - \bar{\phi}\right| = \left|\mathbb{E}\left(\frac{1}{L}\sum_l\phi(\xb_{lh}) - \bar{\phi}\right)\right| \nonumber\\
	=& \left|\frac{C_1}{Lh} - C_2 + h^{K}\sum_{k=1}^K\frac{1}{(k+1)! L}\sum_l\mathbb{E}\tilde{\Lcal}_l^{k+1}\psi_{(l-1)h}\right|~. \nonumber\\
\end{align}
As a result, there exists some constant $D_1$ independent of $(L, h, \tau)$, such that
\begin{align}
	&\left|\mathbb{E}\hat{\phi} - \bar{\phi}\right| 
	\leq D_1\left|\frac{1}{Lh}\right| + \left|C_2\right| + \left|M_1 \tau h + \left|M_2 h^K\right| \right|\label{eq:final_bound}\\
	=& D_1\left(\frac{1}{Lh} + M_1 \tau h + M_2h^K\right)~,\nonumber
\end{align}
where $M_1 \triangleq \max_l \left\|\mathcal{L}_{l}\right\| \max_l\left\|\mathbb{E}\nabla \psi_{lh}\right\| C$,
$M_2 \triangleq \sum_{k=1}^K\frac{1}{(k+1)! L}\sum_l\mathbb{E}\tilde{\Lcal}_l^{k+1}\psi_{(l-1)h}$.
\eqref{eq:final_bound} follows by substituting the inequality for $C_2$ above. 
This completes the proof.
\end{proof}

\section{Proof of Theorem~\ref{theo:MSE}}

\begin{proof}
Similar to the proof of Theorem~\ref{theo:bias}, we first expand $\mathbb{E}\psi_{lh}$ using the property
of $K$th-order integrator as
\begin{align*}
	\sum_{l=1}^L&\mathbb{E}\left(\psi(\xb_{lh})\right) = \sum_{l=1}^L \psi(\xb_{(l-1)h}) + h\sum_{l=1}^L \mathcal{L}\psi(\xb_{(l-1)h}) \\
	&+ h\sum_{l=1}^L \Delta V_l\psi(\xb_{(l-1)h})
	+ \sum_{k=2}^K\frac{h^k}{k!}\sum_{l=1}^L\tilde{\mathcal{L}}_l^k \psi(\xb_{(l-1)h})\\
	& + O\left(\frac{h^{K+1}}{(K+1)!}\sum_l\tilde{\Lcal}_l^{K+1}\psi_{(l-1)h}\right)~.
\end{align*}
Substituting the Poisson equation \eqref{eq:PoissonEq1} into the above equation, dividing both sides by $Lh$ 
and rearranging related terms arrives
\begin{align}
	\hat{\phi} -& \bar{\phi} = \frac{1}{Lh}\left(\mathbb{E}\psi(\xb_{Lh}) - \psi(\xb_0)\right) \label{eq:expansion12}\\
	-& \frac{1}{Lh}\sum_{l=1}^{L}\left(\mathbb{E}\psi_{(l-1)h} - \psi_{(l-1)h}\right)
	- \frac{1}{L}\sum_{l=1}^L \Delta V_l\psi_{(l-1)h} \nonumber\\
	-& \sum_{k=2}^K\frac{h^{k-1}}{2L}\sum_{l=1}^L\tilde{\mathcal{L}}_l^k \psi(\xb_{(l-1)h}) 
	+ O\left(\frac{h^{K}}{L(K+1)!}\sum_l\tilde{\Lcal}_l^{K+1}\psi_{(l-1)h}\right) \nonumber
\end{align}
Taking square on both sides, we have there exists some positive constant $D$, such that
\begin{align}\label{eq:mse1}
	&\left(\hat{\phi} - \bar{\phi}\right)^2\leq D\left(\underbrace{\frac{\left(\mathbb{E}\psi_{Lh} - \psi_0\right)^2}{L^2h^2}}_{A_1} + \underbrace{\frac{1}{L^2h^2}\sum_{l=1}^L\left(\mathbb{E}\psi_{(l-1)h} - \psi_{(l-1)h}\right)^2}_{A_2} \right.\nonumber\\
	+& \left.\underbrace{\left(\frac{1}{L}\sum_{l=1}^L \Delta V_l\psi_{(l-1)h}\right)^2}_{A_3} + \underbrace{\sum_{k=2}^K\frac{h^{2(k-1)}}{k!L^2}\left(\sum_{l=1}^L\tilde{\mathcal{L}}_l^k \psi_{(l-1)h}\right)^2}_{A_4} 
	+ \underbrace{\left(\frac{\sum_l\tilde{\Lcal}_l^{K+1}\psi_{(l-1)h}}{L(K+1)!}\right)^2h^{2K}}_{A_5}\right) 
\end{align}
After taking expectation, we have
\begin{align*}
\mathbb{E}\left(\hat{\phi} - \bar{\phi}\right)^2 \leq C\left(\mathbb{E}A_1 + \mathbb{E}A_2 + \mathbb{E}A_3 + \mathbb{E}A_4 + \mathbb{E}A_5\right)
\end{align*}
$A_1$ is easily bounded by the assumption that $\|\psi\| \leq V^{p_0} < \infty$. From the proof of Theorem~3 in
\cite{ChenDC:NIPS15}, $A_2$ and $A_4$ are also bounded, which are summarized in Lemma~\ref{lem:A_2_A_3}.
\begin{lemma}\label{lem:A_2_A_3}
The terms $\mathbb{E}A_2$ and $\mathbb{E}A_4$ are bounded by:
\begin{align*}
	\mathbb{E}A_2 &= O\left(\frac{1}{Lh}\right) \\
	\mathbb{E}A_4 &= O\left(\frac{1}{Lh} + h^{2K}\sum_{k=2}^K\frac{1}{Lk!}\sum_{l}\tilde{\Lcal}_l^{k+1}\psi_{(l-1)h}\right)~.
\end{align*}
\end{lemma}
We are left to show a bound for $\mathbb{E}A_3$. First we have
\begin{align*}
&\mathbb{E}A_3 = \mathbb{E}\left(\frac{1}{L}\sum_{l=1}^L \Delta V_l\psi_{(l-1)h}\right)^2 \\
=&\mathbb{E}\left(\frac{1}{L}\sum_{l=1}^L \left(\tilde{G}_{(l-\tau_l)h} - G_{lh}\right)\cdot \nabla_{\pb} \psi_{(l-1)h}\right)^2 \\
=&\frac{1}{L^2}\sum_{i=1}^L\sum_{j=1}^L \mathbb{E} \left[\left(\tilde{G}_{(i-\tau_i)h} - G_{ih}\right)\cdot \nabla_{\pb} \psi_{(i-1)h} 
\left(\tilde{G}_{(j-\tau_j)h} - G_{jh}\right)\cdot \nabla_{\pb} \psi_{(j-1)h}\right]
\end{align*}
Using the Cauchy--Schwartz inequality, we have
\begin{align*}
\leq& \frac{1}{L^2}\sum_{i=1}^L\sum_{j=1}^L \left\|\mathbb{E} \left(\tilde{G}_{(i-\tau_i)h} - G_{ih}\right)\right\| 
\left\|\mathbb{E} \left(\tilde{G}_{(j-\tau_j)h} - G_{jh}\right)\right\| 
\left\|\mathbb{E}\nabla\psi_{(i-1)h}\right\| \left\|\mathbb{E}\nabla\psi_{(j-1)h}\right\| \\
\leq& \frac{1}{L^2}\sum_{i=1}^L\sum_{j=1}^L \left(\left\|\mathbb{E} \left(\tilde{G}_{(i-\tau_i)h} - \tilde{G}_{ih}\right)\right\| + \left\|\mathbb{E} \left(\tilde{G}_{ih} - G_{ih}\right)\right\|\right) \\
&\left(\left\|\mathbb{E} \left(\tilde{G}_{(j-\tau_j)h} - \tilde{G}_{jh}\right)\right\| + \left\|\mathbb{E} \left(\tilde{G}_{jh} - G_{jh}\right)\right\|\right) 
\left\|\mathbb{E}\nabla\psi_{(i-1)h}\right\| \left\|\mathbb{E}\nabla\psi_{(j-1)h}\right\| \\
=& \frac{1}{L^2}\sum_{i=1}^L\sum_{j=1}^L \left\|\mathbb{E} \left(\tilde{G}_{(i-\tau_i)h} - \tilde{G}_{ih}\right)\right\| 
\left\|\mathbb{E} \left(\tilde{G}_{(j-\tau_j)h} - \tilde{G}_{jh}\right)\right\| 
\left\|\mathbb{E}\nabla\psi_{(i-1)h}\right\| \left\|\mathbb{E}\nabla\psi_{(j-1)h}\right\|
\end{align*}
Applying \eqref{bound_g2} from Lemma~\ref{lem:bound_g}, we have
\begin{align*}
\mathbb{E}A_3 \leq \max_{l} \left\|\mathbb{E}\nabla\psi_{lh}\right\|^2 \max_l \left(\mathcal{L}_l f_{lh}\right)^2 C^2 \tau^2 h^{2}~.
\end{align*}
Collecting low order terms from the above bounds, we have there exists some constant $D_2$ independent of $(L, h, \tau)$,
such that
\begin{align*}
&\mathbb{E}\left(\hat{\phi} - \bar{\phi}\right)^2 \\
\leq& \frac{C_1}{Lh} + C_2h^{2K} + \max_{l} \left\|\mathbb{E}\nabla\psi_{lh}\right\|^2 \max_l \left\|\mathcal{L}_l\right\|^2 C^2 \tau^2 h^{2} \\
\leq& D_2\left(\frac{1}{Lh} + \tilde{M}_1\tau^2 h^{2} + \tilde{M}_2 h^{2K}\right)~,
\end{align*}
where $\tilde{M}_1 \triangleq \max_{l} \left\|\mathbb{E}\nabla\psi_{lh}\right\|^2 \max_l \left(\mathcal{L}_lf_{lh}\right)^2 C^2$,
$\tilde{M}_2 \triangleq \mathbb{E}\left(\frac{1}{L(K+1)!}\sum_l\tilde{\Lcal}_l^{K+1}\psi_{(l-1)h}\right)^2$.
This completes the proof.
\end{proof}

\section{Proof of Theorem~\ref{theo:var}}

In the proof, we will use the following simple result stated Lemma~\ref{lem:martingale}.

\begin{lemma}\label{lem:martingale}
Let $(\mathcal{M}_1, \cdots, \mathcal{M}_N)$ be a set of independent martingale, {\it i.e}, 
$\mathbb{E}\left[\mathcal{M}_n | \mathcal{F}\right] = 0$, where $\mathcal{F}$ is the filtration
generated by $\mathcal{M}_n$. Then we have
\begin{align}\label{eq:martingale}
	\mathbb{E}\left[\left(\sum_{n=1}^N \mathcal{M}_n\right)^2|\mathcal{F}\right] = \sum_{n=1}^N \mathbb{E}\left[\mathcal{M}_n^2|\mathcal{F}\right]~.
\end{align}
\end{lemma}

\begin{proof}
\begin{align*}
&\mathbb{E}\left[\left(\sum_{n=1}^N \mathcal{M}_n\right)^2|\mathcal{F}\right]
= \mathbb{E}\left[\sum_{i=1}^N\sum_{j=1}^N \mathcal{M}_i\mathcal{M}_j|\mathcal{F}\right] \\
=& \mathbb{E}\left[\sum_{i=1}^N \mathcal{M}_i^2|\mathcal{F}\right] + \sum_{i\neq j} \mathbb{E}\left[\mathcal{M}_i|\mathcal{F}\right] \mathbb{E}\left[\mathcal{M}_j|\mathcal{F}\right] \\
=& \sum_{i=1}^N \mathbb{E}\left[\mathcal{M}_i^2|\mathcal{F}\right]~.
\end{align*}
\end{proof}

In the following we will omitted the filtration $\mathcal{F}$ in the expectation for simplicity.
We we now ready to prove Theorem~\ref{theo:var}.

\begin{proof}
By definition, we have
\begin{align*}
	\mbox{Var}\left(\hat{\phi}_L\right) = \mathbb{E}\left(\hat{\phi}_L - \bar{\phi} - \left(\mathbb{E}\hat{\phi}_L - \bar{\phi}\right)\right)^2
\end{align*}
Substitute \eqref{eq:expansion11} and \eqref{eq:expansion12} into the above equation, we have 
\begin{align*}
	\hat{\phi}_L &- \mathbb{E} \bar{\phi} = -\frac{1}{Lh}\sum_l \left(\mathbb{E} \psi_{(l-1)h} - \psi_{(l-1)h}\right)\\
	&- \frac{1}{L} \sum_l \left(A_1 - \mathbb{E}A_1\right)
	- \sum_k \frac{h^{k-1}}{k! L} \sum_l \left(A_2 - \mathbb{E}A_2\right) 
	- \frac{h^{K}}{(K+1)!L} \sum_l \left(A_3 - \mathbb{E}A_3\right)~,
\end{align*}
where
\begin{align*}
	A_1 &\triangleq \Delta V_l \psi_{(l-1)h} \\
	A_2 &\triangleq \tilde{\mathcal{L}}_l^k \psi_{(l-1)h} \\
	A_3 &\triangleq \tilde{\mathcal{L}}_l^{K+1} \psi_{(l-1)h}~.
\end{align*}
Take square on both sides, following by expectation, and note that all $(A_i - \mathbb{E}A_i)$ are martingale for
$i = 1, 2, 3$, which allows us to use \eqref{eq:martingale} from Lemma~\ref{lem:martingale}.
We have there exists a constant $D$ independent of $(L, h, \tau)$, such that
\begin{align*}
	&\mbox{Var}\left(\hat{\phi}_L\right) \leq D\left(\frac{1}{L^2h^2}\mathbb{E}\left(\sum_l \left(\mathbb{E} \psi_{(l-1)h} - \psi_{(l-1)h}\right)\right)^2\right.\\
	&+ \frac{1}{L^2} \sum_l \mathbb{E}\left(A_1 - \mathbb{E}A_1\right)^2 
	+ \sum_k \frac{h^{2(k-1)}}{(k! L)2} \sum_l \mathbb{E}\left(A_2 - \mathbb{E}A_2\right)^2 \\
	&\left.+ \frac{h^{2K}}{((K+1)!L)^2} \sum_l \mathbb{E}\left(A_3 - \mathbb{E}A_3\right)^2\right) \\
	&\leq D\left(\underbrace{\frac{1}{L^2h^2} \mathbb{E}\left(\sum_l \left(\mathbb{E} \psi_{(l-1)h} - \psi_{(l-1)h}\right)\right)^2}_{B_1} \right. \\
	&~~~~+ \left.\frac{1}{L^2} \sum_l \mathbb{E}\left(A_1 - \mathbb{E}A_1\right)^2
	+ \sum_{k=2}^K \frac{h^{2(k-1)}}{(k! L)2} \sum_l \mathbb{E}A_2^2 
	+ \frac{h^{2K}}{((K+1)!L)^2} \sum_l \mathbb{E}A_3^2\right)~.
\end{align*}
According to Lemma~\ref{lem:A_2_A_3}, $B_1$ is bounded by 
\begin{align*}
	B_1 = O\left(\frac{1}{Lh}\right)~. 
\end{align*}	
Furthermore, according to the assumptions,
both $\mathbb{E}A_2^2$ and $\mathbb{E}A_3^2$ are bounded. The delayed parameter $\tau$ exists in 
$\mathbb{E}\left(A_1 - \mathbb{E}A_1\right)^2$, we have
\begin{align*}
	&\mathbb{E}\left(A_1 - \mathbb{E}A_1\right)^2 \\
	=& \mathbb{E}\left(\Delta V_l\psi_{(l-1)h} - \mathbb{E}\Delta V_l\psi_{(l-1)h}\right)^2 \\
	=&\mathbb{E}\left(\left(\tilde{G}_{(l-\tau_l)h} - G_{lh}\right)\cdot \nabla_{\pb} \psi_{(l-1)h} 
	- \mathbb{E}\left(\tilde{G}_{(l-\tau_l)h} - G_{lh}\right)\cdot \nabla_{\pb} \psi_{(l-1)h}\right)^2
\end{align*}
Expanding the terms, we have there exists a constant $D_1$ such that
\begin{align*}
	&\mathbb{E}\left(A_1 - \mathbb{E}A_1\right)^2 \\
	\leq&D_1\mathbb{E} \left(\tilde{G}_{(l-\tau_l)h}\cdot \nabla_{\pb} \psi_{(l-1)h} - \mathbb{E}\tilde{G}_{(l-\tau_l)h}\cdot \nabla_{\pb} \psi_{(l-1)h} \right)^2\\
	&~~+D_1\mathbb{E} \left(G_{lh}\cdot \nabla_{\pb} \psi_{(l-1)h} - \mathbb{E}G_{lh}\cdot \nabla_{\pb} \psi_{(l-1)h} \right)^2 \\
	=&D_1\mathbb{E} \left(\tilde{G}_{(l-\tau_l)h}\cdot \nabla_{\pb} \psi_{(l-1)h}\right)^2
	+D_1\mathbb{E} \left(G_{lh}\cdot \left(\nabla_{\pb} \psi_{(l-1)h} - \nabla_{\pb} \psi_{(l-1)h}\right) \right)^2 \\
	\leq&D_1\left(\mathbb{E} \left\|\tilde{G}_{(l-\tau_l)h}\right\|^2 \mathbb{E}\left\|\nabla_{\pb} \psi_{(l-1)h}\right\|^2
	+ \mathbb{E} \left\|G_{lh}\right\|^2 \mathbb{E}\left\|\nabla_{\pb} \psi_{(l-1)h} - \nabla_{\pb} \psi_{(l-1)h}\right\|^2\right) \\
	\leq&D_1\sup_l \left\{\mathbb{E} \left\|\tilde{G}_{lh}\right\|^2 \mathbb{E}\left\|\nabla_{\pb} \psi_{lh}\right\|^2 
	+ \mathbb{E} \left\|G_{lh}\right\|^2 \mathbb{E}\left\|\nabla_{\pb} \psi_{lh}\right\|^2\right\}~.
\end{align*}
According to the assumptions, the above bound is bounded, and does not depend on $\tau$.
As a result,
\begin{align*}
	\frac{1}{L^2}\sum_{l}\mathbb{E}\left(A_1 - \mathbb{E}A_1\right)^2
	\leq \frac{D_1}{L}~.
\end{align*}
In addition, the bounds for both $\mathbb{E}A_2^2$ and $\mathbb{E}A_3^2$ are given in Lemma~\ref{lem:A_2_A_3},
which are higher-order terms with respect to $h$, {\it i.e.}, $O\left(h^{2K}\right)$. 

Collecting low order terms, we have 
there exists a constant $D$ independent of $(L, h, \tau)$, such that the variance is bounded by:
\begin{align*}
\mbox{Var}\left(\hat{\phi}_L\right) \leq D\left(\frac{1}{Lh} + h^{2K} \right) = D\left(\frac{1}{W \bar{L} h} + h^{2K} \right)~.
\end{align*}
\end{proof}

\section{Proof of Theorem~\ref{theo:mul-servers}}

We separate the proof for the bias and MSE, respectively.

\begin{proof}[Proof for the bias]
According to the definition of $\hat{\phi}_L^S$, we have
\begin{align}
	&\left|\mathbb{E}\hat{\phi}_L^S - \bar{\phi}\right| = \left|\mathbb{E}\sum_{s=1}^S\frac{T_s}{T}\hat{\phi}_{L_s} - \bar{\phi}\right| \nonumber\\
	=&\left|\sum_{s=1}^S\frac{T_s}{T}\mathbb{E}\left(\hat{\phi}_{L_s} - \bar{\phi}\right)\right| \nonumber\\
	\leq& \sum_{s=1}^S \frac{T_s}{T} \left|\mathbb{E}\hat{\phi}_{L_s} - \bar{\phi}\right| \nonumber\\
	=& \sum_{s=1}^S \frac{T_s}{T} D_1\left(\frac{1}{L_sh_s} + \left(M_1 \tau h_s + M_2h_s^K\right)\right) \label{eq:bias_single}\\
	=& D_1\left(\frac{S}{T} + \sum_{s=1}^S \frac{T_s}{T} \left(M_1 \tau h_s + M_2h_s^K\right)\right) \nonumber \\
	\leq& D_1\left(\frac{S}{T} + \frac{ST_m}{T} \left(M_1 \tau h_m + M_2h_m^K\right)\right)~,
\end{align}
where $T_m \triangleq \max_l T_l$, $h_m \triangleq \max_l h_l$, \eqref{eq:bias_single} follows by 
substituting the bias from Theorem~\ref{theo:bias} for each server into the formula.
\end{proof}

Similarly, for the MSE bound, we have
\begin{align*}
	&\mathbb{E}\left(\hat{\phi}_L^S - \bar{\phi}\right)^2 = \mathbb{E}\left(\sum_{s=1}^S \frac{T_s}{T}\left(\hat{\phi}_{L_s} - \bar{\phi}\right)\right)^2  \\
	=& \sum_{s=1}^S \frac{T_s^2}{T^2} \mathbb{E}\left(\hat{\phi}_{L_s} - \bar{\phi}\right)^2
		+ \sum_{i\neq j} \frac{T_i T_j}{T_2} \mathbb{E}\left[\hat{\phi}_{L_i} - \bar{\phi}\right] \mathbb{E}\left[\hat{\phi}_{L_j} - \bar{\phi}\right]  \\
	\leq& \sum_{s=1}^S \frac{T_s^2}{T^2} \mathbb{E}\left(\hat{\phi}_{L_s} - \bar{\phi}\right)^2
		+ \sum_{i\neq j} \frac{T_i T_j}{T^2} \left|\mathbb{E}\hat{\phi}_{L_i} - \bar{\phi}\right| \left|\mathbb{E}\hat{\phi}_{L_j} - \bar{\phi}\right|~.
\end{align*}
Substituting the bounds for single chain bias and MSE from Theorem~\ref{theo:bias}
and Theorem~\ref{theo:MSE}, respectively, we have
\begin{align*}
	\leq& \sum_{s=1}^S \frac{T_s^2}{T^2} D_2^\prime\left(\frac{1}{T_s} + \left(\tilde{M}_1\tau^2 h_s^{2} + \tilde{M}_2 h_s^{2K}\right)\right)  \\
	&+ \sum_{i\neq j} \frac{T_i T_j}{T^2} D_1\left(\frac{1}{T_i} + \left(M_1\tau h_i + M_2 h_i^{K}\right)\right)
	D_1\left(\frac{1}{T_j} + \left(M_1\tau h_j + M_2 h_j^{K}\right)\right) \\
	\leq& D_2\left(\frac{1}{T} + \frac{S^2-S}{T^2} + \sum_{i,j} \frac{T_i T_j}{T^2} \left(M_1^2\tau^2 h_m^{2} + M_2^2 h_m^{2K}\right)\right) \\
	\leq& D_2\left(\frac{1}{T} + \frac{S^2-S}{T^2} + \frac{S^2 T_m^2}{T^2} \left(M_1^2\tau^2 h_m^{2} + M_2^2 h_m^{2K}\right)\right)~,
\end{align*}
where $D_2 = \max\{D_2^\prime, D_1^2\}$, $T_m \triangleq \max_l T_l$, $h_m \triangleq \max_l h_l$, 
the last equality collects the low order terms. 
This completes the proof.

\section{Proof of Theorem~\ref{theo:mul-servers-var}}

\begin{proof}
Following the proof of Theorem~\ref{theo:mul-servers}, for the variance, we have
\begin{align*}
	&\mathbb{E}\left(\hat{\phi}_L^S - \mathbb{E}\hat{\phi}\right)^2 = \mathbb{E}\left(\sum_{s=1}^S \frac{T_s}{T}\left(\hat{\phi}_{L_s} - \mathbb{E}\hat{\phi}_{L_s}\right)\right)^2  \\
	=& \sum_{s=1}^S \frac{T_s^2}{T^2} \mathbb{E}\left(\hat{\phi}_{L_s} - \bar{\phi}_{L_s}\right)^2
	+ \sum_{i\neq j} \frac{T_i T_j}{T_2} \mathbb{E}\left[\hat{\phi}_{L_i} - \mathbb{E}\hat{\phi}_{L_i}\right] \mathbb{E}\left[\hat{\phi}_{L_j} - \mathbb{E}\hat{\phi}_{L_j}\right]  \\
	=& \sum_{s=1}^S \frac{T_s^2}{T^2} \mathbb{E}\left(\hat{\phi}_{L_s} - \bar{\phi}_{L_s}\right)^2~.
\end{align*}
Substituting the variance bound in Theorem~\ref{theo:var} for each server, we have
\begin{align*}
	&\mathbb{E}\left(\hat{\phi}_L^S - \mathbb{E}\hat{\phi}\right)^2 
	\leq D \sum_{s=1}^S \frac{T_s^2}{T^2} \left(\frac{1}{L_s h_s} + h_s^{2K}\right) \\
	=& D \sum_{s=1}^S \left(\frac{T_s}{T^2} + \frac{T_s^2}{T^2} h_s^{2K}\right) \\
	=& D\left(\frac{1}{T} + \sum_{s=1}^S \frac{T_s^2}{T^2} h_s^{2K}\right)
\end{align*}
\end{proof}

\section{Additional Results}\label{sec:results}

See Figure~\ref{fig:loss_single_full} \ref{fig:loss_mul_time} \ref{fig:loss_server_1} \ref{fig:loss_server_2} \ref{fig:loss_server_4}
\ref{fig:loss_server_6}. The content of the figures is described in the titles.

\begin{figure}[!htb]
	\begin{minipage}{0.49\linewidth}
		\includegraphics[width=\linewidth]{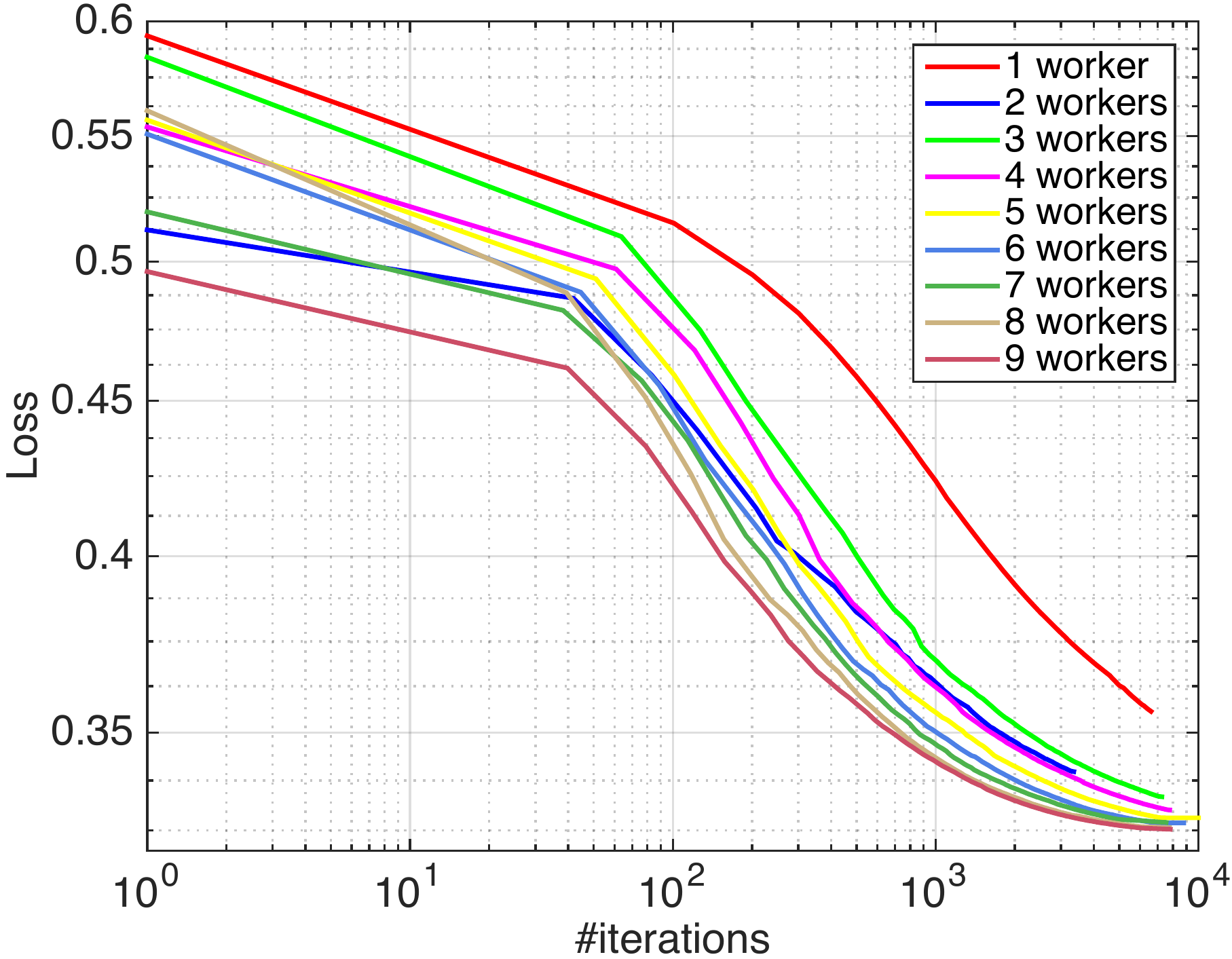}
	\end{minipage}
	\begin{minipage}{0.49\linewidth}
		\includegraphics[width=\linewidth]{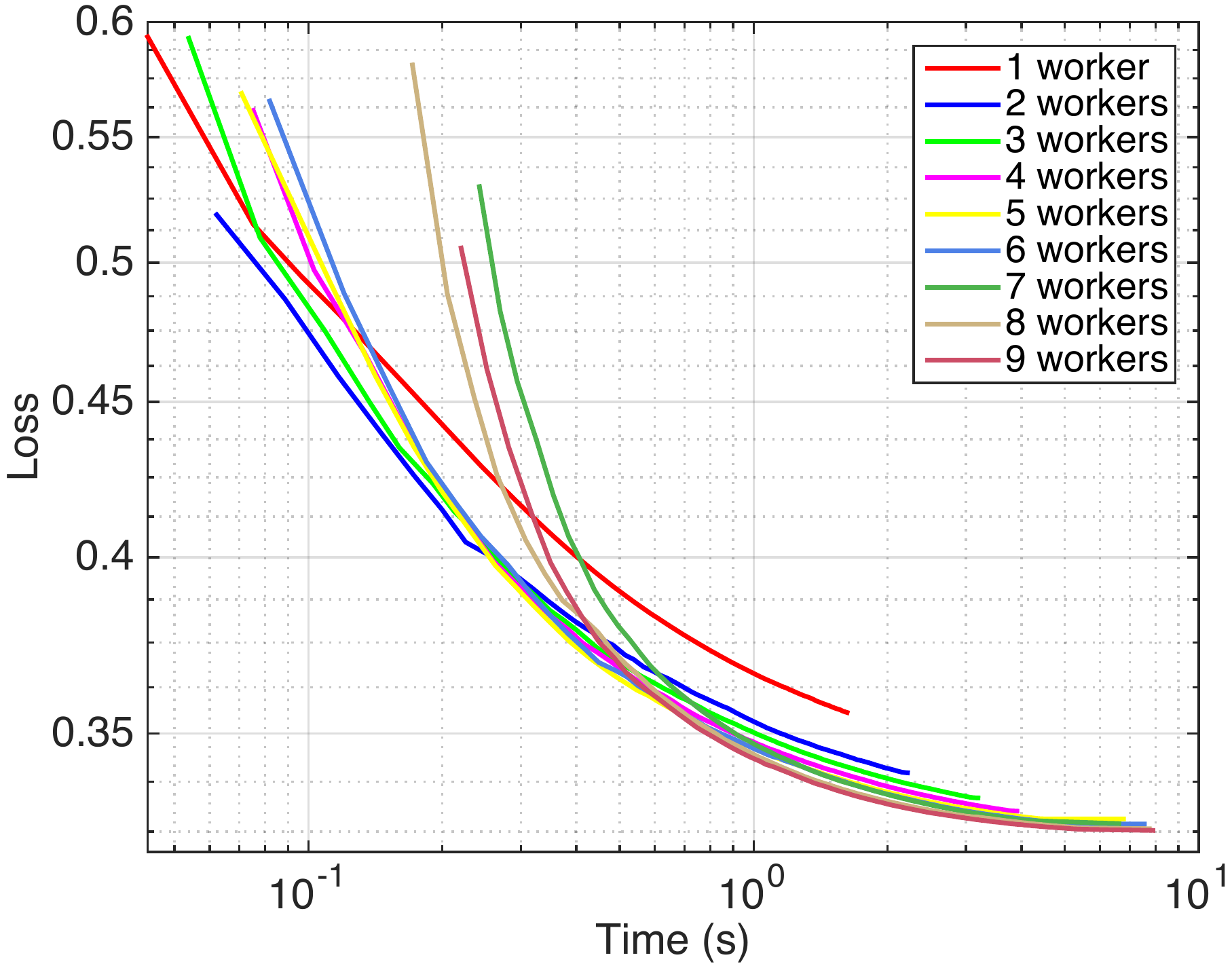}
	\end{minipage}
	
	\begin{minipage}{0.49\linewidth}
		\includegraphics[width=1.02\linewidth]{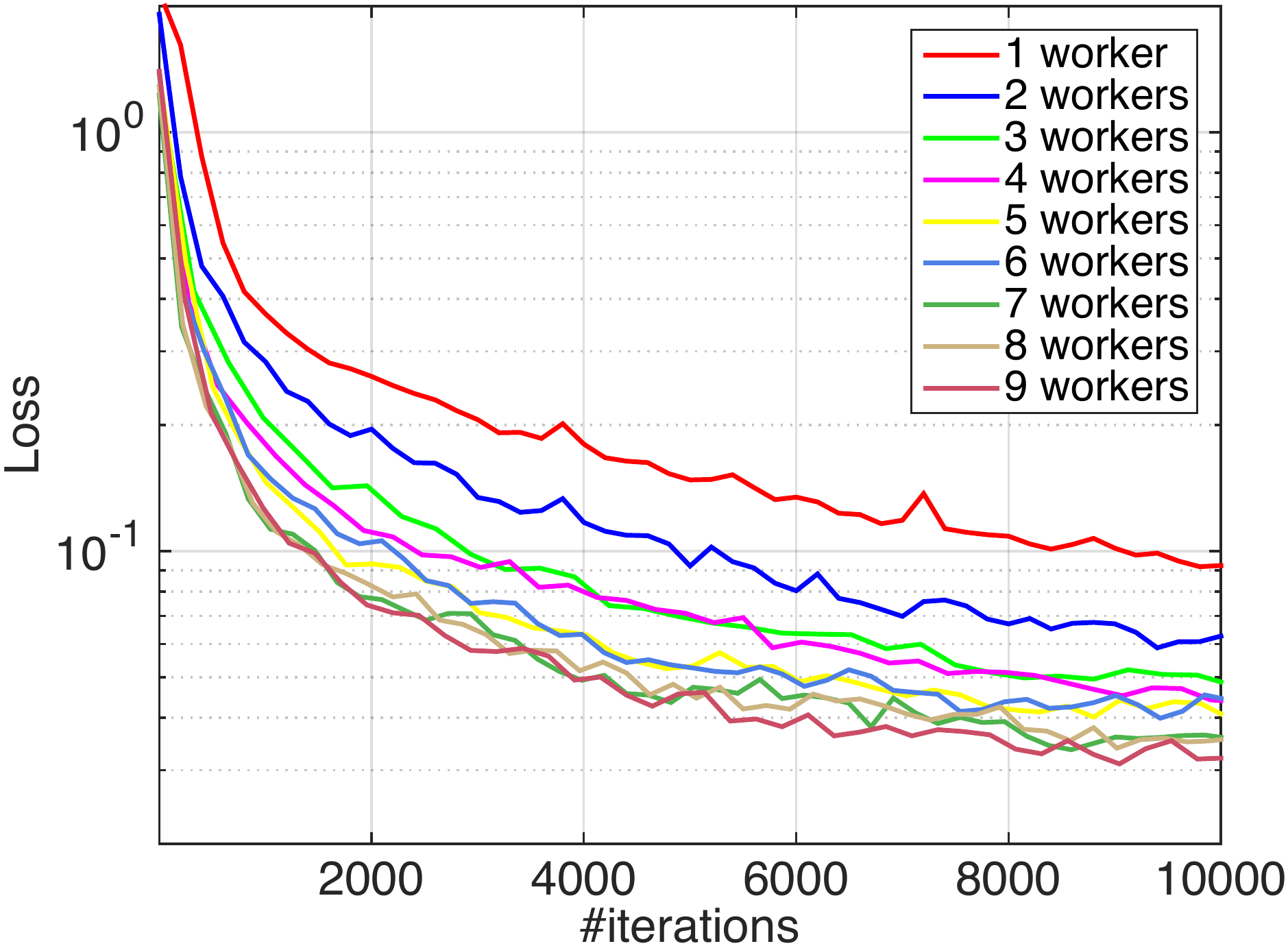}
	\end{minipage}
	\begin{minipage}{0.49\linewidth}
		\includegraphics[width=\linewidth]{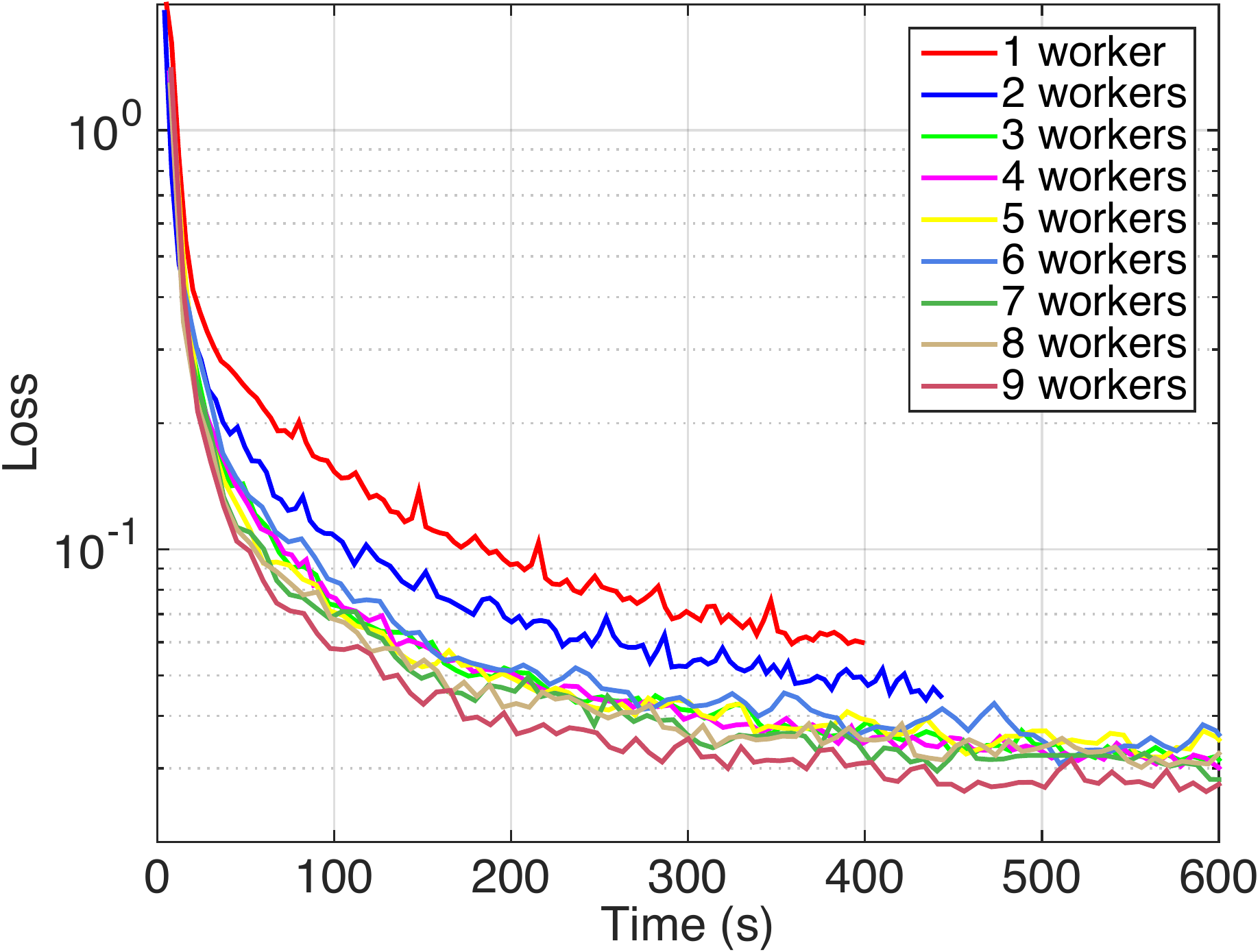}
	\end{minipage}
		
	\begin{minipage}{0.49\linewidth}
		\includegraphics[width=0.96\linewidth]{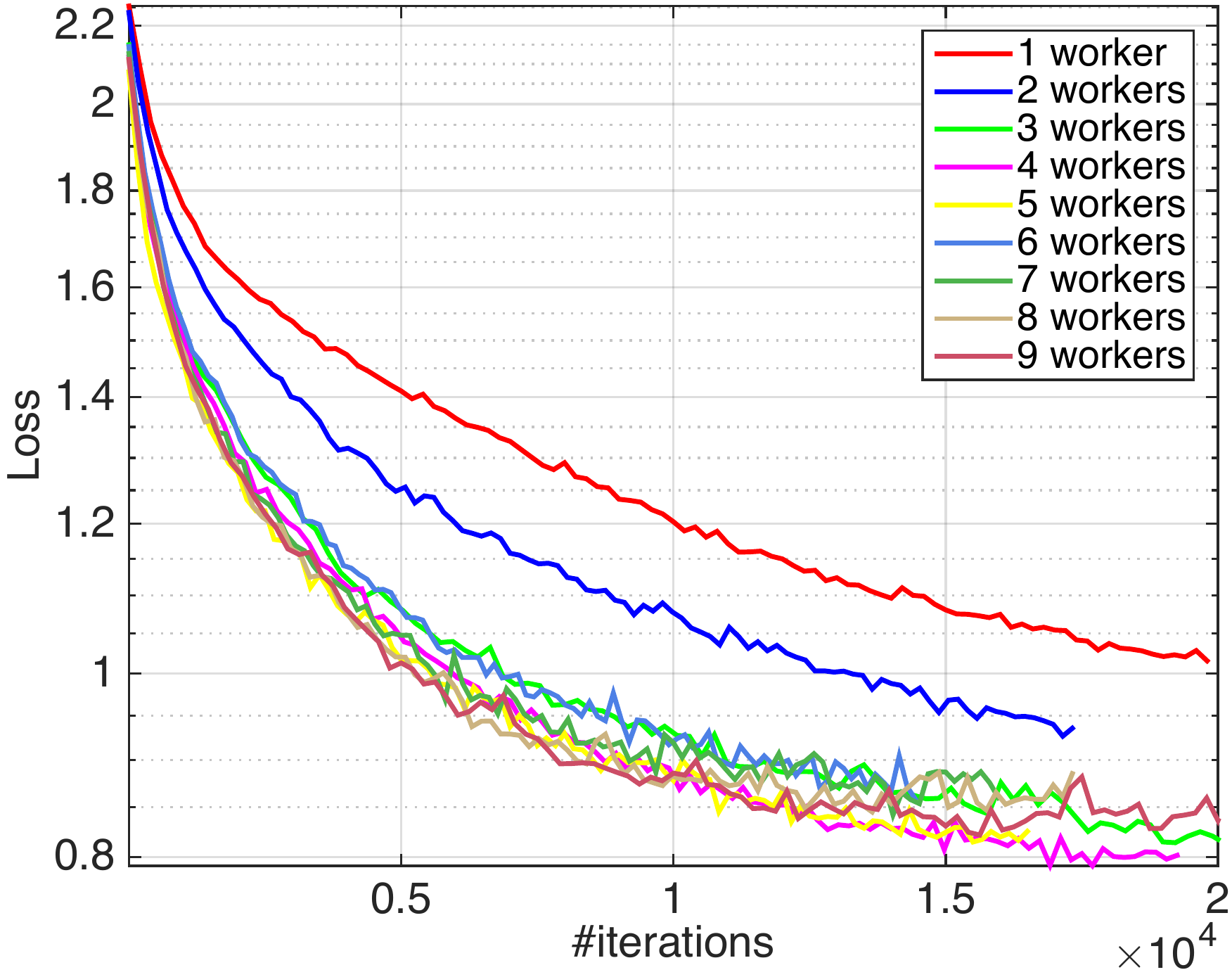}
	\end{minipage}	
	\begin{minipage}{0.49\linewidth}
		\includegraphics[width=\linewidth]{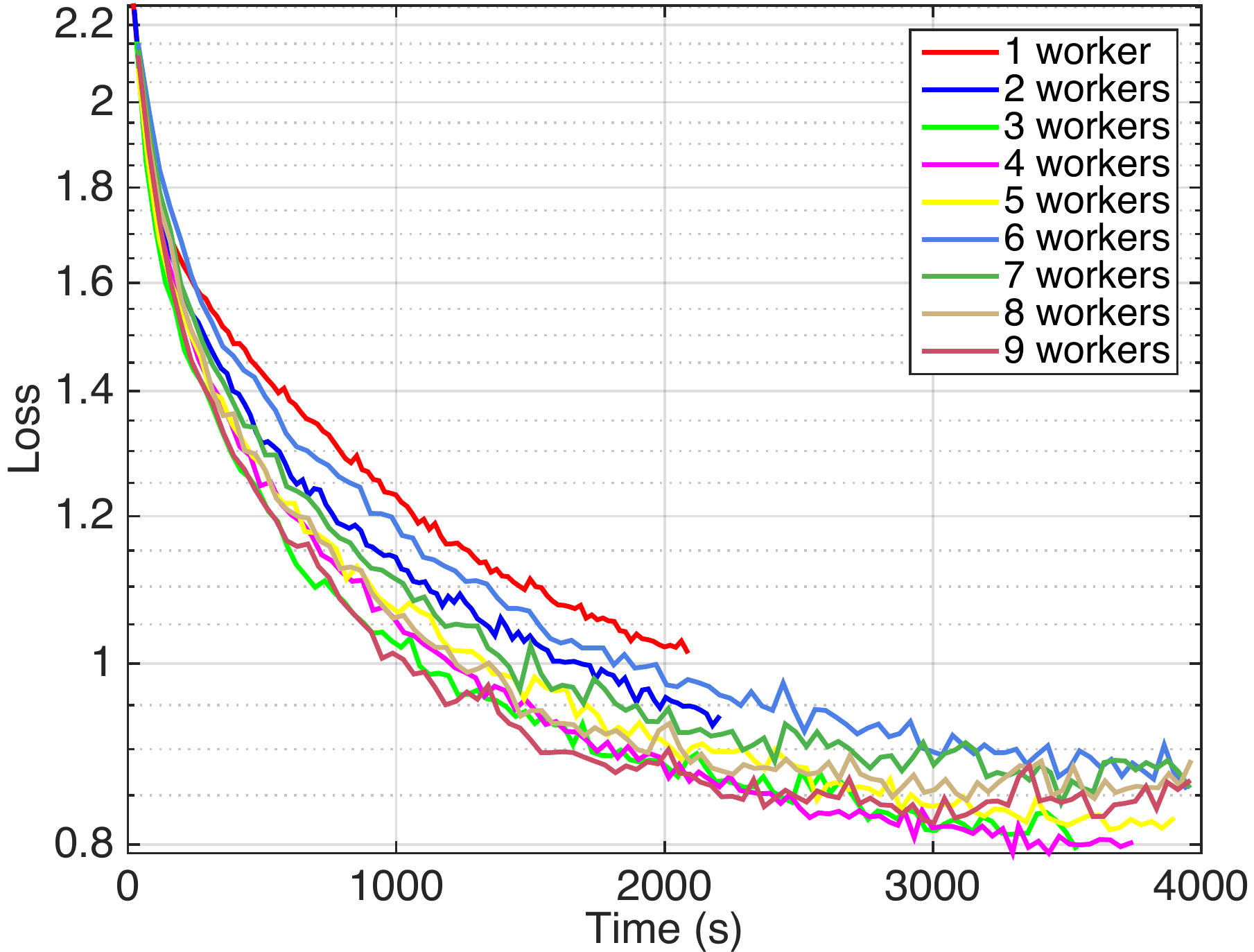}
	\end{minipage}
	\caption{Testing loss vs. \#workers. From top down, each row corresponds to
		the a9a, MNIST and CIFAR dataset, respectively.}
	\label{fig:loss_single_full}
\end{figure} 

\begin{figure}[!htb]
	\begin{center}
		
		\begin{minipage}{0.49\linewidth}
			\centerline{\includegraphics[width=\linewidth]{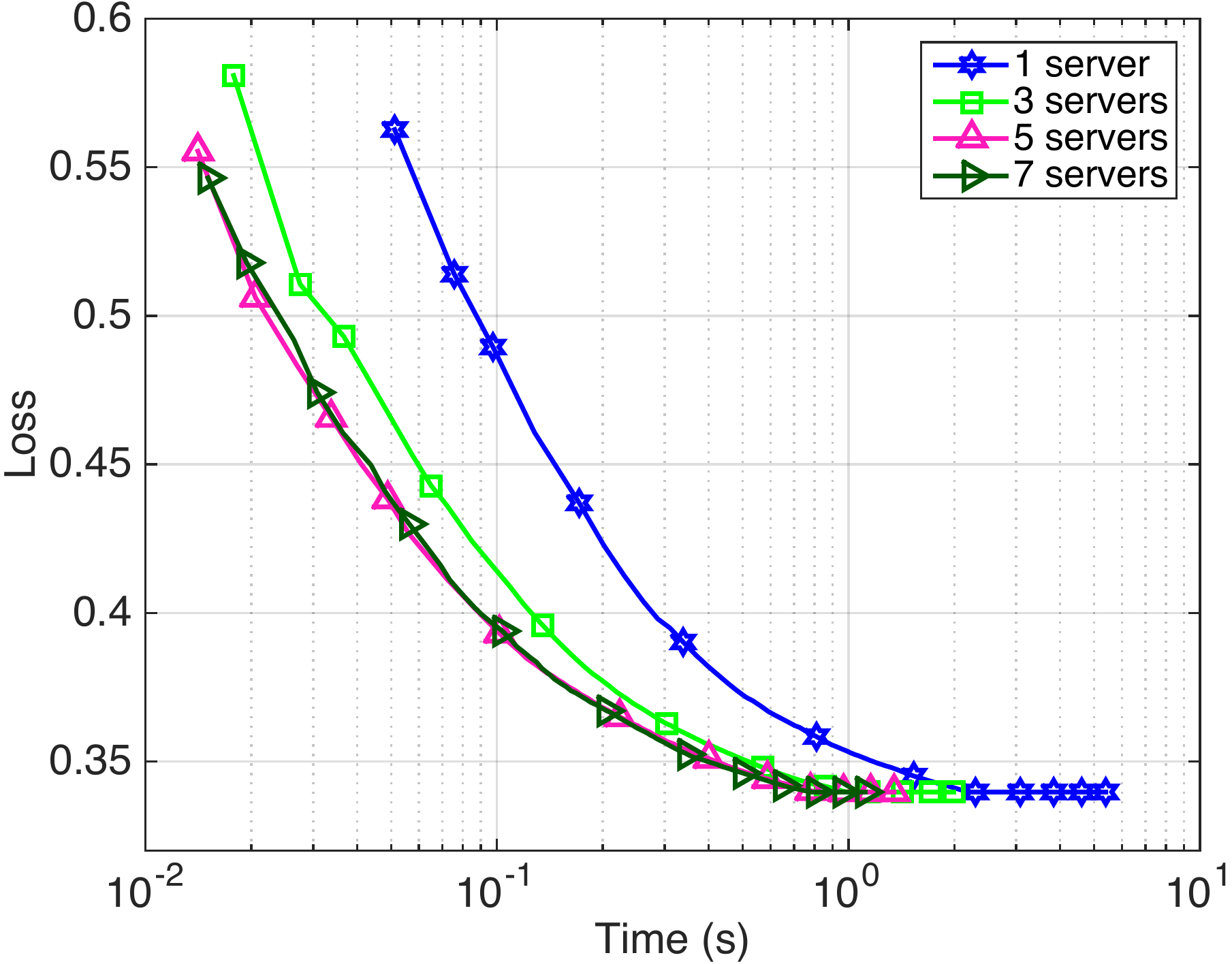}}
		\end{minipage}
		\begin{minipage}{0.49\linewidth}
			\centerline{\includegraphics[width=\linewidth,height=0.77\linewidth]{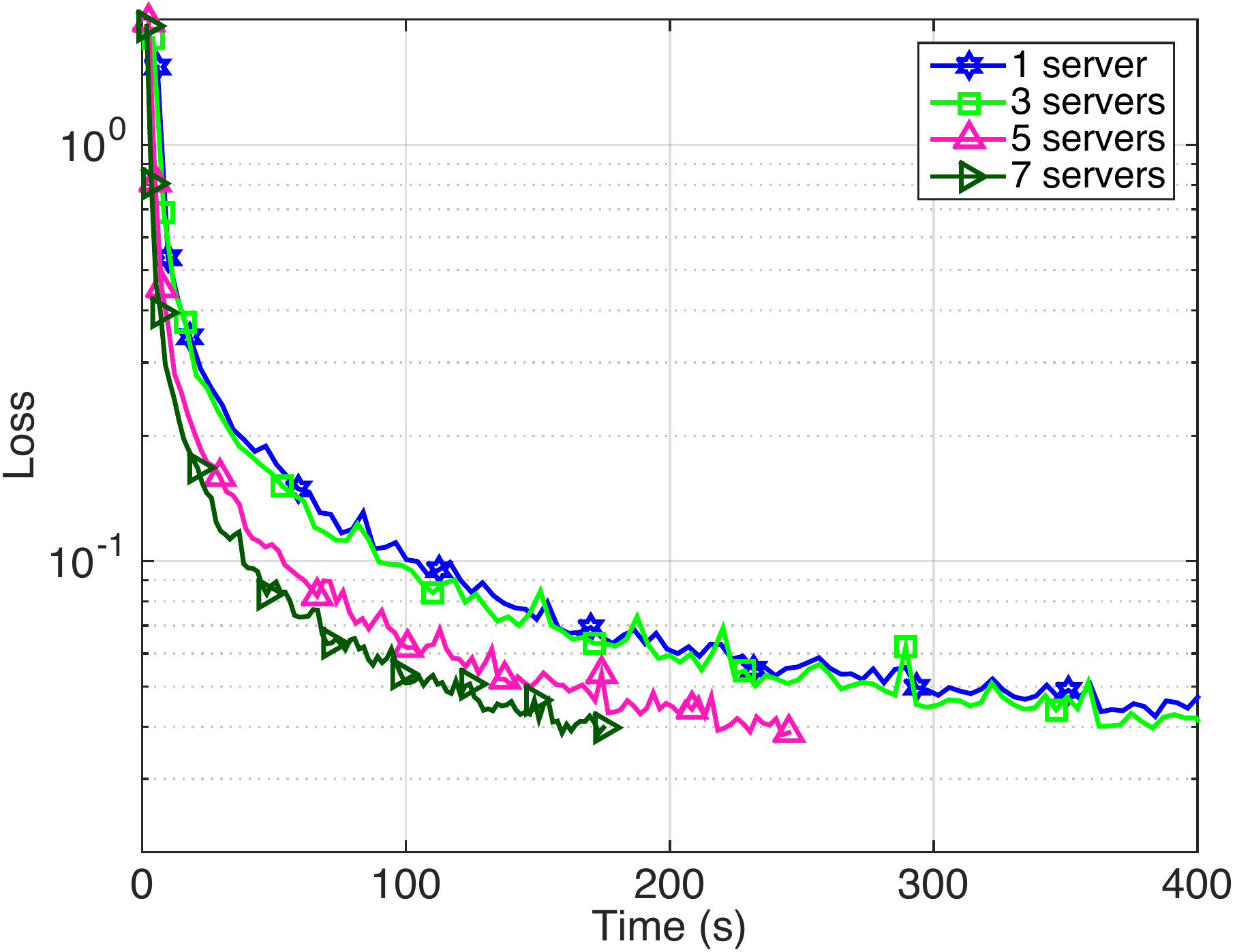}}
		\end{minipage}
		\begin{minipage}{0.49\linewidth}
			\centerline{\includegraphics[width=\linewidth]{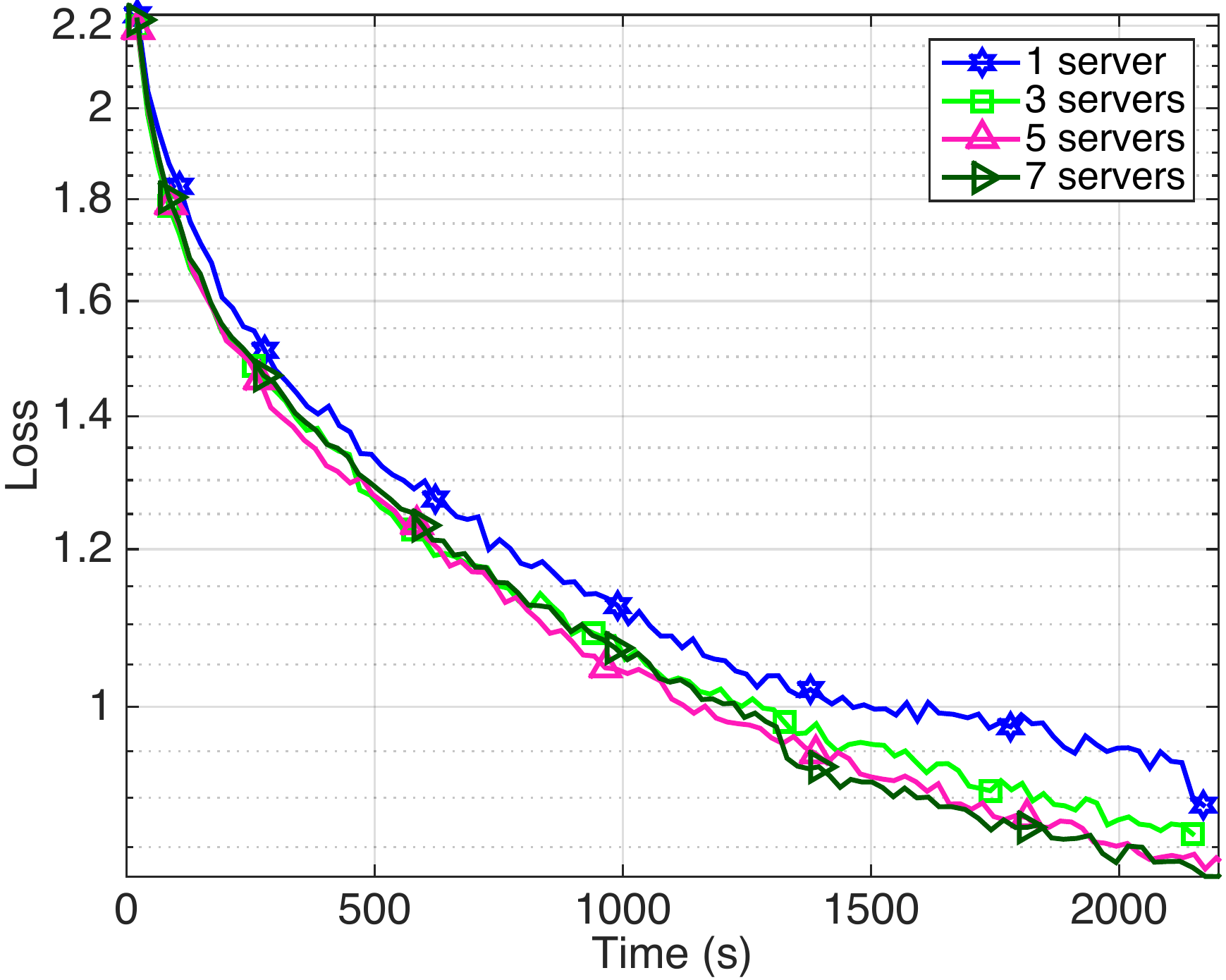}}
		\end{minipage}
		\caption{Testing loss vs. \#servers. From left to right, the first row corresponds to
			the a9a, MNIST datasets, and the second row corresponds to the CIFAR dataset, respectively.}
		\label{fig:loss_mul_time}
	\end{center}
\end{figure}

\begin{figure}[!htb]
	\begin{minipage}{0.49\linewidth}
		\includegraphics[width=\linewidth]{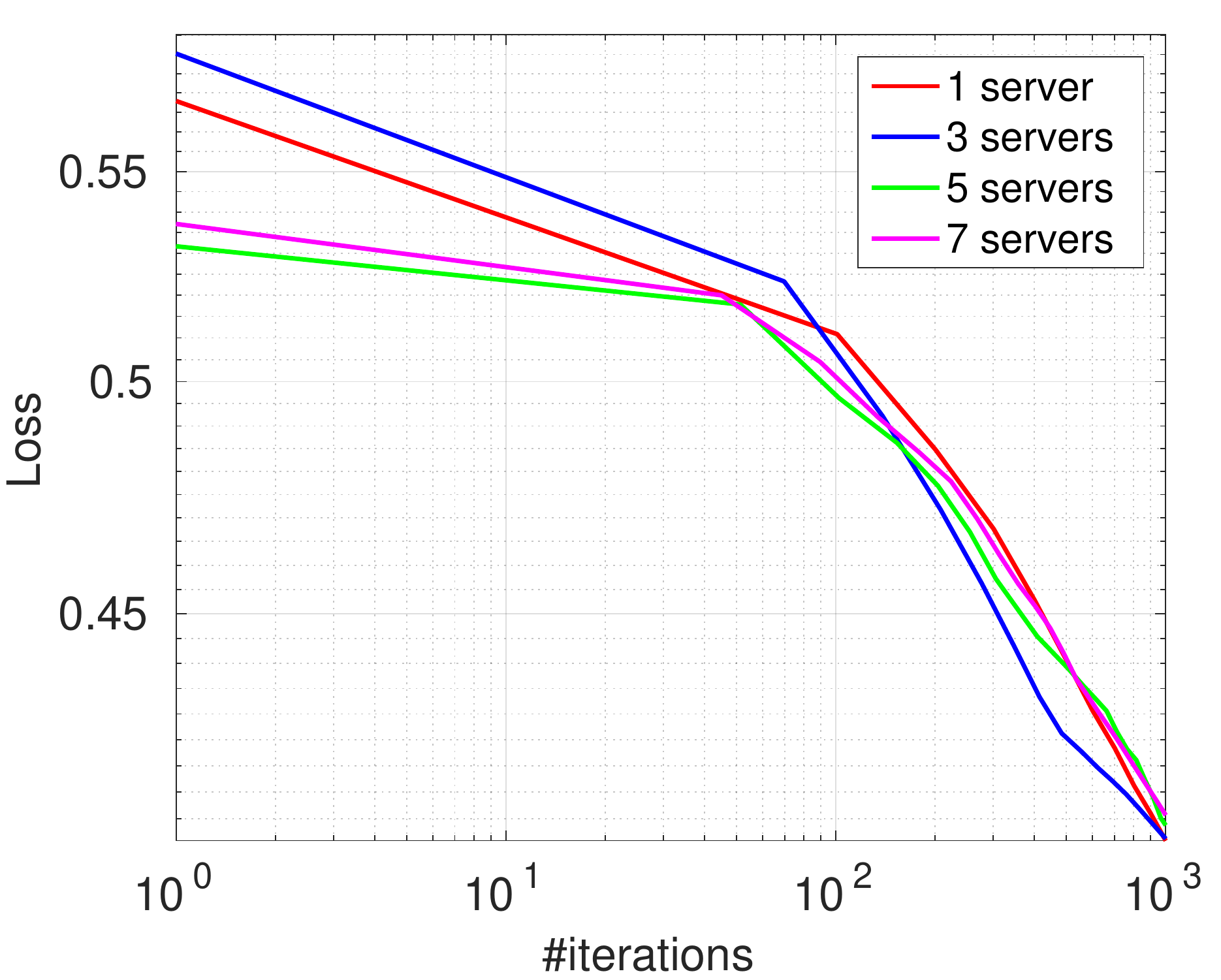}
	\end{minipage}\hspace{2mm}
	\begin{minipage}{0.49\linewidth}
		\includegraphics[width=\linewidth]{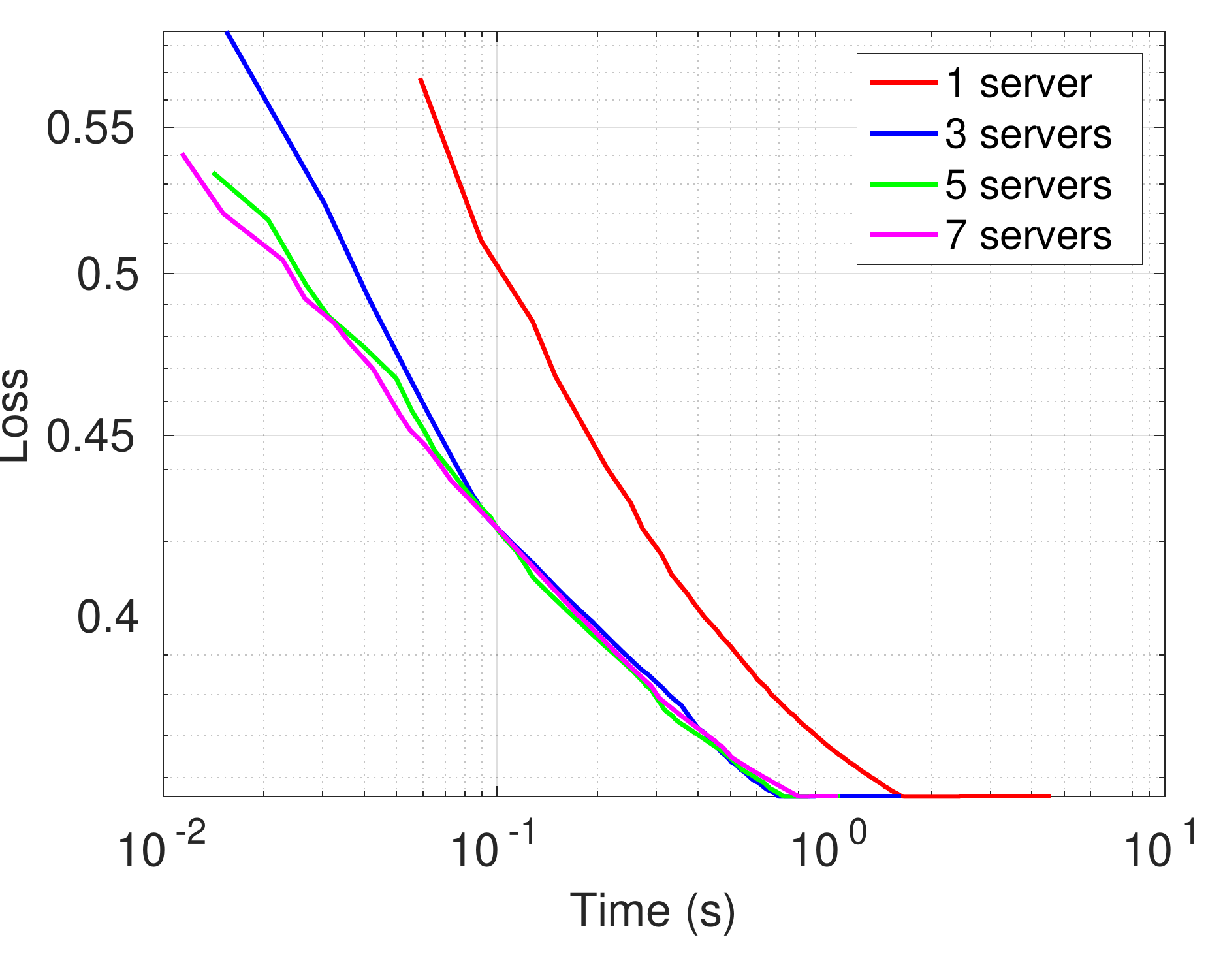}
	\end{minipage}
	
	\begin{minipage}{0.49\linewidth}\hspace{-1mm}
		\includegraphics[width=1.02\linewidth]{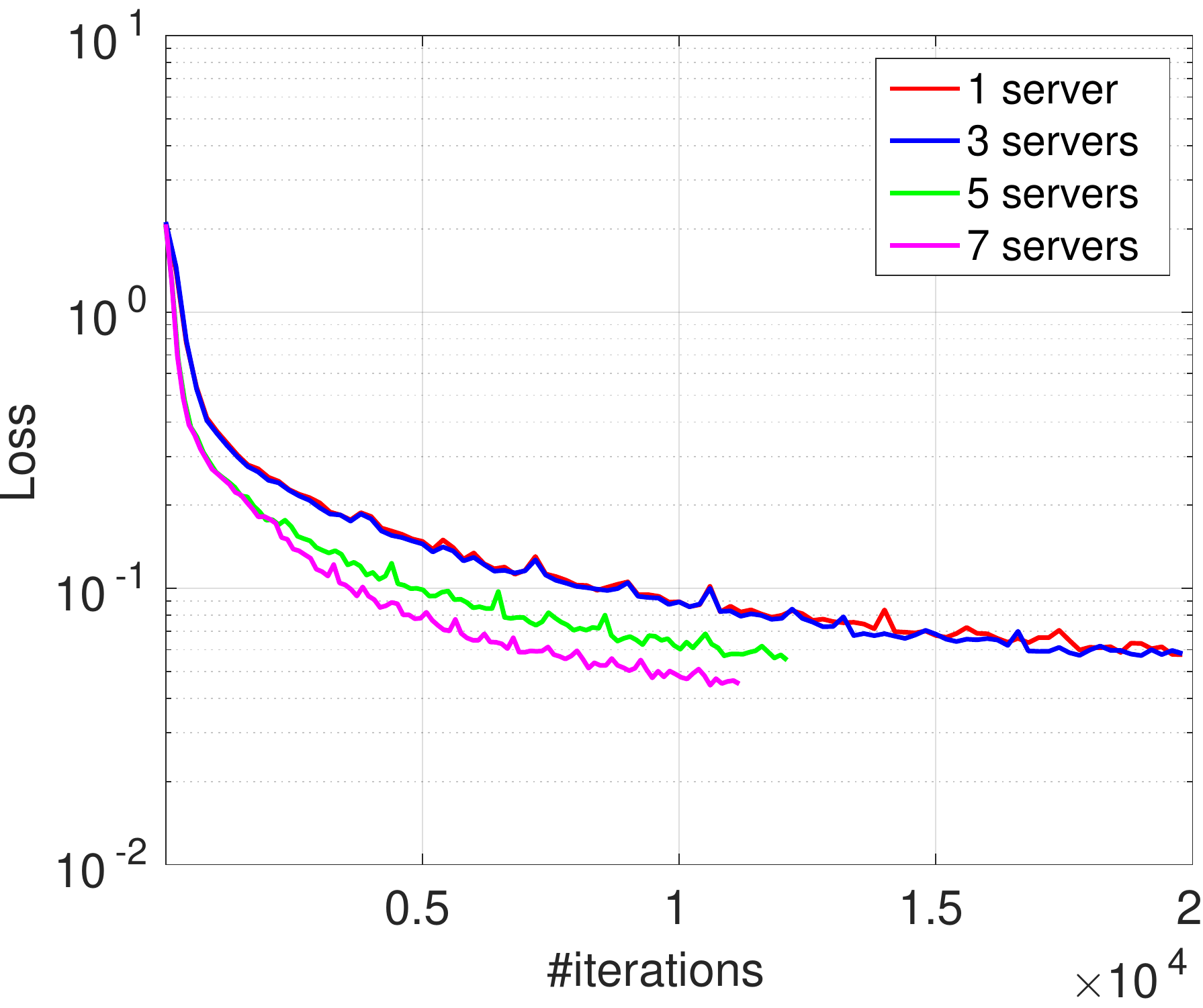}
	\end{minipage}
	\begin{minipage}{0.49\linewidth}\vspace{-0mm}\hspace{-1mm}
		\includegraphics[width=\linewidth]{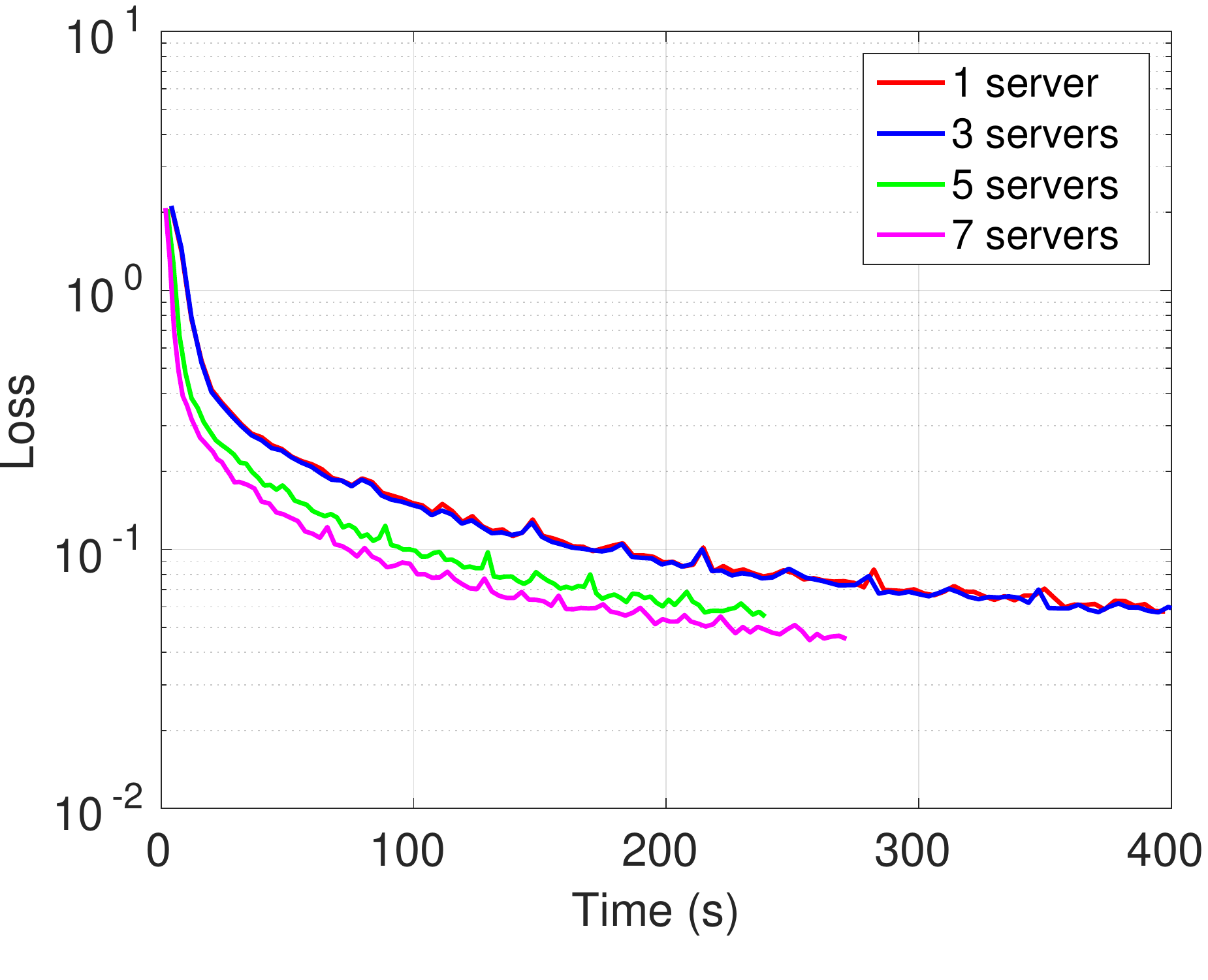}
	\end{minipage}
		
	\begin{minipage}{0.49\linewidth}\hspace{1mm}
		\includegraphics[width=0.96\linewidth]{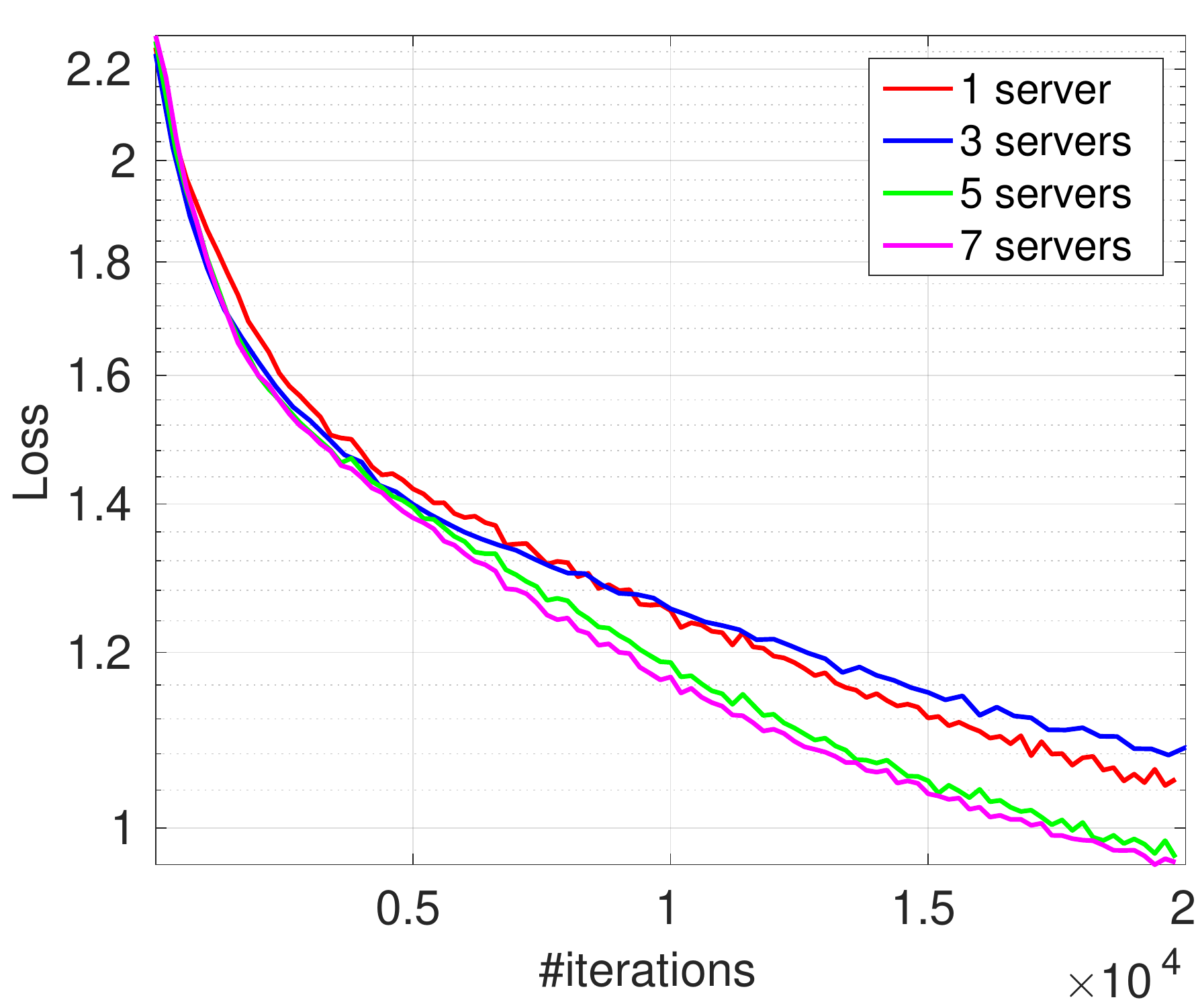}
	\end{minipage}	
	\begin{minipage}{0.49\linewidth}\hspace{1mm}
		\includegraphics[width=\linewidth]{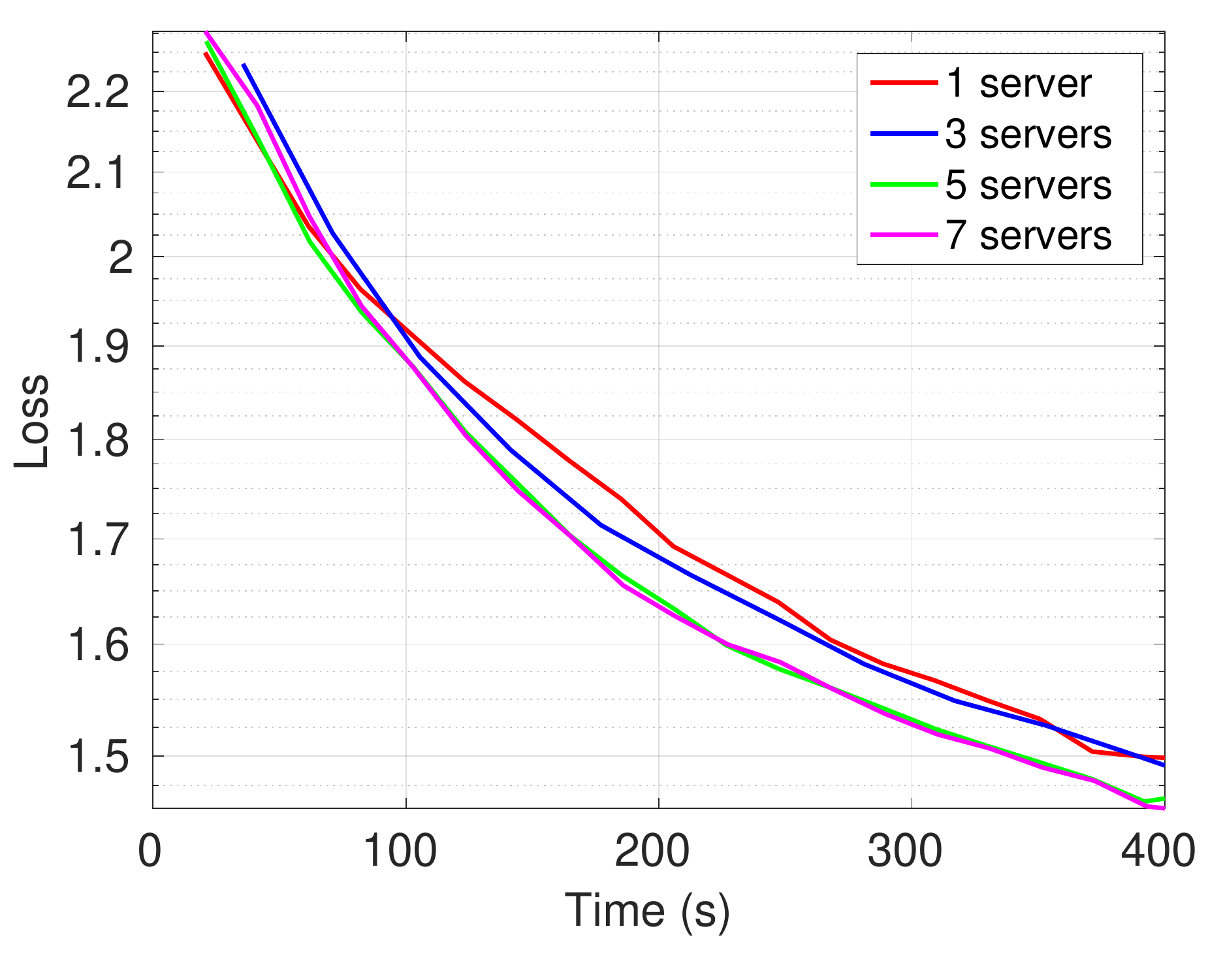}
	\end{minipage}
	\caption{Testing loss vs. \#servers. From top down, each row corresponds to
			the a9a, MNIST and CIFAR dataset, respectively. Each server is associated with 1 worker.}
	\label{fig:loss_server_1}
\end{figure} 

\begin{figure}[!htb]
	\begin{minipage}{0.49\linewidth}
		\includegraphics[width=\linewidth]{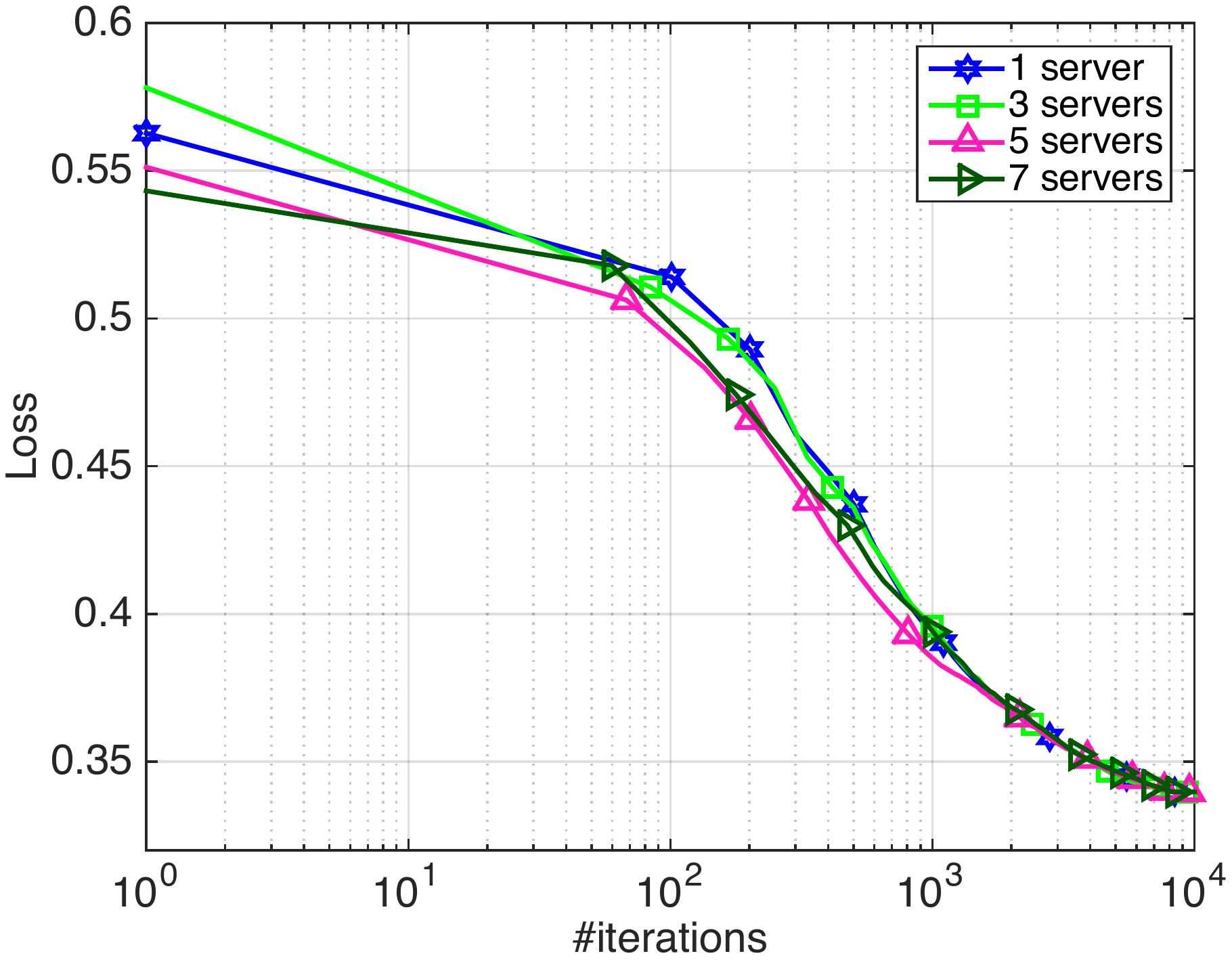}
	\end{minipage}
	\begin{minipage}{0.49\linewidth}
		\includegraphics[width=\linewidth]{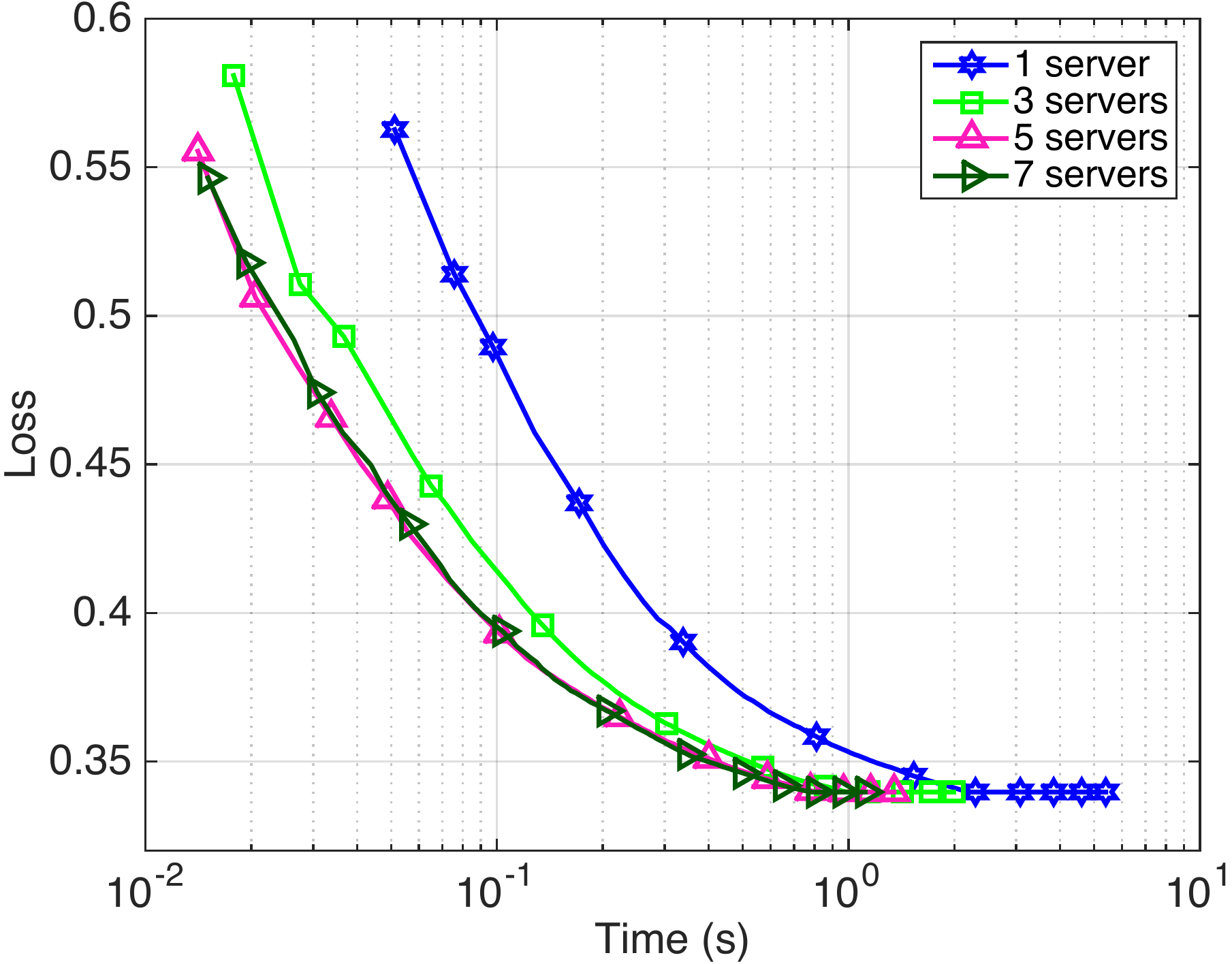}
	\end{minipage}
	
	\begin{minipage}{0.49\linewidth}
		\includegraphics[width=1.02\linewidth]{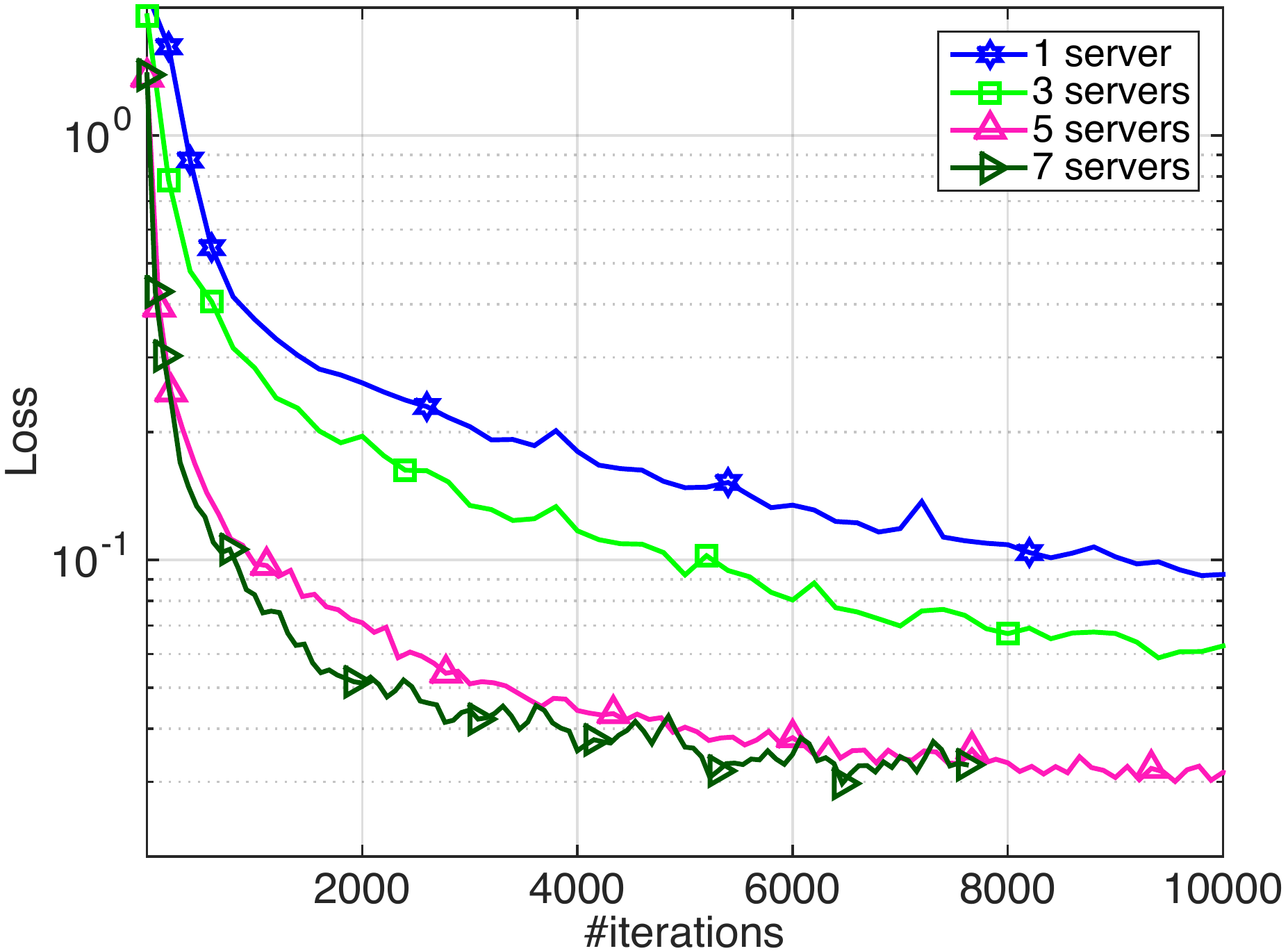}
	\end{minipage}
	\begin{minipage}{0.49\linewidth}
		\includegraphics[width=\linewidth]{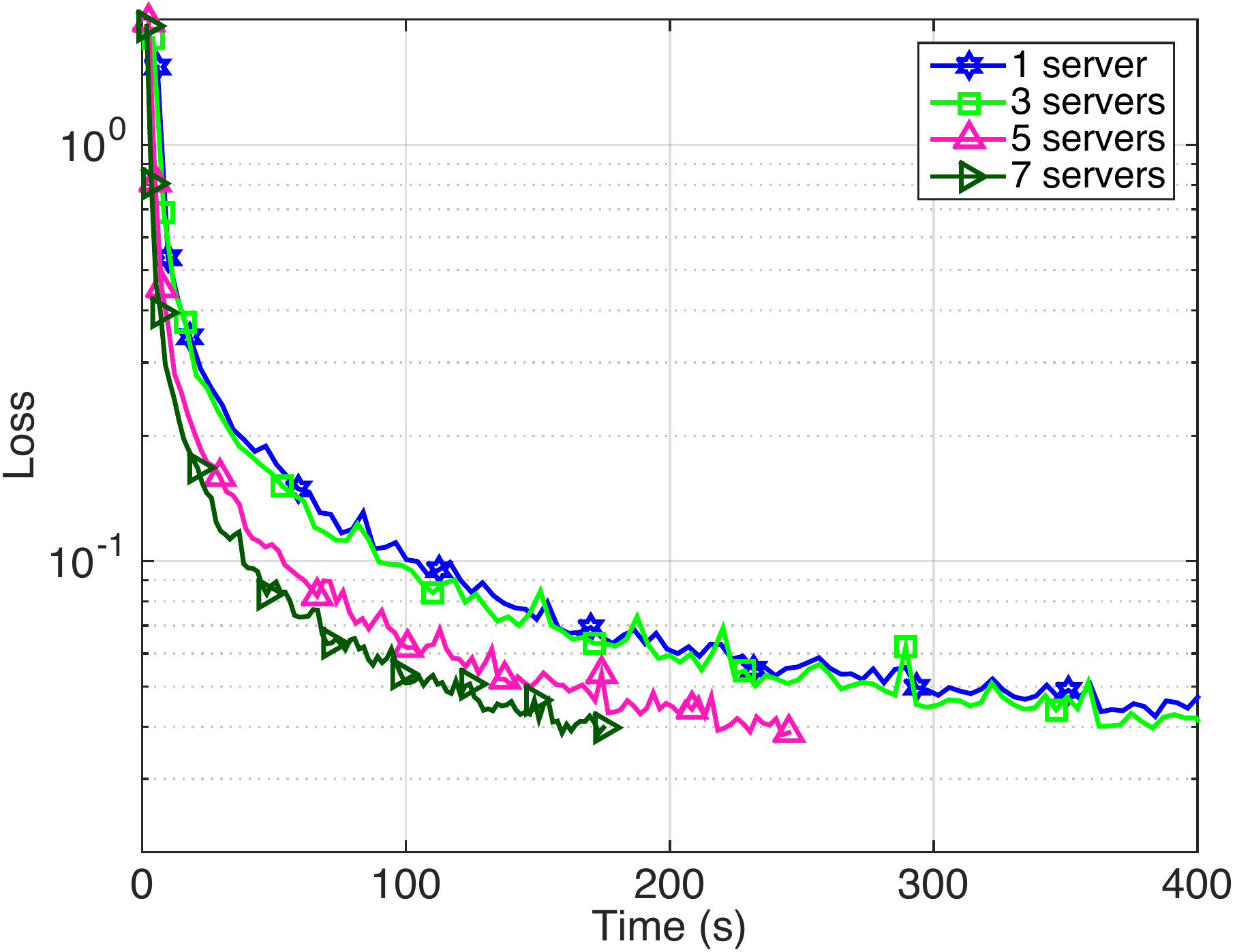}
	\end{minipage}
		
	\begin{minipage}{0.49\linewidth}
		\includegraphics[width=0.96\linewidth]{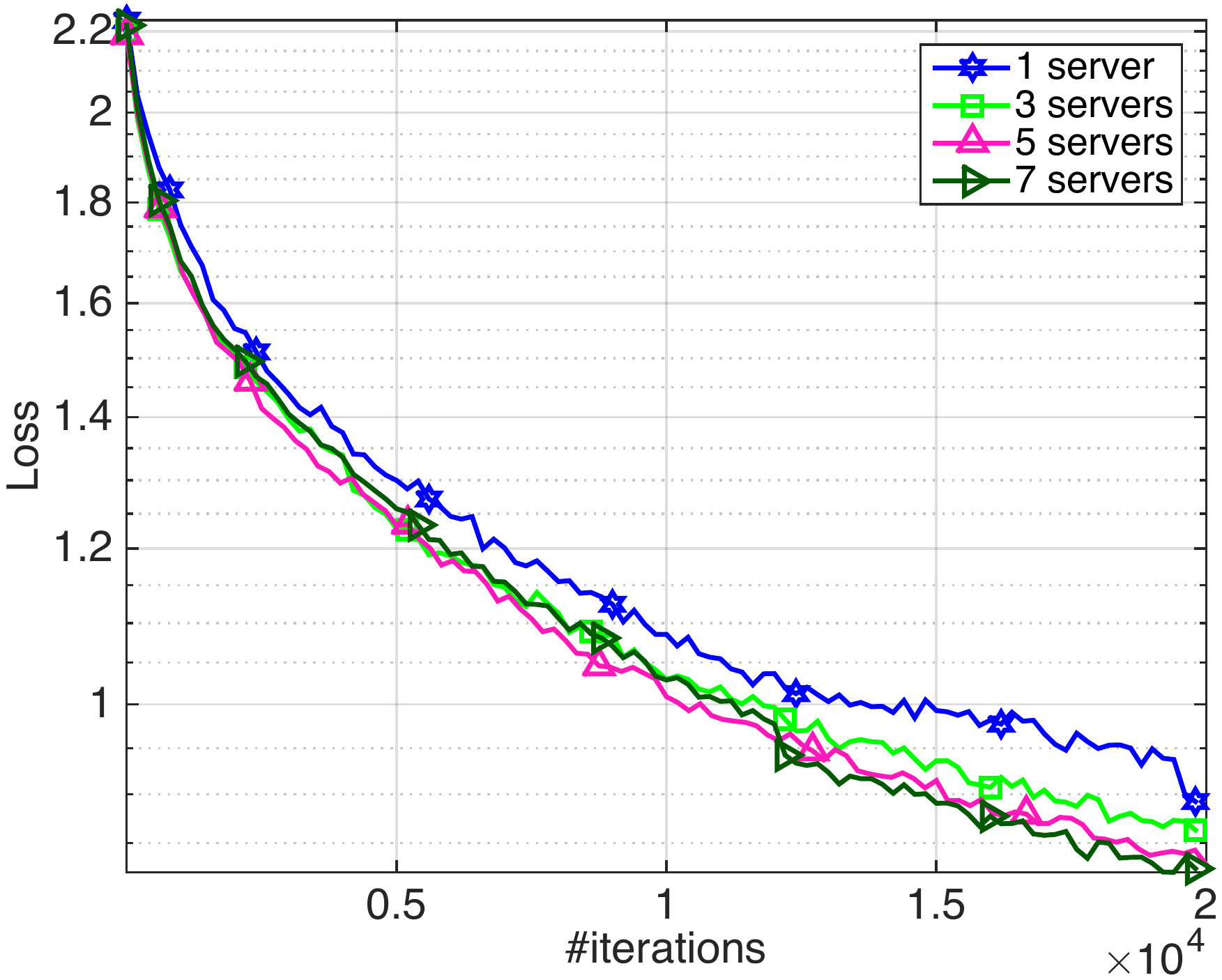}
	\end{minipage}
	\begin{minipage}{0.49\linewidth}
		\includegraphics[width=\linewidth]{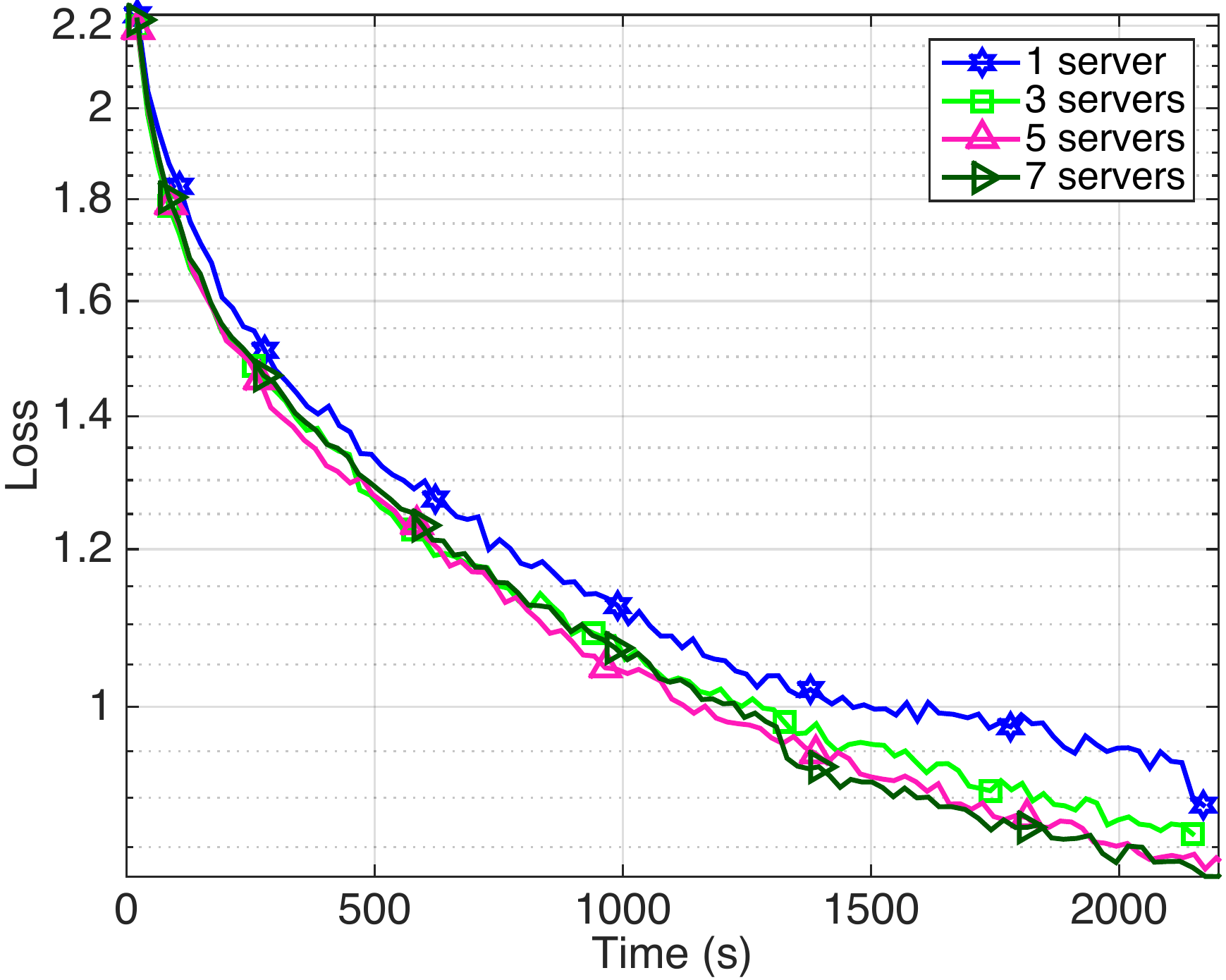}
	\end{minipage}
	\caption{Testing loss vs. \#servers. From top down, each row corresponds to
			the a9a, MNIST and CIFAR dataset, respectively. Each server is associated with 2 workers.}
	\label{fig:loss_server_2}
\end{figure} 

\begin{figure}[!htb]
	\begin{minipage}{0.49\linewidth}
		\includegraphics[width=\linewidth]{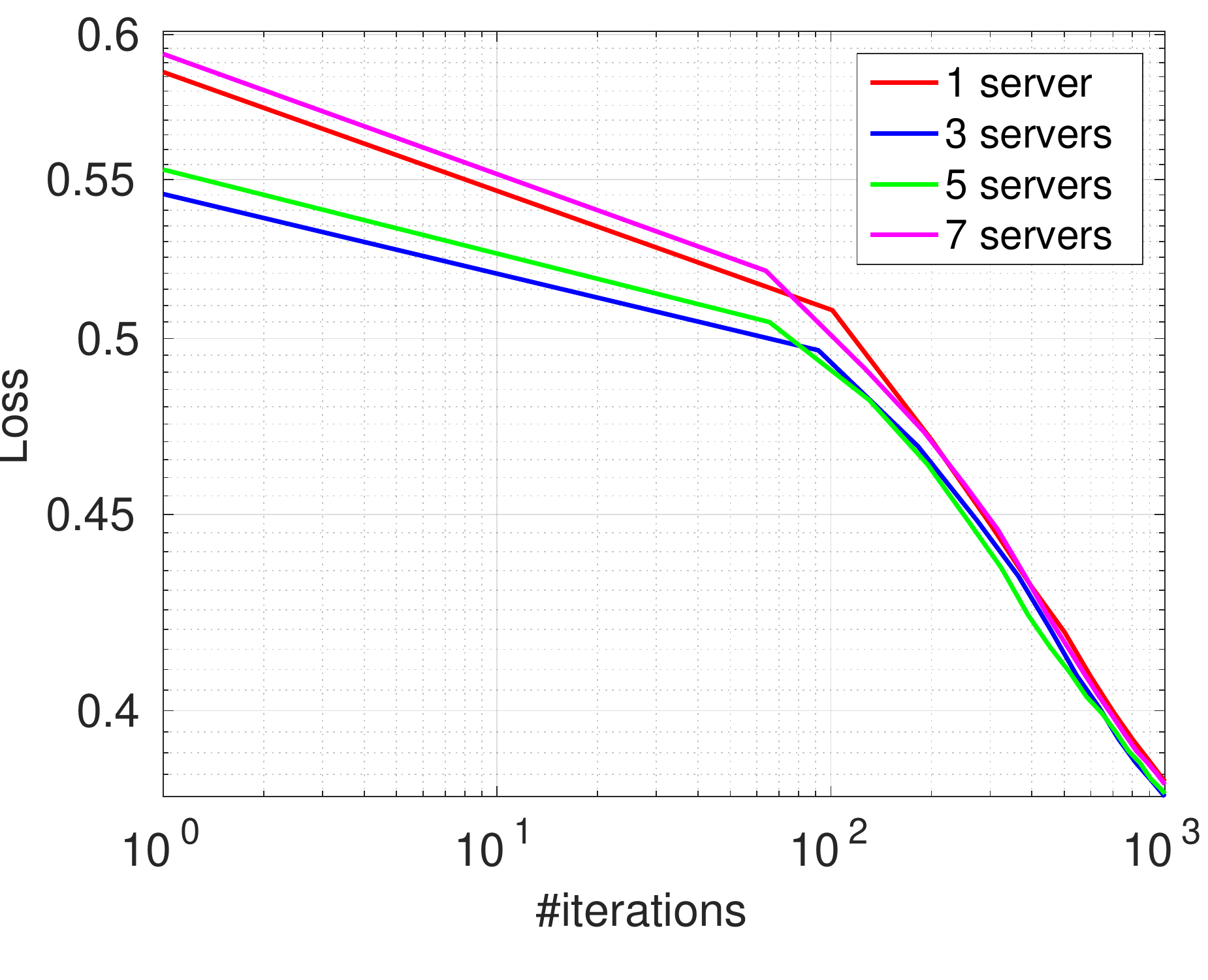}
	\end{minipage}
	\begin{minipage}{0.49\linewidth}
		\includegraphics[width=\linewidth]{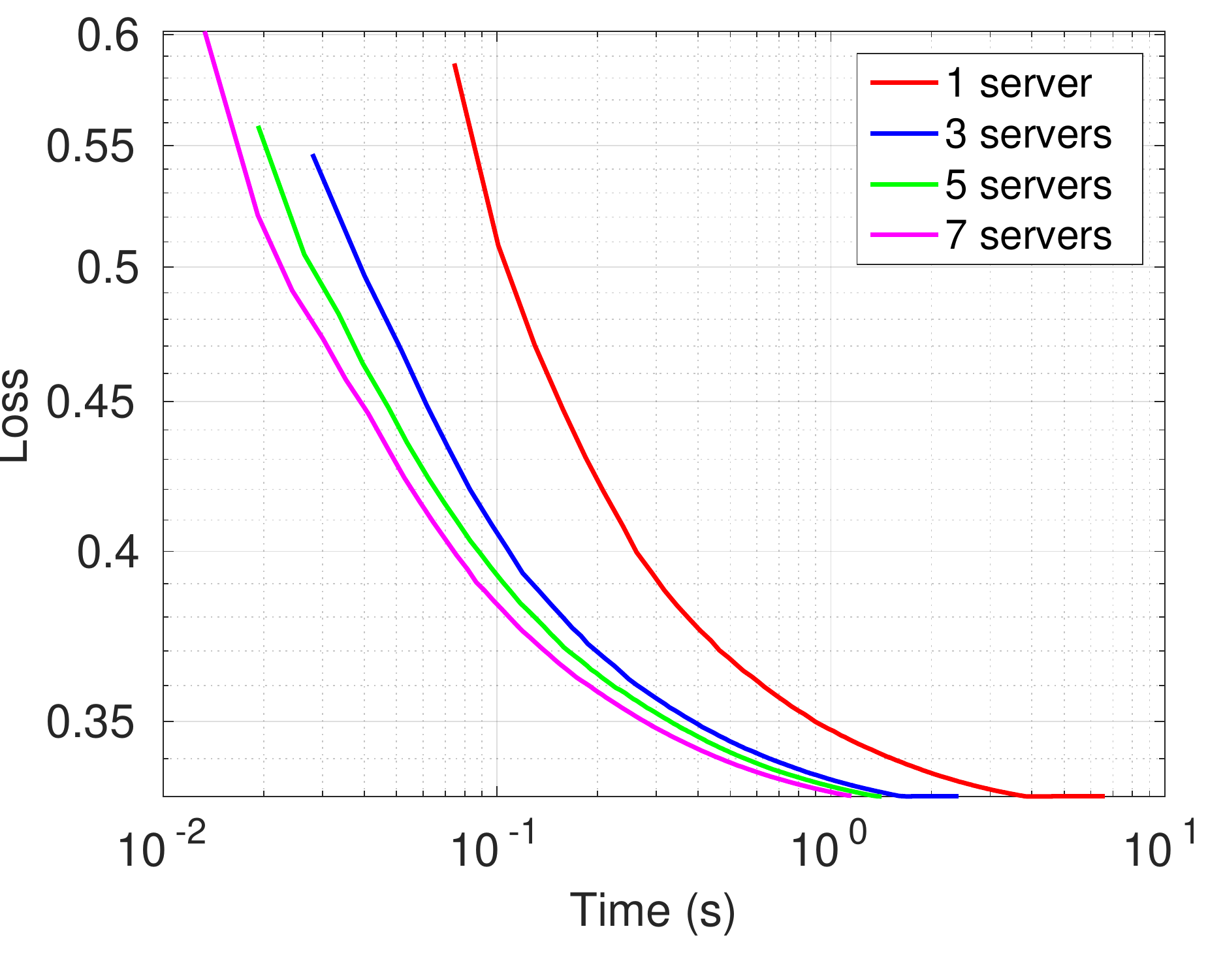}
	\end{minipage}
	
	\begin{minipage}{0.49\linewidth}
		\includegraphics[width=1.02\linewidth]{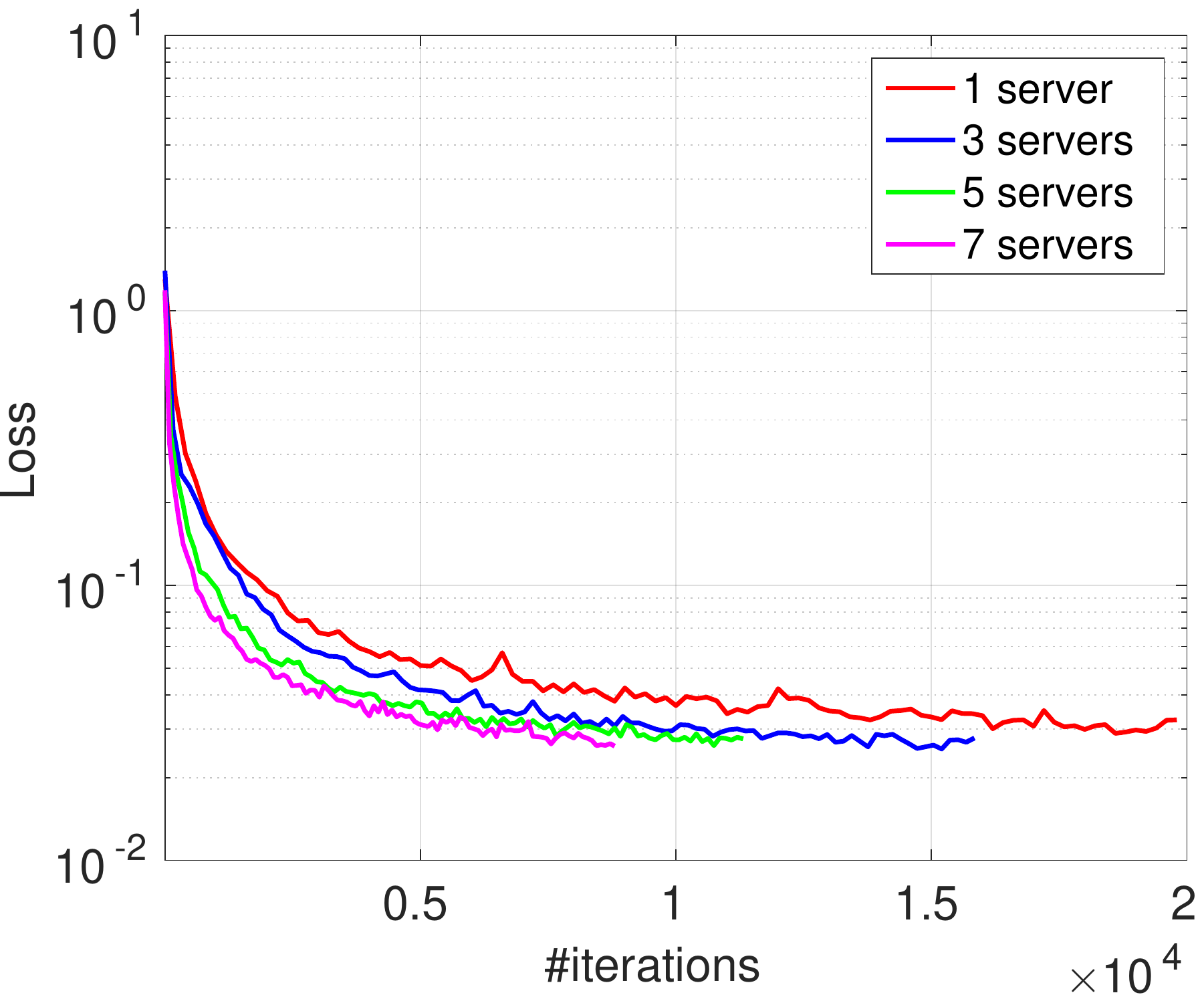}
	\end{minipage}
	\begin{minipage}{0.49\linewidth}
		\includegraphics[width=\linewidth]{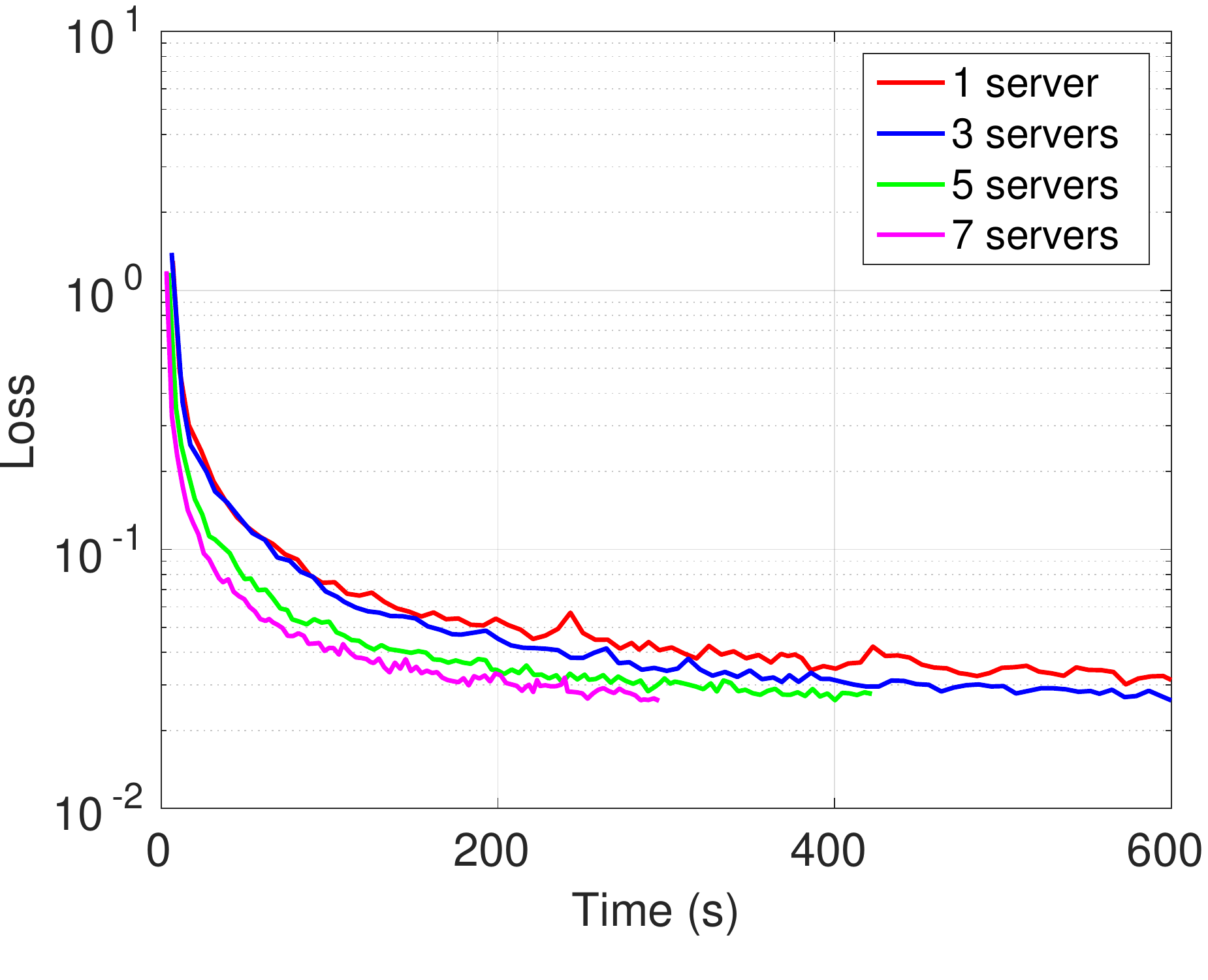}
	\end{minipage}
		
	\begin{minipage}{0.49\linewidth}
		\includegraphics[width=0.96\linewidth]{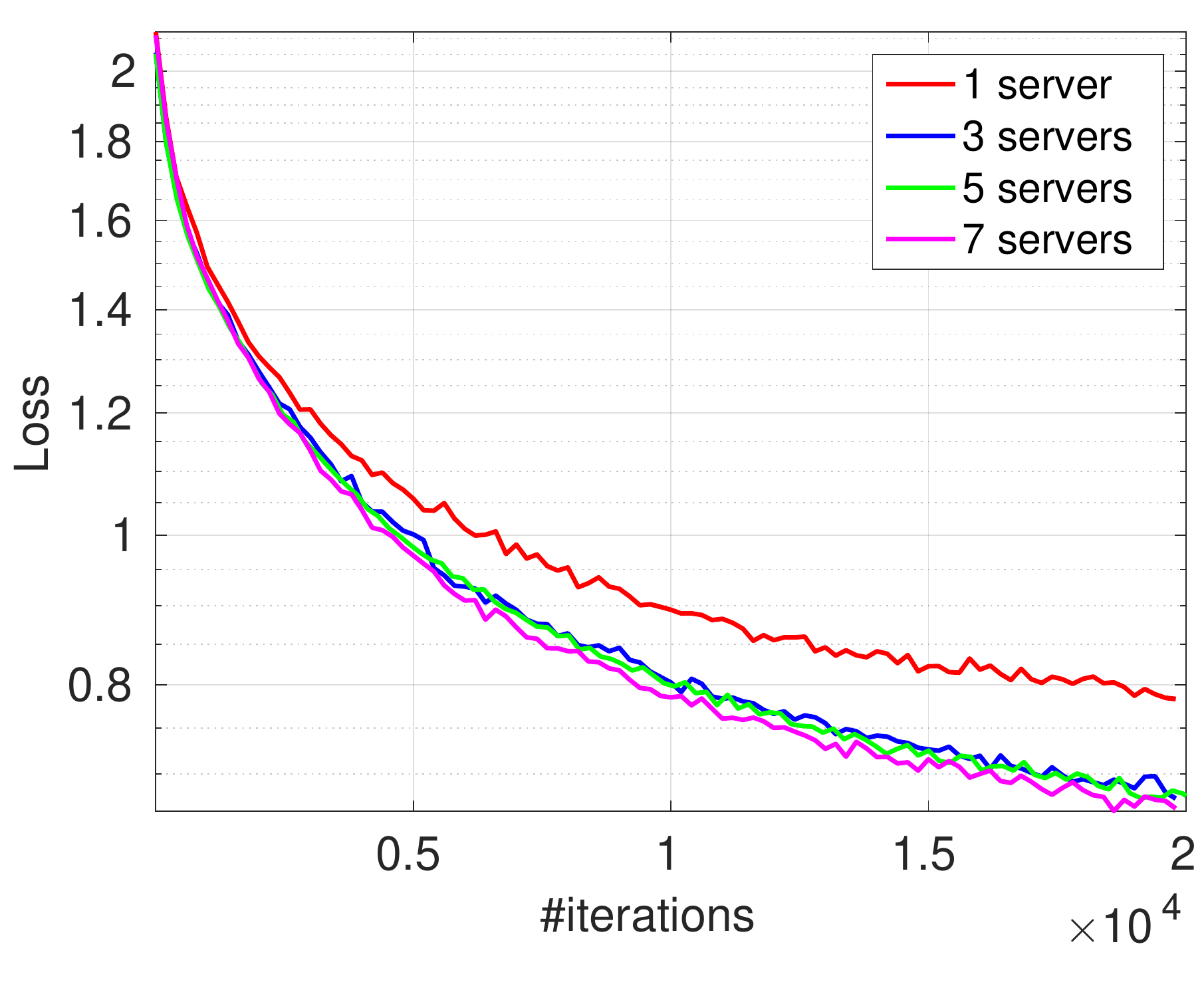}
	\end{minipage}
	\begin{minipage}{0.49\linewidth}
		\includegraphics[width=\linewidth]{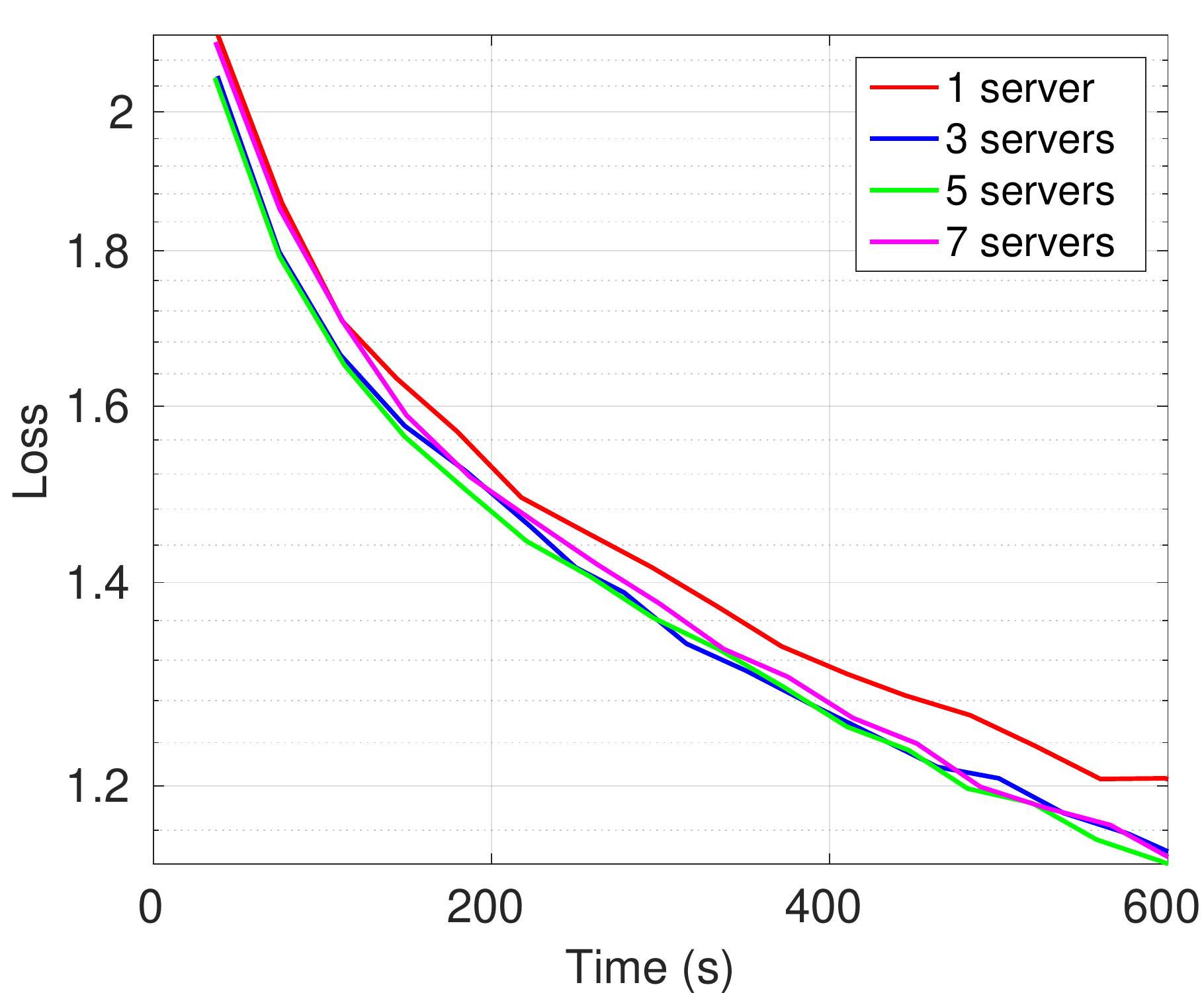}
	\end{minipage}
	\caption{Testing loss vs. \#servers. From top down, each row corresponds to
			the a9a, MNIST and CIFAR dataset, respectively. Each server is associated with 4 workers.}
	\label{fig:loss_server_4}
\end{figure} 

\begin{figure}[!htb]
	\begin{minipage}{0.49\linewidth}
		\includegraphics[width=\linewidth]{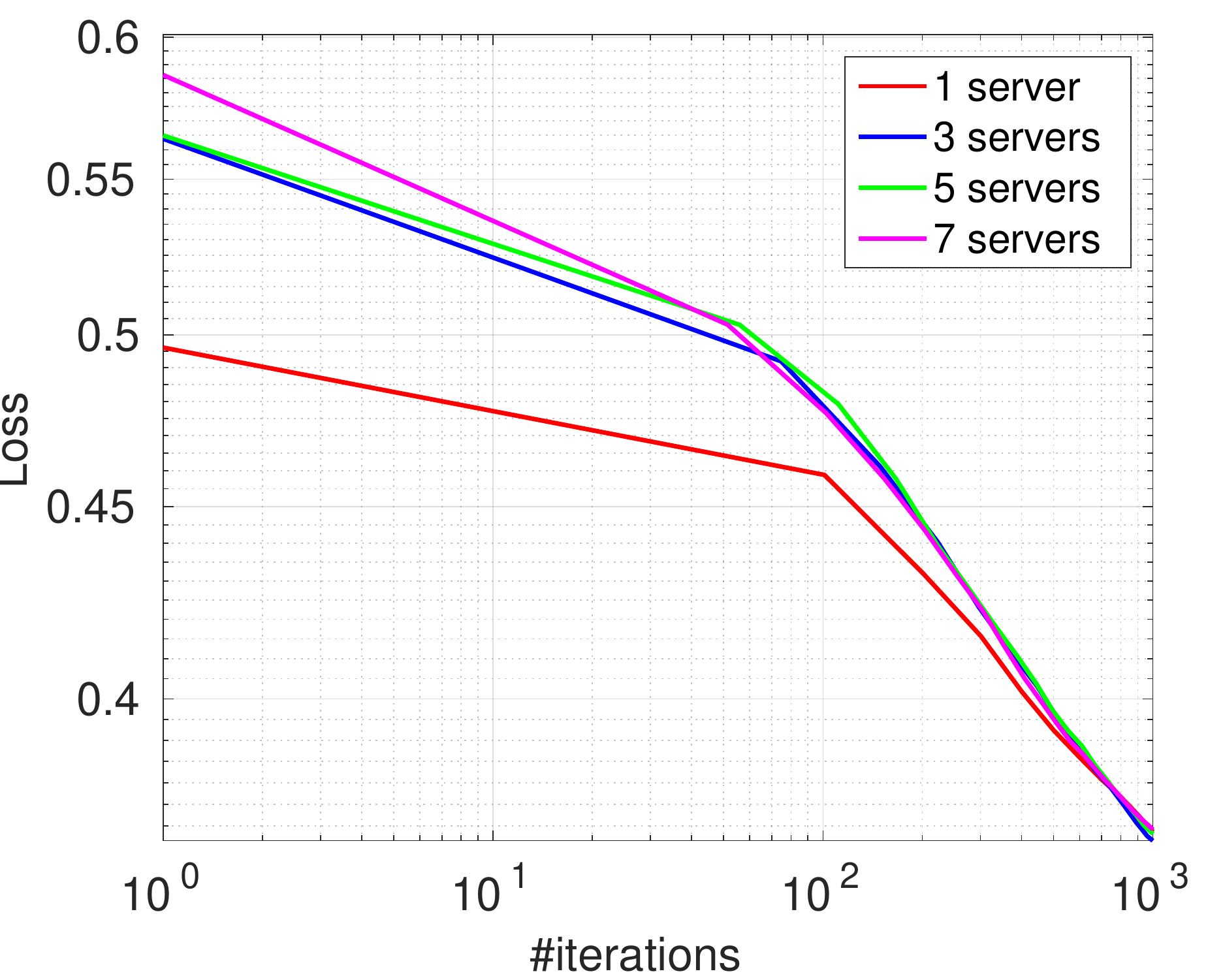}
	\end{minipage}
	\begin{minipage}{0.49\linewidth}
		\includegraphics[width=\linewidth]{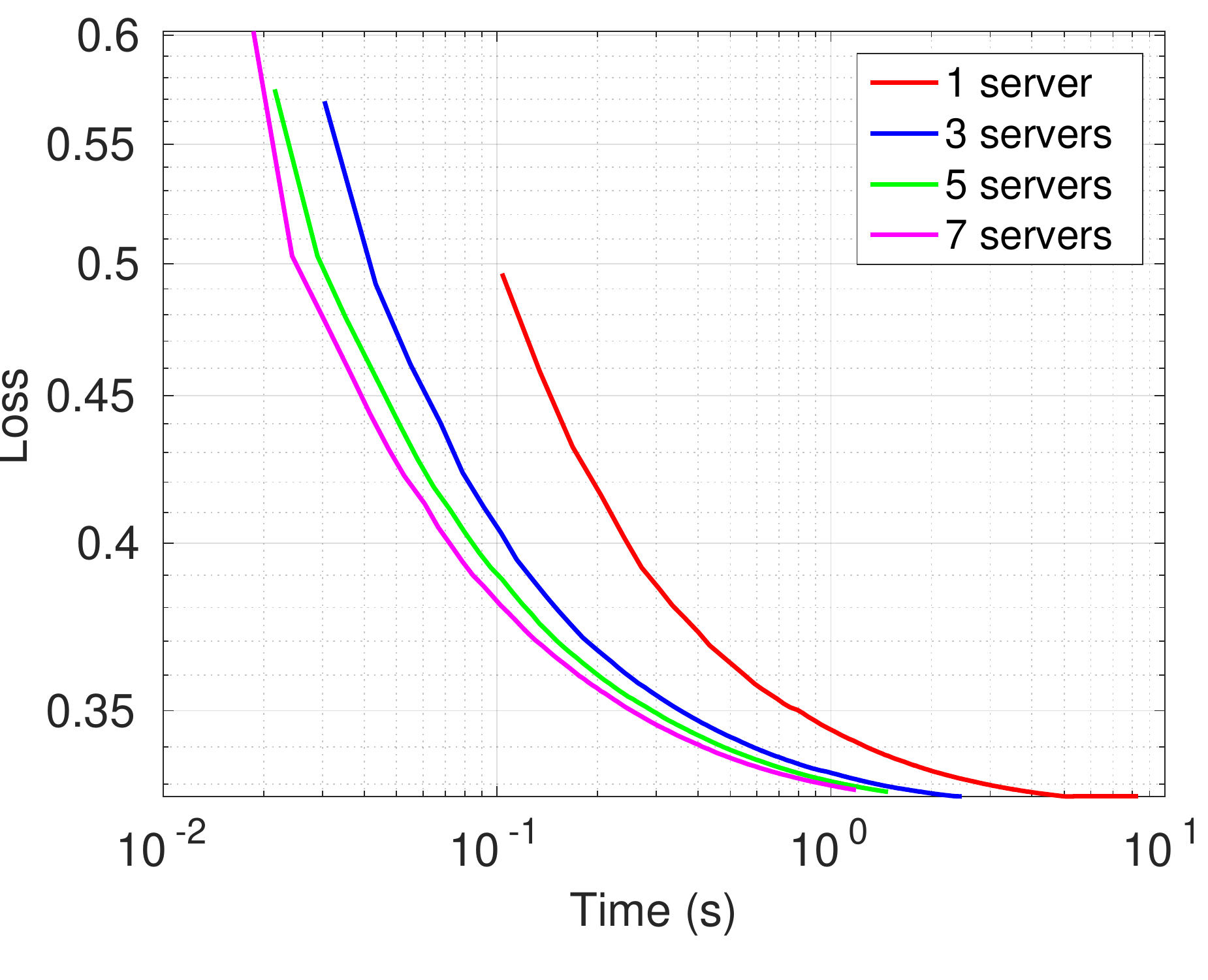}
	\end{minipage}
	
	\begin{minipage}{0.49\linewidth}\hspace{-1mm}
		\includegraphics[width=1.02\linewidth]{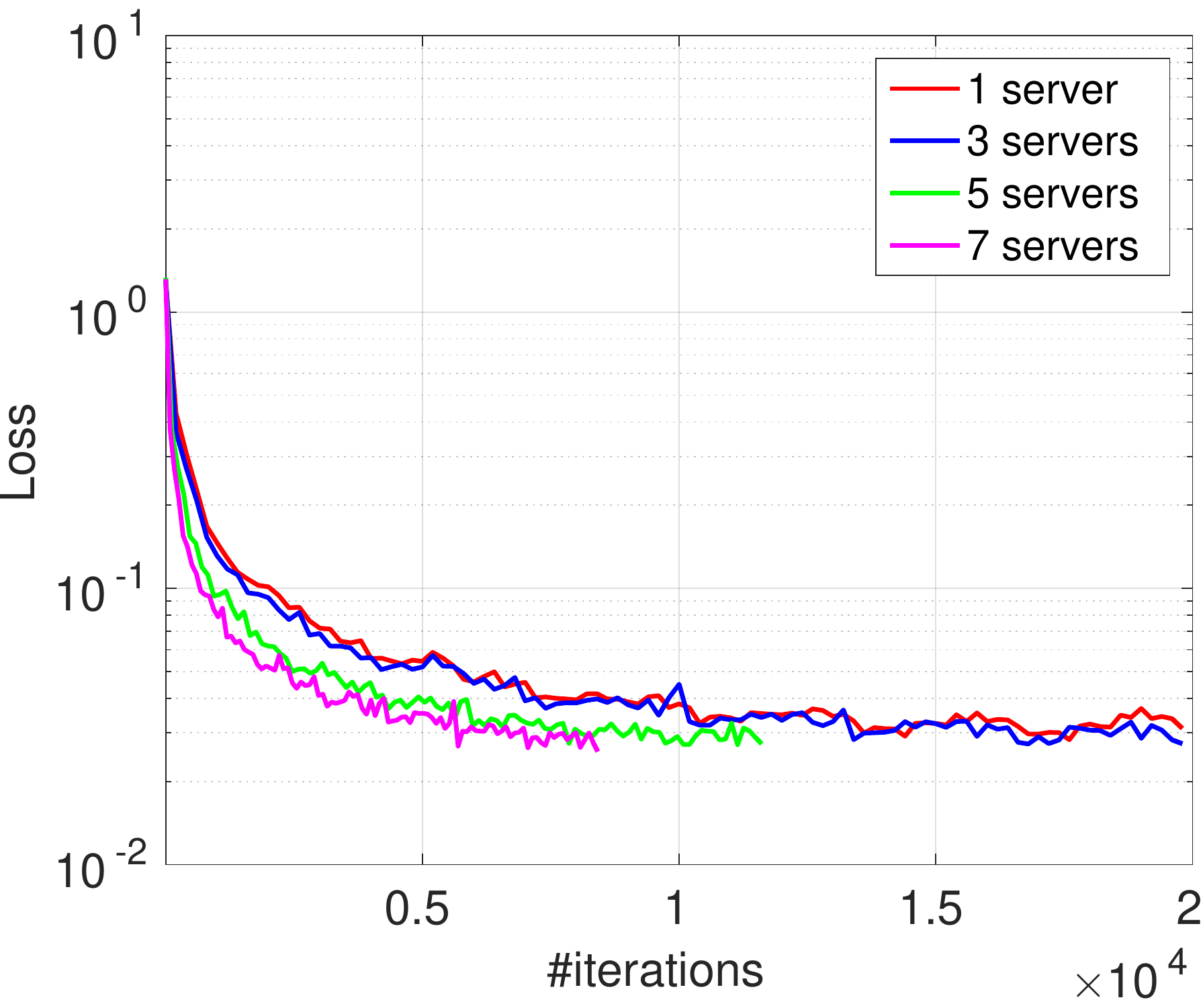}
	\end{minipage}
	\begin{minipage}{0.49\linewidth}\vspace{-0mm}\hspace{-1mm}
		\includegraphics[width=\linewidth]{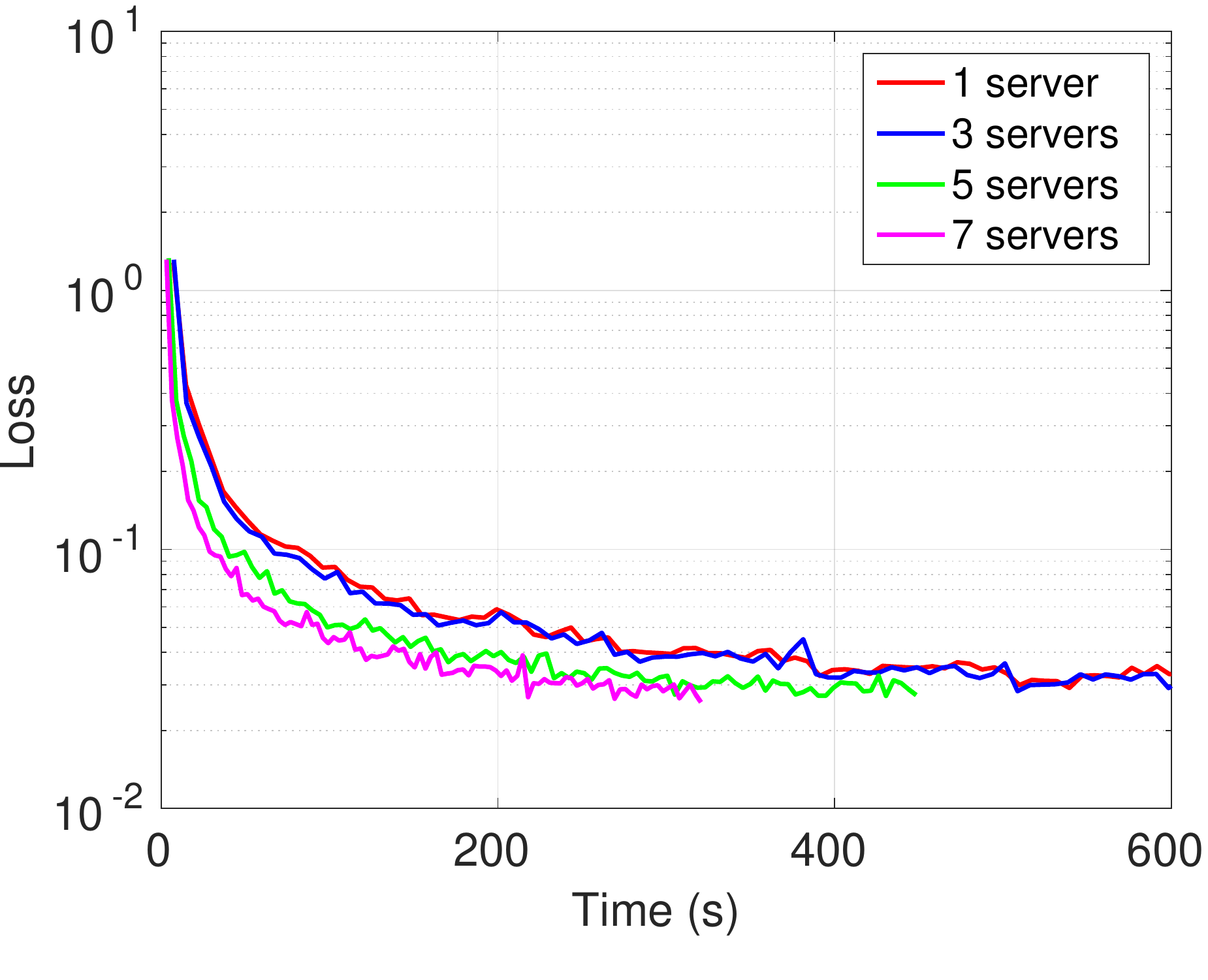}
	\end{minipage}
		
	\begin{minipage}{0.49\linewidth}\hspace{1mm}
		\includegraphics[width=0.96\linewidth]{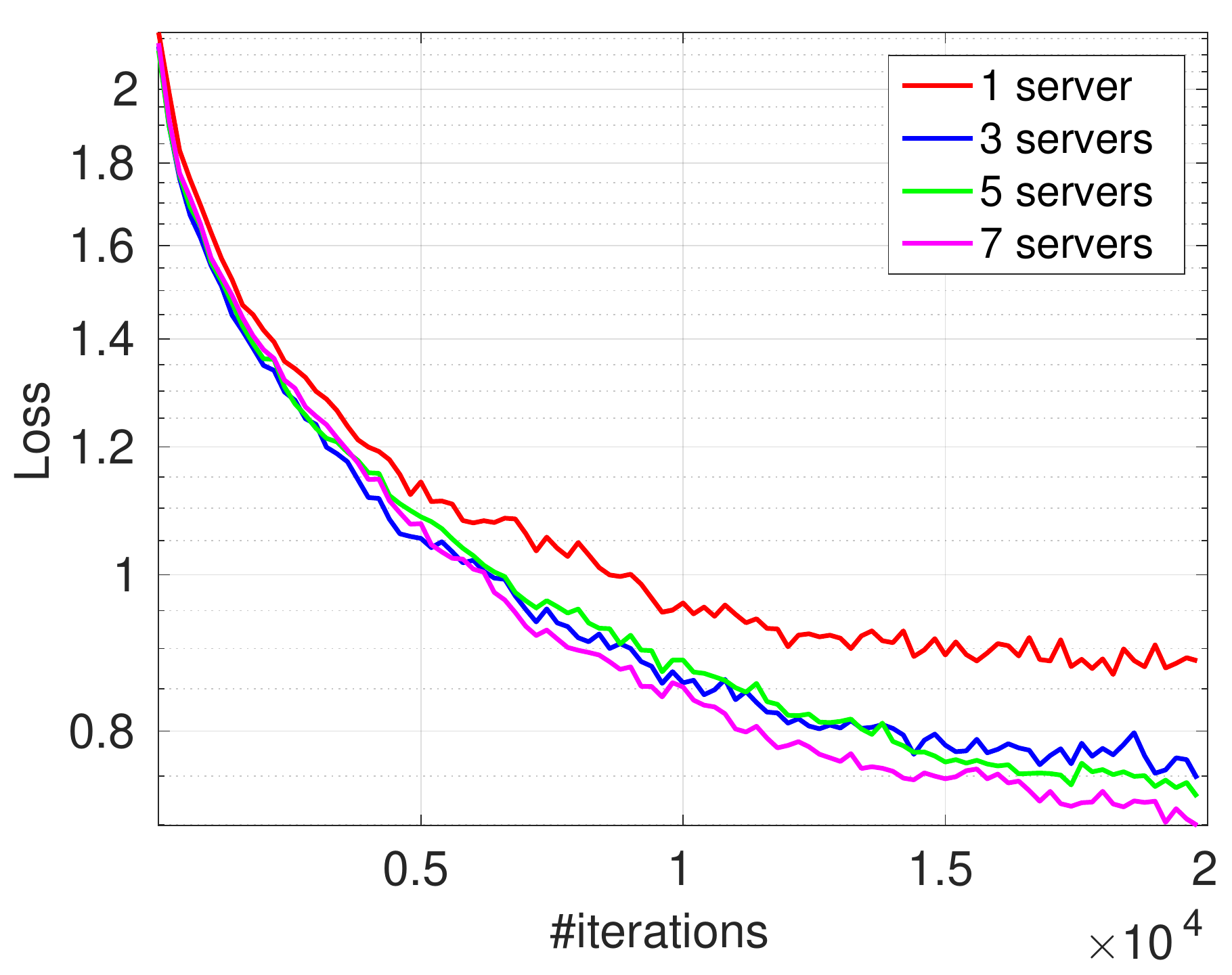}
	\end{minipage}
	\begin{minipage}{0.49\linewidth}\hspace{1mm}
		\includegraphics[width=\linewidth]{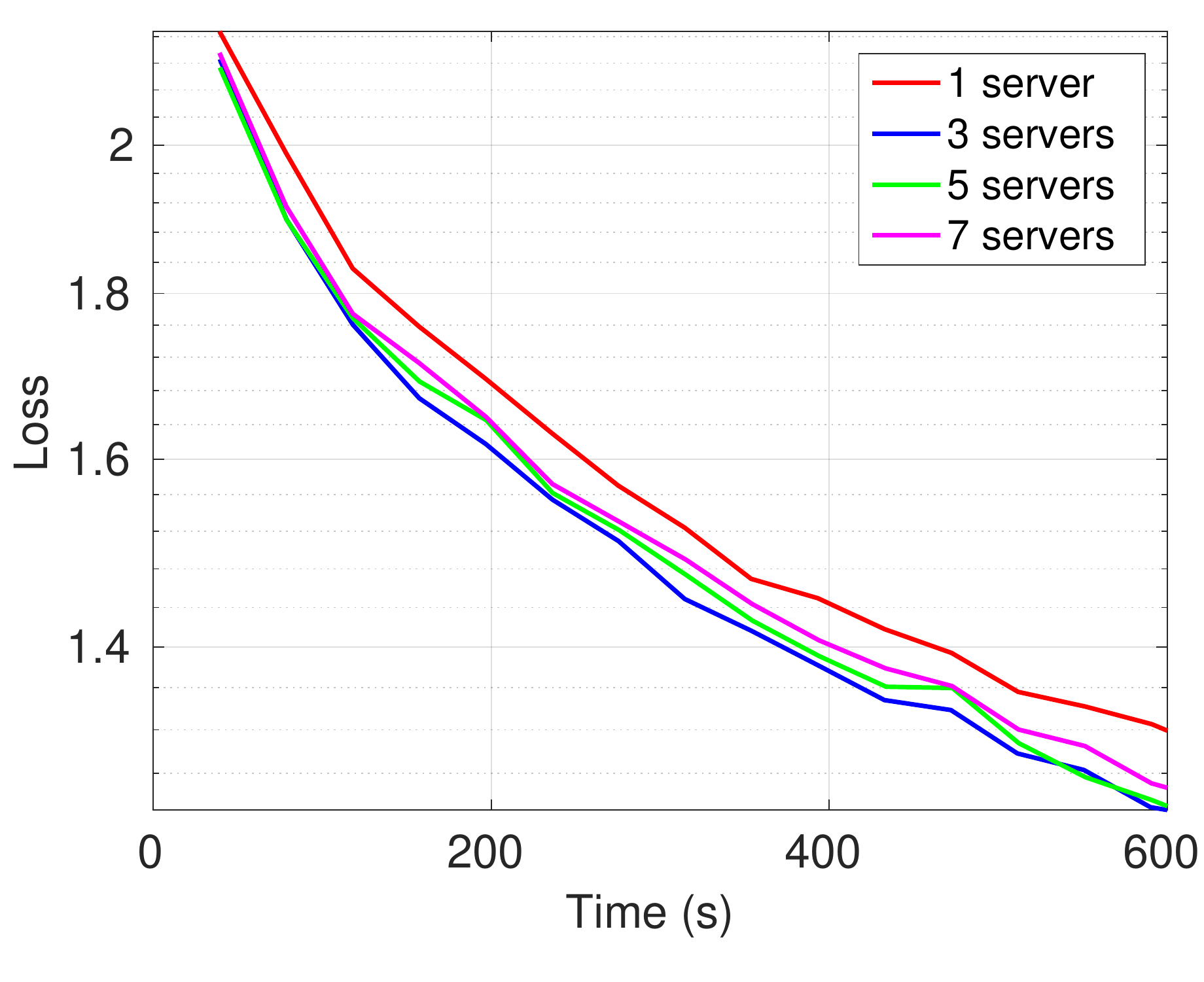}
	\end{minipage}
	\caption{Testing loss vs. \#servers. From top down, each row corresponds to
			the a9a, MNIST and CIFAR dataset, respectively. Each server is associated with 6 workers.}
	\label{fig:loss_server_6}
\end{figure} 

\end{document}